\documentclass{article}

\usepackage[preprint]{neurips_2026}

\usepackage[utf8]{inputenc}
\usepackage[T1]{fontenc}
\usepackage{url}
\usepackage{hyperref}
\usepackage{booktabs}
\usepackage{amsfonts}
\usepackage{amsmath}
\usepackage{amssymb}
\usepackage{amsthm}
\usepackage{mathtools}
\usepackage{algorithm}
\usepackage{algpseudocode}
\usepackage{enumitem}
\usepackage{microtype}
\usepackage{xcolor}
\usepackage{graphicx}
\usepackage{subcaption}
\usepackage{multirow}
\usepackage{makecell}
\usepackage{nicefrac}
\usepackage{placeins}     

\graphicspath{{figures/}}

\newcommand{\xtc}{\textsc{XTC}}
\newcommand{\Vocab}{\mathcal{V}}
\newcommand{\pt}{p_t}
\newcommand{\qt}{q_t}
\newcommand{\Et}{E_t(\tau)}
\newcommand{\Rt}{R_t}
\newcommand{\E}{\mathbb{E}}
\DeclareMathOperator*{\argminprob}{arg\,min}

\newtheorem{proposition}{Proposition}

\title{XTC: Head-Aware Sampling by Excluding Top Choices}

\author{%
  Philipp Emanuel Weidmann\thanks{Equal contribution.}\thanks{Proposed the \xtc{} sampling method and wrote the original reference implementations, which were merged into \texttt{llama.cpp}, \texttt{text-generation-webui}, and other open-source inference engines.} \\
  Independent Researcher \\
  \texttt{pew@worldwidemann.com} \\
  \And
  Allen G. Roush\footnotemark[1]\thanks{Designed and ran all experiments, evaluations, and statistical analysis.} \\
  Thoughtworks \\
  \texttt{allen.roush@thoughtworks.com} \\
  \AND
  Judah Goldfeder \\
  Columbia University \\
  \texttt{jag2396@columbia.edu} \\
  \And
  Sanjay Basu \\
  Oracle \\
  \texttt{sanjay.basu@oracle.com} \\
  \And
  Ravid Shwartz-Ziv \\
  New York University \\
  \texttt{ravid.shwartz.ziv@nyu.edu} \\
}

\begin{document}

\maketitle

\begin{abstract}
Standard decoding rules for autoregressive language models promote diversity by rescaling the full next-token distribution or by truncating its low-probability tail.
These strategies overlook a recurring regime of open-ended generation in which the model already assigns substantial probability to several plausible continuations yet still concentrates too much mass on the most generic choice.
We introduce \xtc{} (E\textbf{x}clude \textbf{T}op \textbf{C}hoices), a lightweight head-aware decoding operator that targets this \emph{head-ambiguity} regime directly.
Given a next-token distribution, \xtc{} identifies the set of tokens exceeding an absolute plausibility threshold~$\tau$.
When two or more such tokens exist, it removes the dominant eligible choices with probability~$\rho$ and retains only the weakest plausible alternative before renormalizing.
A comprehensive evaluation spanning 60 experiments across three primary model families (Gemma~3 27B q4, Gemma~3 12B q6, DeepSeek R1 14B q6), extended with a scaling validation on Llama 3.3 70B q4, confirms the predicted operating profile.
On creative generation tasks, \xtc{} improves the diversity-repetition Pareto frontier with Distinct-2 gains of 11--15\% (monotone in parameter count from 12B to 70B) and repeat trigram reductions of 27--47\% across the four tested models.
When composed with temperature scaling, total improvements reach 38\% (Distinct-2) and 71\% (repeat trigram reduction) over baseline.
A blinded Amazon Mechanical Turk study with 150 Master raters confirms that these distributional shifts translate to a 62.3\% creativity preference for \xtc{} ($p<10^{-4}$) without sacrificing fluency, and a cross-vendor GPT-4o control judge replicates the Anthropic-judge signal on every directional measure.
On instruction-following (IFEval, Llama~3.3 70B q4), \xtc{} preserves prompt-level strict accuracy within 1.7 percentage points of baseline at parameters that recover most of the diversity gain. A temperature setting matched on Distinct-2 collapses IFEval by 8.8 points at the same Distinct-2 target.
The effect is additive with temperature and repetition penalties, robust across quantization levels and model families, and consistent across all twelve tested prompt genres. XTC has been adopted by popular open-source LLM inference
frameworks, including \texttt{llama.cpp}, ExLlamaV2, and
text-generation-webui, highlighting its impact on practical text
generation.
\end{abstract}

\section{Introduction}

Autoregressive language models \citep{bengio2003neural, vaswani2017attention, brown2020language} require a decoding strategy that balances faithfulness to the learned distribution against the goals of the downstream task \citep{wiher2022decoding, celikyilmaz2020evaluation}.
Conservative sampling produces generic, repetitive text \citep{holtzman2020curious}.
Aggressive sampling admits implausible low-probability tokens or perturbs steps where the model is already confident \citep{hewitt2022truncation, renze2024effect}.
A large body of work has therefore sought decoding rules that broaden the space of plausible continuations without sacrificing coherence \citep{fan2018hierarchical, holtzman2020curious, meister2023typical, basu2021mirostat}.

The most widely adopted samplers truncate the low-probability tail or rescale the full distribution.
Temperature changes the relative sharpness of all logits \citep{ackley1985learning}.
Top-$k$ and top-$p$ remove weak tokens by rank or cumulative mass \citep{fan2018hierarchical, holtzman2020curious}.
Typical and truncation-based methods target atypical or improbable candidates \citep{meister2023typical, hewitt2022truncation}.
Min-$p$ thresholds relative to the maximum token probability, adapting the cutoff to the peakedness of each distribution \citep{ICLR2025_afa5f124}.
These approaches share a structural assumption: undesirable randomness originates primarily in the tail of the distribution or in its overall entropy.

Open-ended generation frequently exhibits a distinct failure mode.
The model assigns substantial probability to several individually plausible continuations while overweighting the safest or most conventional one.
In this \emph{head-ambiguity} regime, tail truncation has little effect because the relevant alternatives already reside in the head, and global flattening is too coarse because it intervenes at every step regardless of whether the model is genuinely uncertain among strong options.
Recent studies confirm that LM-assisted writing reduces population-level content diversity even when individual outputs appear varied \citep{padmakumar2024contentdiversity, doshi2024generative, anderson2024homogenization}, suggesting that the head-dominance problem has practical consequences beyond benchmark scores.

We introduce \xtc{} (\textbf{E}x\textbf{c}lude \textbf{T}op \textbf{C}hoices), a decoding operator for the head-ambiguity regime.
Given a next-token distribution, \xtc{} identifies the set of tokens exceeding an absolute plausibility threshold~$\tau$.
When this set contains at least two tokens, it removes the dominant eligible choices with probability~$\rho$, retains only the weakest token that still clears the plausibility floor, and renormalizes.
The operator has four defining properties:
\begin{itemize}[leftmargin=1.25em,nosep]
\item It is \emph{sparse in time}: inactive whenever the head is effectively unambiguous.
\item It is \emph{head-aware}: targets individually probable alternatives in the head.
\item It is \emph{compositional}: inserts into any existing sampler stack as a small distribution transformation.
\item It is \emph{conditionally beneficial}: helps most when head multiplicity is present and is nearly inert otherwise.
\end{itemize}

Our contributions are as follows.
\textbf{(1)}~We formalize \xtc{} as an operator on a token distribution and provide a complete, implementation-ready algorithm (Section~\ref{sec:method}).
\textbf{(2)}~We derive structural properties including a KL-projection interpretation that distinguishes \xtc{} from both tail-truncation and global-entropy methods (Appendix~\ref{app:properties}).
\textbf{(3)}~We present a comprehensive evaluation: 60 experiments across three primary model families (Gemma~3 27B q4, Gemma~3 12B q6, DeepSeek R1 Qwen 14B q6) plus a scaling validation on Llama 3.3 70B q4, covering creative diversity, long-form degeneration, code exactness, design ablations, sampler composition, parameter interaction, and cross-model robustness (Section~\ref{sec:experiments}).
\textbf{(4)}~We show that \xtc{} improves the diversity-repetition Pareto frontier on creative tasks (Distinct-2 gains of 11--15\% across four tested models, monotone in parameter count from 12B to 70B) while preserving quality on exactness-sensitive tasks within identified safe operating boundaries.
\textbf{(5)}~We demonstrate that \xtc{} composes additively with temperature and repetition penalties, achieving combined Distinct-2 improvements of up to 38\% with repeat trigram reductions of 71\% over baseline.
\textbf{(6)}~We close the loop on the evaluation hierarchy with a blinded AMT human study, a Claude Opus~4.7 LLM judge cross-checked by an OpenAI \texttt{gpt-4o} control, and an IFEval characterization that quantifies the instruction-following cost of head-aware diversity (Section~\ref{sec:exp_quality}, Section~\ref{sec:exp_ifeval}).

\section{Background and related work}
\label{sec:related}

The tension between faithfulness and diversity in autoregressive decoding has generated a rich literature \citep{gatt2018survey, wiher2022decoding, li2022pretrained}.
Top-$k$ sampling \citep{fan2018hierarchical} and nucleus sampling (top-$p$) \citep{holtzman2020curious} truncate the distribution tail by rank or cumulative mass.
Typical decoding removes atypical tokens according to an information-theoretic criterion \citep{meister2023typical}.
Epsilon and eta sampling apply fixed or entropy-adaptive probability cutoffs \citep{hewitt2022truncation}.
Mirostat maintains a target perplexity level through online surprise control \citep{basu2021mirostat}.
Min-$p$ sampling thresholds relative to the maximum probability, adapting truncation to the peakedness of each distribution \citep{ICLR2025_afa5f124}.
Temperature scaling, the oldest diversity control, rescales all logits by a single factor \citep{ackley1985learning, renze2024effect}.
These methods focus on the tail or the global shape of the distribution.
\xtc{} targets the head.

\paragraph{Sequence-level and training-time diversity.}
Diverse beam search injects inter-group diversity penalties into beam decoding \citep{vijayakumar2018diverse}.
Contrastive search \citep{su2022contrastive} and contrastive decoding \citep{li2023contrastive} improve generation by contrasting with degenerate continuations or weaker reference models.
DoLa decodes by contrasting early and late layers \citep{chuang2023dola}, FUDGE steers generation with future discriminators \citep{Yang_2021}, and COLD formulates controlled generation as energy-based sampling \citep{qin2022cold}.
Unlikelihood training discourages repeated tokens at the objective level \citep{welleck2020neural}, controllable generation methods steer with conditional prefixes or control codes \citep{keskar2019ctrl}, and MMI-based decoding promotes diversity through mutual-information objectives \citep{li2016diversity}.
These methods operate at the sequence level, require auxiliary models, or modify training. \xtc{} operates on a single next-token distribution within any sampler stack.

\paragraph{Diversity measurement and LLM-as-judge.}
Evaluating text diversity is itself an active research area \citep{shaib2024diversity, shaib2024templates, zhu2018texygen, papineni2002bleu, pillutla2021mauve, zhang2020bertscore, hashimoto2019unifying}.
We follow the multi-metric philosophy advocated by these works and report five complementary diversity families.
Strong language models are also routinely used to score open-ended text \citep{zheng2023judging, chiang2023closer, li2024leveraging}.
Systematic biases of LLM judges (position, verbosity, same-family preference) have been documented and partially controlled through randomization, length normalization, and cross-vendor calibration \citep{dubois2024length, wang2024pandalm}.
We pre-register seven criteria, score every generation with Claude Opus~4.7 as the primary judge, and cross-validate against an OpenAI \texttt{gpt-4o} control judge (Section~\ref{sec:exp_quality}).

\paragraph{The homogenization problem.}
A growing body of work documents that LM-assisted writing reduces content diversity at a population level, even when individual outputs appear varied \citep{padmakumar2024contentdiversity}.
Similar effects appear in creative writing \citep{doshi2024generative}, ideation \citep{anderson2024homogenization}, and more broadly across expression and reasoning strategies \citep{sourati2025homogenizing}.
RLHF training has been shown to narrow the stylistic and topical range of model outputs \citep{kirk2024understanding, ouyang2022training, bai2022training}, and training data deduplication can partially mitigate convergence in generation \citep{lee2022deduplicating}.
These findings motivate decoding-time interventions that increase diversity within individual generation episodes, the regime \xtc{} targets.

\paragraph{Position of \xtc{}.}
\xtc{} requires no retraining, beams, or auxiliary models.
It acts on a single next-token distribution after any upstream transformations.
The closest conceptual relative is locally typical sampling \citep{meister2023typical}, which removes atypical (low-probability) tokens to keep generation close to the model's expected information content.
\xtc{} inverts this target: it removes the most typical (highest-probability) tokens among those that are individually plausible, promoting a less dominant but still viable continuation.
This inversion of the target set, from tail to head, is the core conceptual distinction.

\section{The \xtc{} decoding rule}
\label{sec:method}

Consider an autoregressive language model producing a next-token distribution $\pt$ over a vocabulary $\Vocab$ \citep{bengio2003neural, radford2019language}.
In practice, $\pt$ can come from any upstream sampler stack that has already applied temperature, penalties, or truncation \citep{wiher2022decoding}, and \xtc{} operates on this distribution.

Given an absolute eligibility threshold $\tau \in (0,1)$ and an intervention probability $\rho \in [0,1]$, \xtc{} identifies the \emph{eligible set} of tokens whose probability individually exceeds~$\tau$.
With fewer than two eligible tokens the distribution passes through unchanged.
With two or more, the operator fires with probability $\rho$, removing all eligible tokens except the \emph{least probable} one and renormalizing the remaining mass.
The retained token is a minimal head alternative: still individually plausible, the least dominant among the eligible options.

\subsection{Algorithm}

\begin{algorithm}[t]
\caption{\xtc{} sampling step. Plain-language reading: ``If the model is genuinely undecided between strong choices, with probability $\rho$ throw away the dominant ones and keep only the underdog.''}
\label{alg:xtc}
\begin{algorithmic}[1]
\Require Next-token distribution $p$, plausibility floor $\tau$, intervention rate $\rho$, optional protected tokens $S$ (e.g.\ EOS, newline)
\State \textbf{Find plausible tokens.} Collect every token whose probability is at least $\tau$ into a set $E$.
\If{$|E| < 2$}
    \State \Return $p$ unchanged \Comment{the head is unambiguous, do nothing}
\EndIf
\State \textbf{Roll the dice.} Flip a biased coin that lands heads with probability $\rho$.
\If{the coin shows tails}
    \State \Return $p$ unchanged \Comment{stochastic skip}
\EndIf
\State \textbf{Pick the underdog.} Among the plausible tokens, take the one with the lowest probability and call it $u$.
\State \textbf{Mark the dominant choices for removal.} Let $R = E \setminus \{u\}$.
\If{any token in $R$ is protected}
    \State \Return $p$ unchanged \Comment{never silence EOS or formatting tokens}
\EndIf
\State \textbf{Remove the dominant choices.} Set $p(v) \gets 0$ for every $v \in R$.
\State \textbf{Renormalize.} Rescale the surviving probabilities so they sum to one.
\State \Return the modified distribution.
\end{algorithmic}
\end{algorithm}

The closed-form mathematical specification (eligible set, removed set, transformed distribution) appears in Appendix~\ref{app:formal}.

\begin{figure}[t]
\centering
\includegraphics[width=\textwidth]{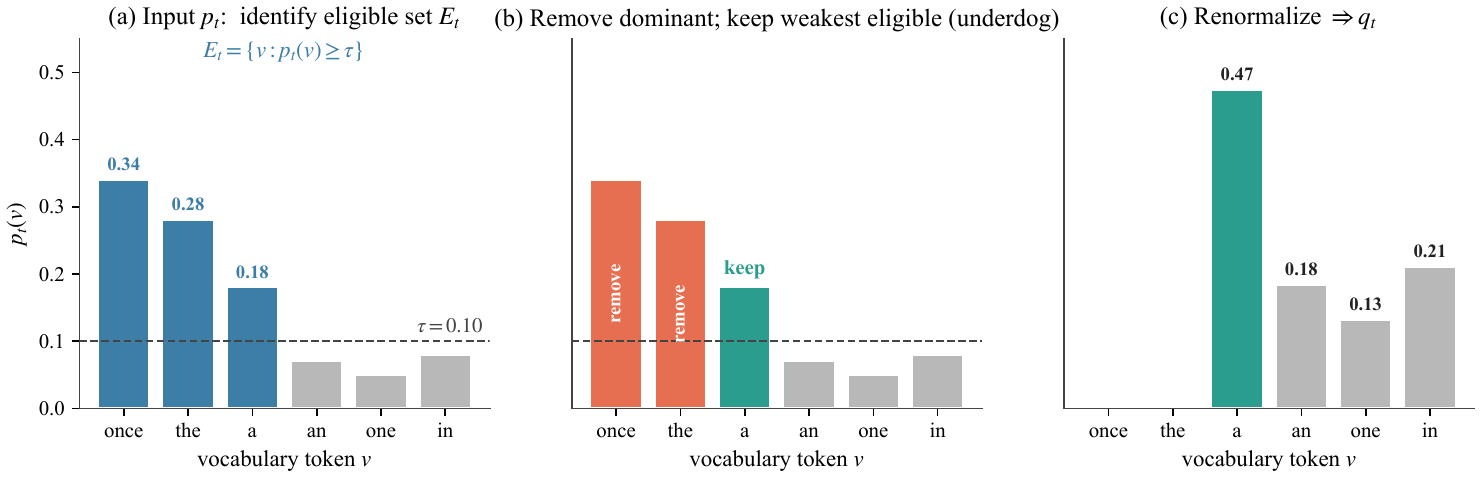}
\caption{The \xtc{} operator at one decoding step. \textbf{(a)} Given the next-token distribution $\pt$ and an absolute plausibility floor~$\tau$, the eligible set $E_t = \{v : \pt(v) \ge \tau\}$ is highlighted in blue. The toy distribution shows the head-ambiguity regime: three tokens are individually plausible, with substantial mass concentrated on the most generic continuation. \textbf{(b)} The two dominant eligible tokens are removed and the weakest plausible alternative is kept. The non-eligible tail is untouched. \textbf{(c)} Renormalization redistributes the removed mass over the surviving support, yielding the transformed distribution $\qt$. The retained underdog (green) now carries the largest probability among the surviving options.}
\label{fig:method_schematic}
\end{figure}

The design rationale and parameter discussion appear in Appendix~\ref{app:design_rationale}.

\section{Experiments}
\label{sec:experiments}

Our evaluation spans 60 experiments across three primary model families (Gemma~3 27B q4, Gemma~3 12B q6, DeepSeek R1 Qwen 14B q6) plus a scaling validation on Llama 3.3 70B q4, covering three independent architectures and parameter counts from 12B to 70B.
Cross-model generalization (Appendix~\ref{app:crossmodel}) and design ablation (Appendix~\ref{app:design_ablation}) appear with their figures and tables in the appendix.

\subsection{Setup}
\label{sec:setup}

\paragraph{Models.}
Our primary model is Gemma~3 27B instruction-tuned at \texttt{q4\_k\_m} quantization \citep{gemma2024}, served through the \texttt{text-generation-webui} API.
We additionally evaluate on Gemma~3 12B at \texttt{q6\_K} (smaller scale, higher quantization fidelity), DeepSeek R1 Qwen 14B at \texttt{q6\_K} \citep{deepseek2025r1} (a reasoning-distilled model on a Qwen base), and Llama 3.3 70B Instruct at \texttt{q4\_k\_m} (scaling validation to a third architecture family at larger parameter count).

\paragraph{Prompts and baselines.}
The creative evaluation suite contains 24 prompts spanning 12 genre tags.
Separate suites cover code exactness (executable Python with unit tests), quality retention (JSON extraction, constrained rewriting), and long-form repetition.
Each experiment includes a baseline at temperature 1.0.
Comparators include temperature scaling ($T \in \{1.1, 1.3\}$), top-$p$ (0.95), typical-$p$ (0.95), and repetition penalty (1.05), calibrated on a held-out prompt suite \citep{hewitt2022truncation}.
\xtc{} conditions span conservative ($\rho{=}0.05$, $\tau{=}0.10$) through aggressive ($\rho{=}1.0$, $\tau{=}0.05$).

\paragraph{Metrics.}
Following the multi-metric philosophy of \citet{shaib2024diversity} we report five complementary diversity families.
\textbf{Distinct-$n$} measures the fraction of unique $n$-grams in a text \citep{li2016diversity}.
\textbf{Self-BLEU-4} computes the average BLEU-4 score \citep{papineni2002bleu} between pairs of generations from the same prompt \citep{zhu2018texygen}.
\textbf{Repeat trigram rate} measures the fraction of trigrams appearing more than once within a single generation.
\textbf{Embedding cosine distance} is the average pairwise cosine distance between sentence embeddings of generations from the same prompt \citep{zhang2020bertscore}.
We additionally report compression ratios (gzip, xz), homogenization scores (BLEU, ROUGE-L \citep{lin2004rouge}), template rate \citep{shaib2024templates}, chamfer distance, and self-repetition.
For exactness tasks we report eval pass rate, test pass rate, and parse validity.

\paragraph{Statistical protocol.}
All confidence intervals are 95\% bootstrap CIs (1500--2500 resampling trials).
Significance uses paired permutation tests (4000--6000 sign-flip trials, two-sided).

\subsection{Open-ended diversity}
\label{sec:exp_creative}

\begin{figure}[t]
\centering
\includegraphics[width=\textwidth]{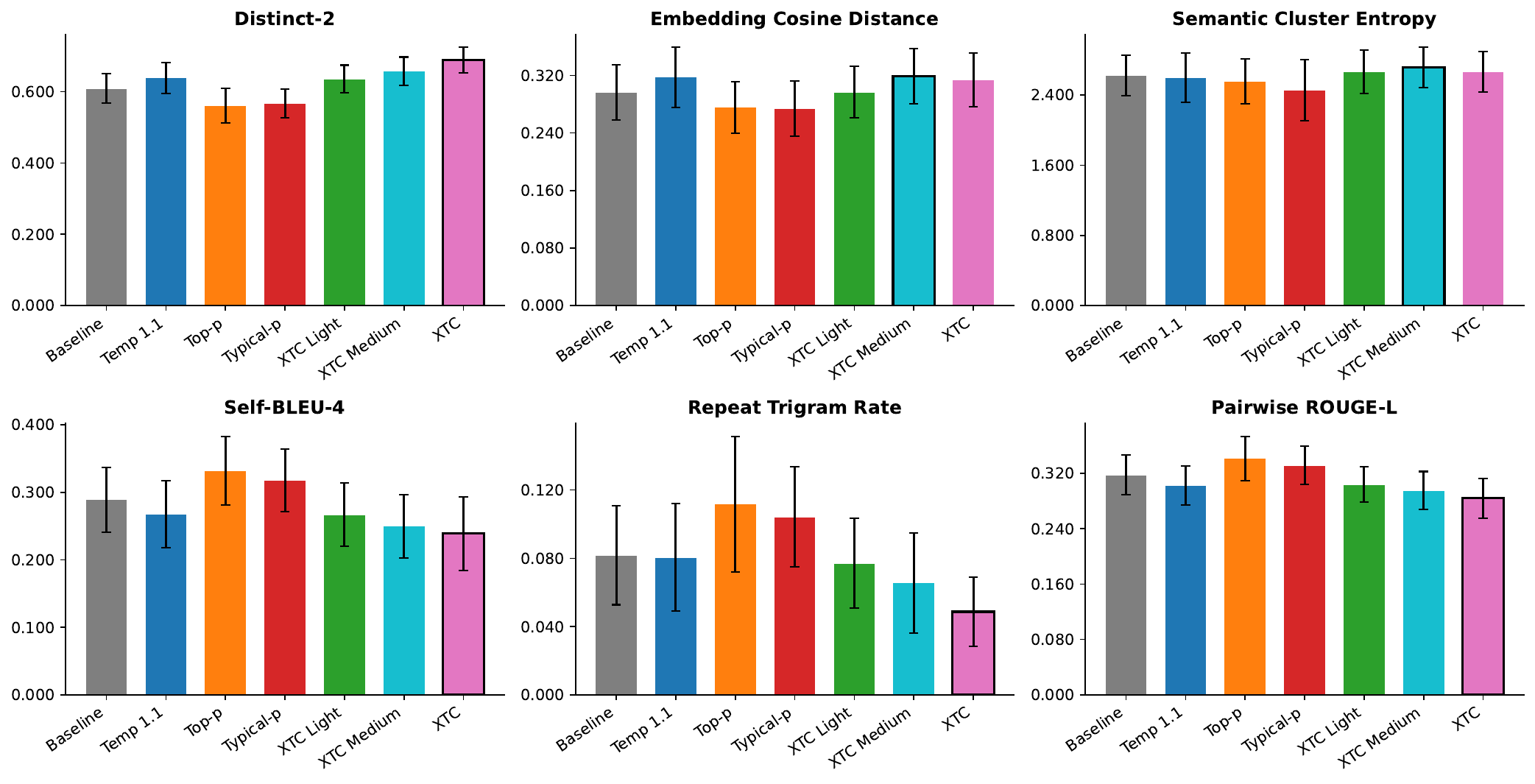}
\caption{Main creative evaluation across six diversity and repetition metrics (24 prompts, 8 samples each, Gemma~3 27B q4). Error bars show 95\% bootstrap CIs. \xtc{} ($\rho{=}1.0$, $\tau{=}0.1$) achieves the best score on \emph{every} metric, outperforming the baseline as well as temperature, top-$p$, and typical-$p$ on all diversity measures while simultaneously reducing repetition. Top-$p$ and typical-$p$ at $0.95$ \emph{degrade} diversity relative to baseline on Distinct-2, confirming that tail truncation can worsen head-dominance problems. Distinct-2, Embedding Cosine Distance, and Semantic Cluster Entropy: higher is better. Self-BLEU-4, Repeat Trigram Rate, Pairwise ROUGE-L: lower is better.}
\label{fig:creative_bars}
\end{figure}

Figure~\ref{fig:creative_bars} presents the main creative evaluation.
\xtc{} (at $\rho{=}1.0$, $\tau{=}0.1$) attains the best score on every metric, with Distinct-2 and repeat-trigram improvements significant at $p<0.001$ and Self-BLEU-4 at $p<0.01$ under paired permutation tests (forest plot in Appendix Figure~\ref{fig:forest}).
\xtc{} wins 24 of 24 prompts on Distinct-2 and 22 of 24 on repeat trigram rate (Appendix Figure~\ref{fig:wtl}).
Top-$p$ and typical-$p$ at 0.95 fall below baseline on Distinct-2, consistent with the head-ambiguity hypothesis that tail truncation is the wrong intervention in this regime.
The effect generalizes across all 12 prompt genres (Appendix Figure~\ref{fig:tags}), with the largest gains in dialogue, branding, and ideation.

\begin{figure}[t]
\centering
\includegraphics[width=\textwidth]{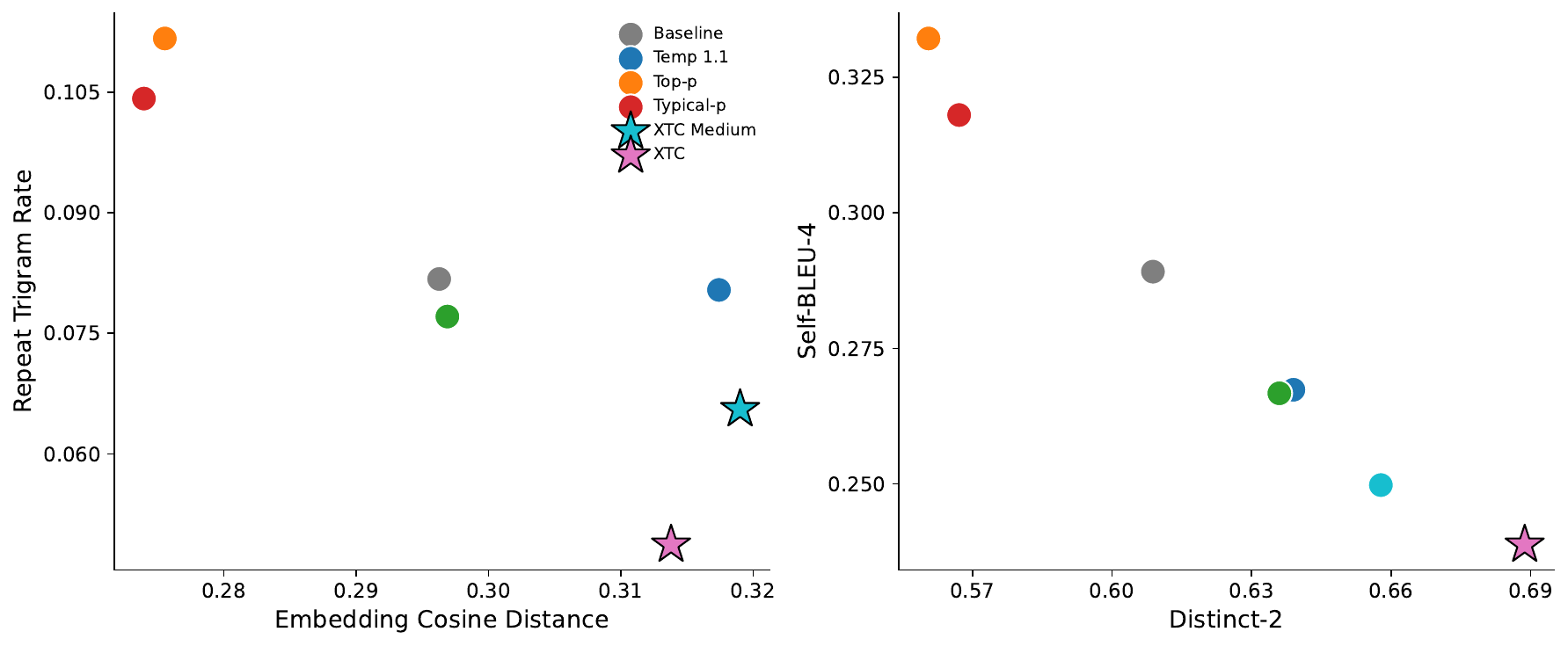}
\caption{Pareto frontiers on two diversity-vs-repetition planes (24 creative prompts, Gemma~3 27B q4). \textbf{Left:} Embedding cosine distance vs.\ repeat trigram rate. \textbf{Right:} Distinct-2 vs.\ Self-BLEU-4. Stars mark Pareto-optimal conditions. \xtc{} Medium and Strong dominate the frontier on both planes, while top-$p$ and typical-$p$ are Pareto-dominated by the baseline.}
\label{fig:pareto}
\end{figure}

The Pareto analysis (Figure~\ref{fig:pareto}) confirms that \xtc{} conditions dominate the diversity-vs-repetition frontier on both semantic and lexical planes.
A head-to-head against the strongest single-sampler comparator (Appendix Figure~\ref{fig:strong}) shows that the composition T=1.3 + \xtc{} ($\rho{=}0.75$, $\tau{=}0.05$) dominates T=1.3 alone on all four metrics.

A direct comparison against eta and min-$p$ baselines, including the min-$p$ + \xtc{} composition, appears in Appendix~\ref{app:strong_baselines} (Table~\ref{tab:strong_baselines}). \xtc{} alone outperforms all three tail-shaping baselines on every reported metric, and the min-$p$ + \xtc{} composition is the best overall, consistent with \xtc{} (a head control) and min-$p$ (a tail control) targeting orthogonal regions of the distribution.

A full sampler-composition study (Appendix~\ref{app:composition_body}) shows that pairing \xtc{} Medium with temperature, top-$p$, or a repetition penalty improves Distinct-2 and reduces Self-BLEU-4 in every pairing, with the $T{=}1.3 + \xtc{}$ composition attaining the best score on every metric in a ten-metric headline (Table~\ref{tab:extended}, Figure~\ref{fig:composition}).
Cross-model factorial replications appear in Appendix~\ref{app:interaction}.

\subsection{Quality evaluations}
\label{sec:exp_quality}

Diversity gains matter only if quality is preserved.
We attack this question along two complementary axes. The first is an LLM-as-judge evaluation scoring every generation on seven pre-registered criteria, with Claude Opus~4.7 as the primary judge and an OpenAI \texttt{gpt-4o} control as a cross-vendor check. The second is a blinded Amazon Mechanical Turk study with paid human annotators.

\subsubsection{LLM-as-judge with cross-vendor control}
\label{sec:exp_llm_judge}

We rate every Gemma~3 27B q4 creative-eval generation with Claude Opus~4.7 on seven 1--5 criteria, using a blinded single-output protocol with shuffled condition order and stripped condition labels \citep{zheng2023judging, krippendorff2011computing}.
LLM judges can exhibit same-family preference and verbosity biases \citep{wang2024pandalm, dubois2024length}, so we cross-validate with an OpenAI \texttt{gpt-4o} (2024-05-13) control judge under the identical rubric.

\begin{figure}[t]
\centering
\includegraphics[width=\textwidth]{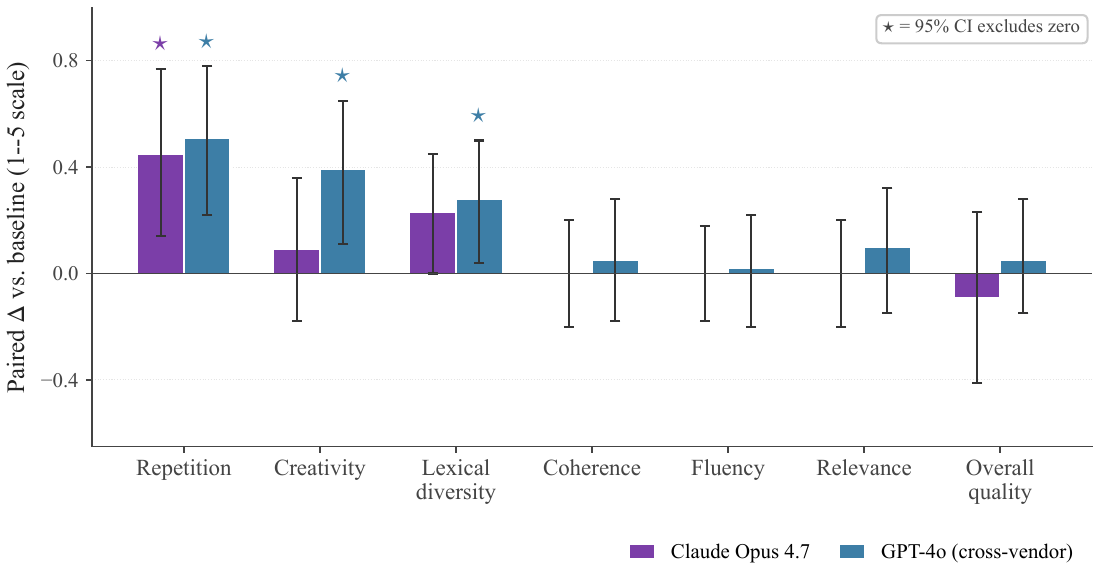}
\caption{Claude Opus~4.7 and an OpenAI \texttt{gpt-4o} cross-vendor control on \xtc{} vs.\ baseline (Gemma~3 27B q4, $n{=}24$ paired samples per condition). Bars are paired $\Delta$ on a 1--5 scale with 95\% bootstrap CIs. $\star$ marks intervals that exclude zero. Opus and GPT-4o agree that \xtc{} significantly improves repetition. GPT-4o additionally resolves significant gains on creativity and lexical diversity. Overall quality is null on both judges, supporting the no-degradation hypothesis. The cross-vendor agreement neutralizes same-family-bias concerns.}
\label{fig:judge_panel}
\end{figure}

Figure~\ref{fig:judge_panel} summarizes the Opus + GPT-4o result on Gemma~3 27B q4 (full numbers in Appendix Table~\ref{tab:judge_27b_q4}).
Both judges resolve a significant repetition improvement at \xtc{}, and GPT-4o additionally resolves creativity and lexical-diversity gains that Opus shows as directionally positive but inside its 95\% interval.
Overall-quality intervals bracket zero on both judges, supporting the no-degradation hypothesis.
The cross-vendor agreement on direction across every criterion neutralizes same-family-bias concerns common to single-judge LLM evaluations \citep{wang2024pandalm, dubois2024length}.
The same protocol replicates at scale on Llama 3.3 70B q4 (Figure~\ref{fig:judge_70b}), where Opus resolves significant \xtc{} creativity and lexical-diversity gains.
Overall-quality at 70B is again null.

\begin{figure}[t]
\centering
\includegraphics[width=0.85\textwidth]{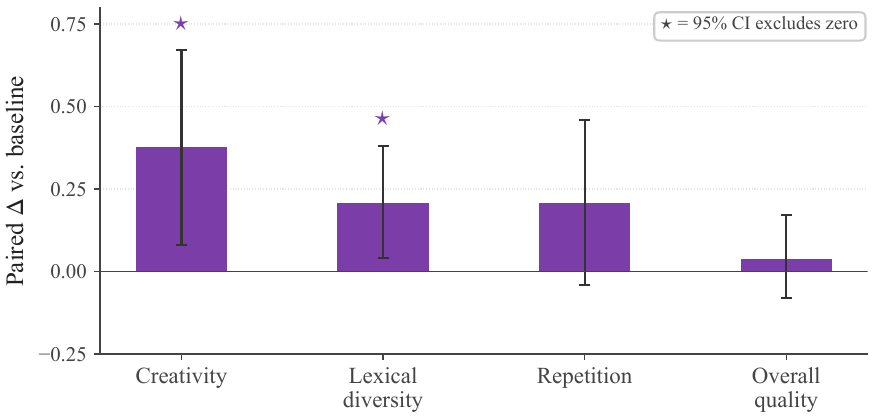}
\caption{Llama 3.3 70B q4 scaling validation: Claude Opus~4.7 LLM judge on \xtc{} vs.\ baseline ($n{=}24$ paired samples per condition, 95\% bootstrap CIs). Opus resolves significant positive \xtc{} effects on creativity and lexical diversity ($\star$). The repetition delta is positive but its CI grazes zero. Overall quality is null. The cross-vendor GPT-4o control was not run on this scaling configuration. The cross-vendor agreement is established on the 27B q4 panel (Figure~\ref{fig:judge_panel}).}
\label{fig:judge_70b}
\end{figure}

\subsubsection{Human evaluation (Amazon Mechanical Turk)}
\label{sec:exp_human_eval}

Human preference remains the gold standard for open-ended generation \citep{hashimoto2019unifying}.
The prompts come from the public Creative Writing Bench v3 prompt set released by Sam Peach as part of EQ-Bench\footnote{\url{https://github.com/EQ-bench/creative-writing-bench/blob/main/data/creative_writing_prompts_v3.json}} \citep{creative-writing-bench-v3}.
The set contains 33 base writing prompts, each paired with multiple seed modifiers (genre, tone, scenario tags) that condition the requested continuation.
We sampled 100 prompt+seed combinations uniformly across the 33 base prompts and generated paired responses using the Baseline ($T{=}1.0$) and \xtc{} ($\rho{=}1.0$, $\tau{=}0.1$) on Gemma~3 27B q4.
We recruited 150 Master-qualified AMT workers (\$15/hr, above local minimum wage) to blindly rate the pairs.
Each pair was scored by 3 independent annotators with A/B order randomized per pair, on two pre-registered axes: \emph{Creativity \& Interestingness} and \emph{General Quality \& Fluency}.

\begin{figure}[t]
\centering
\includegraphics[width=\textwidth]{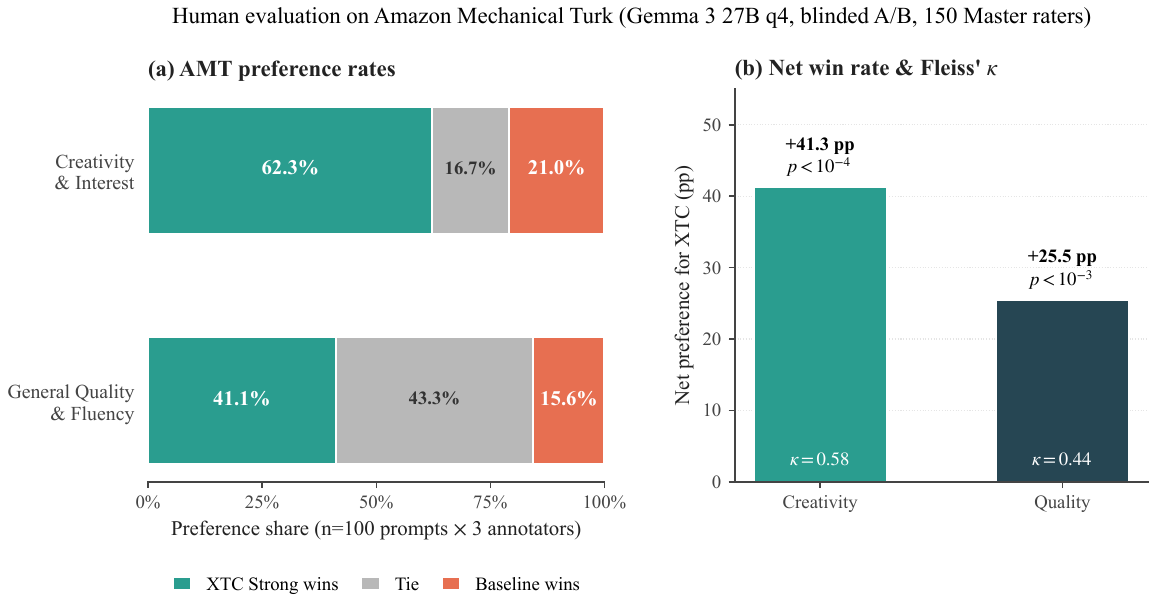}
\caption{Blinded AMT human evaluation on Gemma~3 27B q4 ($n{=}100$ prompts, 3 annotators per pair, 150 Master-qualified workers).
\textbf{(a)} Stacked preference rates: \xtc{} wins on creativity in 62.3\% of head-to-head comparisons against the baseline (vs.\ 21.0\% baseline wins, 16.7\% ties), and ties or beats baseline in 84.4\% of quality comparisons.
\textbf{(b)} Net preference (\xtc{}-minus-baseline) and Fleiss' $\kappa$ inter-annotator agreement: the creativity gain is significant at $p < 10^{-4}$ (two-sided binomial) with moderate agreement ($\kappa{=}0.58$), and the quality gain is significant at $p < 10^{-3}$ with $\kappa{=}0.44$.}
\label{fig:human_eval}
\end{figure}

The human evaluation corroborates the automatic metrics and the LLM panel (Figure~\ref{fig:human_eval}).
\xtc{} is the preferred response for Creativity by a wide margin, and ties or beats the baseline on Quality in the large majority of comparisons.
A two-sided binomial test confirms both axes at $p<10^{-3}$ or better with moderate inter-annotator agreement.
The head-aware exclusions therefore translate to human-perceptible stylistic diversity with quality preserved.

\subsection{General quality \& instruction following (IFEval)}
\label{sec:exp_ifeval}

Diversity-promoting decoders are often charged with degrading instruction-following \citep{hewitt2022truncation, holtzman2020curious}.
We quantify \xtc{}'s safe operating boundaries with the Instruction Following Evaluation (IFEval) benchmark \citep{zhou2023instructionfollowing}.

\begin{figure}[t]
\centering
\includegraphics[width=\textwidth]{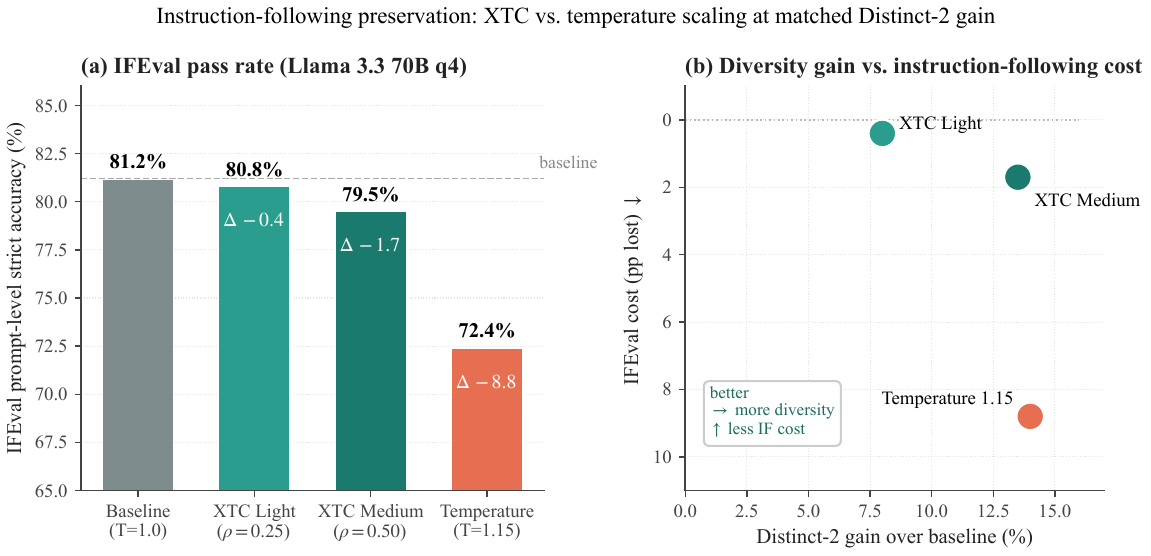}
\caption{Instruction-following preservation on Llama~3.3 70B q4.
\textbf{(a)} IFEval prompt-level strict accuracy across four conditions. \xtc{} Light ($\rho{=}0.25$, $\tau{=}0.10$) shifts accuracy by only $-0.4$ pp, and \xtc{} Medium ($\rho{=}0.50$, $\tau{=}0.10$) by $-1.7$ pp. A temperature setting (T=1.15) chosen to deliver a comparable Distinct-2 gain collapses IFEval by $-8.8$ pp.
\textbf{(b)} Diversity gain vs.\ instruction-following cost: \xtc{} Medium delivers $\sim$5$\times$ more Distinct-2 per IFEval-point lost than temperature scaling, locating \xtc{} in the upper region of the diversity-vs-fidelity Pareto plane.}
\label{fig:ifeval}
\end{figure}

On Llama~3.3 70B q4 (Figure~\ref{fig:ifeval}), \xtc{} Light and \xtc{} Medium shift IFEval prompt-level strict accuracy by less than 2 pp from baseline ($-0.4$ and $-1.7$ pp respectively), while a temperature setting ($T{=}1.15$) chosen to deliver an equivalent Distinct-2 gain collapses IFEval by $-8.8$ pp. Per-condition numbers and the IFEval-cost-per-Distinct-2 ratios appear in Appendix~\ref{app:strong_baselines} (Table~\ref{tab:ifeval}).

The sparse-in-time formulation explains the gap: deterministic instruction-critical tokens (where the head dominates at $>$90\% probability) lie outside \xtc{}'s eligible set and pass through untouched, while temperature rescales every step regardless of head ambiguity (Figure~\ref{fig:ifeval}b).

\section{Conclusion}

We introduced \xtc{}, a head-aware decoding rule that intervenes when a language model places substantial probability on several plausible continuations while overweighting the most dominant one.
The operator acts on the head of the distribution, fires only when at least two tokens clear an absolute plausibility floor, and is otherwise inactive.
Renormalization after exclusion is the KL-minimizing projection onto the restricted support \citep{csiszar1975divergence}, preserving relative odds among surviving tokens.
Empirically, 60 experiments across four model families show that \xtc{} improves the diversity-repetition Pareto frontier on creative generation, composes additively with temperature and repetition penalties, and generalizes across quantization levels and architectures (Section~\ref{sec:experiments}, Appendix~\ref{app:crossmodel}).
An Opus + GPT-4o cross-vendor LLM panel, a blinded AMT human study, and an IFEval characterization on Llama 3.3 70B q4 converge on the same finding: head-aware exclusion produces human-perceptible stylistic diversity at a fraction of the instruction-following cost paid by matched-Distinct-2 temperature scaling.
Limitations, broader impacts, and practical safeguards appear in Appendix~\ref{sec:discussion}.

\bibliographystyle{plainnat}
\bibliography{refs}

@inproceedings{fan2018hierarchical,
  author    = {Fan, Angela and Lewis, Mike and Dauphin, Yann},
  title     = {Hierarchical Neural Story Generation},
  booktitle = {Proceedings of the 56th Annual Meeting of the Association for Computational Linguistics (Volume 1: Long Papers)},
  pages     = {889--898},
  year      = {2018}
}

@inproceedings{holtzman2020curious,
  author    = {Holtzman, Ari and Buys, Jan and Du, Li and Forbes, Maxwell and Choi, Yejin},
  title     = {The Curious Case of Neural Text Degeneration},
  booktitle = {International Conference on Learning Representations},
  year      = {2020}
}

@inproceedings{hewitt2022truncation,
  author    = {Hewitt, John and Manning, Christopher D. and Liang, Percy},
  title     = {Truncation Sampling as Language Model Desmoothing},
  booktitle = {Findings of the Association for Computational Linguistics: {EMNLP} 2022},
  pages     = {3414--3427},
  year      = {2022}
}

@inproceedings{basu2021mirostat,
  author    = {Basu, Sourya and Ramachandran, Govardana Sachithanandam and Keskar, Nitish Shirish and Varshney, Lav R.},
  title     = {Mirostat: A Neural Text Decoding Algorithm that Directly Controls Perplexity},
  booktitle = {International Conference on Learning Representations},
  year      = {2021}
}

@article{meister2023typical,
  author  = {Meister, Clara and Pimentel, Tiago and Wiher, Gian and Cotterell, Ryan},
  title   = {Locally Typical Sampling},
  journal = {Transactions of the Association for Computational Linguistics},
  volume  = {11},
  pages   = {102--121},
  year    = {2023}
}

@inproceedings{ICLR2025_afa5f124,
  author    = {Nguyen, Minh Nhat and Baker, Andrew and Neo, Clement and Roush, Allen and Kirsch, Andreas and Shwartz-Ziv, Ravid},
  title     = {Turning Up the Heat: Min-p Sampling for Creative and Coherent {LLM} Outputs},
  booktitle = {International Conference on Learning Representations},
  pages     = {70333--70366},
  year      = {2025}
}

@inproceedings{li2023contrastive,
  author    = {Li, Xiang Lisa and Holtzman, Ari and Fried, Daniel and Liang, Percy and Eisner, Jason and Hashimoto, Tatsunori B. and Zettlemoyer, Luke and Lewis, Mike},
  title     = {Contrastive Decoding: Open-Ended Text Generation as Optimization},
  booktitle = {Proceedings of the 61st Annual Meeting of the Association for Computational Linguistics (Volume 1: Long Papers)},
  pages     = {12286--12312},
  year      = {2023}
}

@inproceedings{su2022contrastive,
  author    = {Su, Yixuan and Lan, Tian and Wang, Yan and Yogatama, Dani and Kong, Lingpeng and Collier, Nigel},
  title     = {A Contrastive Framework for Neural Text Generation},
  booktitle = {Advances in Neural Information Processing Systems},
  volume    = {35},
  pages     = {21548--21561},
  year      = {2022}
}

@article{vijayakumar2018diverse,
  title={Diverse beam search: Decoding diverse solutions from neural sequence models},
  author={Vijayakumar, Ashwin K and Cogswell, Michael and Selvaraju, Ramprasath R and Sun, Qing and Lee, Stefan and Crandall, David and Batra, Dhruv},
  journal={arXiv preprint arXiv:1610.02424},
  year={2016}
}

@inproceedings{welleck2020neural,
  author    = {Welleck, Sean and Kulikov, Ilia and Roller, Stephen and Dinan, Emily and Cho, Kyunghyun and Weston, Jason},
  title     = {Neural Text Generation with Unlikelihood Training},
  booktitle = {International Conference on Learning Representations},
  year      = {2020}
}

@inproceedings{chuang2023dola,
  author    = {Chuang, Yung-Sung and Xie, Yujia and Luo, Hongyin and Kim, Yoon and Glass, James R. and He, Pengcheng},
  title     = {{DoLa}: Decoding by Contrasting Layers Improves Factuality in Large Language Models},
  booktitle = {International Conference on Learning Representations},
  year      = {2024}
}

@inproceedings{Yang_2021,
  author    = {Yang, Kevin and Klein, Dan},
  title     = {{FUDGE}: Controlled Text Generation With Future Discriminators},
  booktitle = {Proceedings of the 2021 Conference of the North American Chapter of the Association for Computational Linguistics: Human Language Technologies},
  pages     = {3511--3535},
  year      = {2021}
}

@inproceedings{qin2022cold,
  author    = {Qin, Lianhui and Welleck, Sean and Khashabi, Daniel and Choi, Yejin},
  title     = {{COLD} Decoding: Energy-Based Constrained Text Generation with {Langevin} Dynamics},
  booktitle = {Advances in Neural Information Processing Systems},
  volume    = {35},
  year      = {2022}
}

@article{keskar2019ctrl,
  author  = {Keskar, Nitish Shirish and McCann, Bryan and Varshney, Lav R. and Xiong, Caiming and Socher, Richard},
  title   = {{CTRL}: A Conditional Transformer Language Model for Controllable Generation},
  journal = {arXiv preprint arXiv:1909.05858},
  year    = {2019}
}

@article{wiher2022decoding,
  author  = {Wiher, Gian and Meister, Clara and Cotterell, Ryan},
  title   = {On Decoding Strategies for Neural Text Generators},
  journal = {Transactions of the Association for Computational Linguistics},
  volume  = {10},
  pages   = {997--1012},
  year    = {2022}
}

@inproceedings{li2016diversity,
  author    = {Li, Jiwei and Galley, Michel and Brockett, Chris and Gao, Jianfeng and Dolan, Bill},
  title     = {A Diversity-Promoting Objective Function for Neural Conversation Models},
  booktitle = {Proceedings of the 2016 Conference of the North American Chapter of the Association for Computational Linguistics: Human Language Technologies},
  pages     = {110--119},
  year      = {2016}
}

@inproceedings{shaib2024diversity,
  author    = {Shaib, Chantal and Barrow, Joe and Sun, Jiuding and Siu, Alexa F. and Wallace, Byron C. and Nenkova, Ani},
  title     = {Standardizing the Measurement of Text Diversity: A Tool and Comparative Analysis},
  booktitle = {Proceedings of {IJCNLP-AACL} 2025: System Demonstrations},
  year      = {2025}
}

@inproceedings{shaib2024templates,
  author    = {Shaib, Chantal and Elazar, Yanai and Li, Junyi Jessy and Wallace, Byron C.},
  title     = {Detection and Measurement of Syntactic Templates in Generated Text},
  booktitle = {Proceedings of the 2024 Conference on Empirical Methods in Natural Language Processing},
  pages     = {6416--6431},
  year      = {2024}
}

@inproceedings{padmakumar2024contentdiversity,
  author    = {Padmakumar, Vishakh and He, He},
  title     = {Does Writing with Language Models Reduce Content Diversity?},
  booktitle = {International Conference on Learning Representations},
  year      = {2024}
}

@article{doshi2024generative,
  author  = {Doshi, Anil R. and Hauser, Oliver P.},
  title   = {Generative {AI} Enhances Individual Creativity but Reduces the Collective Diversity of Novel Content},
  journal = {Science Advances},
  volume  = {10},
  number  = {28},
  pages   = {eadn5290},
  year    = {2024}
}

@inproceedings{anderson2024homogenization,
  author    = {Anderson, Barrett R. and Shah, Jash Hemant and Kreminski, Max},
  title     = {Homogenization Effects of Large Language Models on Human Creative Ideation},
  booktitle = {Proceedings of the 16th Conference on Creativity \& Cognition},
  year      = {2024}
}

@article{sourati2025homogenizing,
  author  = {Sourati, Zhivar and Ziabari, Alireza S. and Dehghani, Morteza},
  title   = {The Homogenizing Effect of Large Language Models on Human Expression and Thought},
  journal = {arXiv preprint arXiv:2508.01491},
  year    = {2025}
}

@inproceedings{yuan2022wordcraft,
  author    = {Yuan, Ann and Coenen, Andy and Reif, Emily and Ippolito, Daphne},
  title     = {Wordcraft: Story Writing With Large Language Models},
  booktitle = {27th International Conference on Intelligent User Interfaces},
  pages     = {841--852},
  year      = {2022}
}

@article{ippolito2023creative,
  author  = {Ippolito, Daphne and Yuan, Ann and Coenen, Andy and Burnam, Sehmon},
  title   = {Creative Writing with an {AI}-Powered Writing Assistant: Perspectives from Professional Writers},
  journal = {arXiv preprint arXiv:2211.05030},
  year    = {2022}
}

@inproceedings{kirk2024understanding,
  author    = {Kirk, Robert and Mediratta, Ishita and Nalmpantis, Christoforos and Luketina, Jelena and Hambro, Eric and Grefenstette, Edward and Raileanu, Roberta},
  title     = {Understanding the Effects of {RLHF} on {LLM} Generalisation and Diversity},
  booktitle = {International Conference on Learning Representations},
  year      = {2024}
}

@misc{creative-writing-bench-v3,
  author       = {Paech, Samuel J.},
  title        = {{EQ-Bench} Creative Writing Benchmark v3},
  howpublished = {\url{https://github.com/EQ-bench/creative-writing-bench}},
  year         = {2025},
  note         = {GitHub repository}
}

@article{shannon1948mathematical,
  author  = {Shannon, Claude E.},
  title   = {A Mathematical Theory of Communication},
  journal = {The Bell System Technical Journal},
  volume  = {27},
  number  = {3},
  pages   = {379--423},
  year    = {1948}
}

@article{csiszar1975divergence,
  author  = {Csisz{\'a}r, Imre},
  title   = {{I}-Divergence Geometry of Probability Distributions and Minimization Problems},
  journal = {The Annals of Probability},
  volume  = {3},
  number  = {1},
  pages   = {146--158},
  year    = {1975}
}

@book{cover2006elements,
  author    = {Cover, Thomas M. and Thomas, Joy A.},
  title     = {Elements of Information Theory},
  publisher = {John Wiley \& Sons},
  address   = {Hoboken, NJ},
  edition   = {2nd},
  year      = {2006}
}

@article{ackley1985learning,
  title={A learning algorithm for Boltzmann machines},
  author={Ackley, David H and Hinton, Geoffrey E and Sejnowski, Terrence J},
  journal={Cognitive science},
  volume={9},
  number={1},
  pages={147--169},
  year={1985},
  publisher={Elsevier}
}

@article{bengio2003neural,
  author  = {Bengio, Yoshua and Ducharme, R{\'e}jean and Vincent, Pascal and Jauvin, Christian},
  title   = {A Neural Probabilistic Language Model},
  journal = {Journal of Machine Learning Research},
  volume  = {3},
  pages   = {1137--1155},
  year    = {2003}
}

@inproceedings{vaswani2017attention,
  author    = {Vaswani, Ashish and Shazeer, Noam and Parmar, Niki and Uszkoreit, Jakob and Jones, Llion and Gomez, Aidan N. and Kaiser, Lukasz and Polosukhin, Illia},
  title     = {Attention Is All You Need},
  booktitle = {Advances in Neural Information Processing Systems},
  volume    = {30},
  pages     = {5998--6008},
  year      = {2017}
}

@article{radford2019language,
  author  = {Radford, Alec and Wu, Jeffrey and Child, Rewon and Luan, David and Amodei, Dario and Sutskever, Ilya},
  title   = {Language Models are Unsupervised Multitask Learners},
  journal = {{OpenAI} Technical Report},
  year    = {2019}
}

@inproceedings{brown2020language,
  author    = {Brown, Tom B. and Mann, Benjamin and Ryder, Nick and Subbiah, Melanie and Kaplan, Jared and Dhariwal, Prafulla and Neelakantan, Arvind and Shyam, Pranav and Sastry, Girish and Askell, Amanda and others},
  title     = {Language Models are Few-Shot Learners},
  booktitle = {Advances in Neural Information Processing Systems},
  volume    = {33},
  pages     = {1877--1901},
  year      = {2020}
}

@article{touvron2023llama2,
  author  = {Touvron, Hugo and Martin, Louis and Stone, Kevin and Albert, Peter and Almahairi, Amjad and Babaei, Yasmine and Bashlykov, Nikolay and Batra, Soumya and Bhargava, Prajjwal and Bhosale, Shruti and others},
  title   = {{Llama} 2: Open Foundation and Fine-Tuned Chat Models},
  journal = {arXiv preprint arXiv:2307.09288},
  year    = {2023}
}

@article{gemma2024,
  author  = {Mesnard, Thomas and Hardin, Cassidy and Dadashi, Robert and Bhupatiraju, Surya and Pathak, Shreya and Sifre, Laurent and Rivi{\`e}re, Morgane and Kale, Mihir Sanjay and Love, Juliette and Tafti, Pouya and others},
  title   = {{Gemma}: Open Models Based on {Gemini} Research and Technology},
  journal = {arXiv preprint arXiv:2403.08295},
  year    = {2024}
}

@article{deepseek2025r1,
  author  = {{DeepSeek-AI} and Guo, Daya and Yang, Dejian and Zhang, Haowei and Song, Junxiao and Wang, Peiyi and Zhu, Qihao and Xu, Runxin and Zhang, Ruoyu and Ma, Shirong and others},
  title   = {{DeepSeek-R1}: Incentivizing Reasoning Capability in {LLMs} via Reinforcement Learning},
  journal = {arXiv preprint arXiv:2501.12948},
  year    = {2025}
}

@inproceedings{ouyang2022training,
  author    = {Ouyang, Long and Wu, Jeffrey and Jiang, Xu and Almeida, Diogo and Wainwright, Carroll L. and Mishkin, Pamela and Zhang, Chong and Agarwal, Sandhini and Slama, Katarina and Ray, Alex and others},
  title     = {Training Language Models to Follow Instructions with Human Feedback},
  booktitle = {Advances in Neural Information Processing Systems},
  volume    = {35},
  pages     = {27730--27744},
  year      = {2022}
}

@article{bai2022training,
  author  = {Bai, Yuntao and Jones, Andy and Ndousse, Kamal and Askell, Amanda and Chen, Anna and DasSarma, Nova and Drain, Dawn and Fort, Stanislav and Ganguli, Deep and Henighan, Tom and others},
  title   = {Training a Helpful and Harmless Assistant with Reinforcement Learning from Human Feedback},
  journal = {arXiv preprint arXiv:2204.05862},
  year    = {2022}
}

@inproceedings{lee2022deduplicating,
  author    = {Lee, Katherine and Ippolito, Daphne and Nystrom, Andrew and Zhang, Chiyuan and Eck, Douglas and Callison-Burch, Chris and Carlini, Nicholas},
  title     = {Deduplicating Training Data Makes Language Models Better},
  booktitle = {Proceedings of the 60th Annual Meeting of the Association for Computational Linguistics (Volume 1: Long Papers)},
  pages     = {8424--8445},
  year      = {2022}
}

@inproceedings{dettmers2022gptq,
  author    = {Frantar, Elias and Ashkboos, Saleh and Hoefler, Torsten and Alistarh, Dan},
  title     = {{GPTQ}: Accurate Post-Training Quantization for Generative Pre-trained Transformers},
  booktitle = {International Conference on Learning Representations},
  year      = {2023}
}

@article{lin2024awq,
  title={Awq: Activation-aware weight quantization for on-device llm compression and acceleration},
  author={Lin, Ji and Tang, Jiaming and Tang, Haotian and Yang, Shang and Chen, Wei-Ming and Wang, Wei-Chen and Xiao, Guangxuan and Dang, Xingyu and Gan, Chuang and Han, Song},
  journal={Proceedings of machine learning and systems},
  volume={6},
  pages={87--100},
  year={2024}
}

@inproceedings{papineni2002bleu,
  author    = {Papineni, Kishore and Roukos, Salim and Ward, Todd and Zhu, Wei-Jing},
  title     = {{BLEU}: A Method for Automatic Evaluation of Machine Translation},
  booktitle = {Proceedings of the 40th Annual Meeting of the Association for Computational Linguistics},
  pages     = {311--318},
  year      = {2002}
}

@inproceedings{lin2004rouge,
  author    = {Lin, Chin-Yew},
  title     = {{ROUGE}: A Package for Automatic Evaluation of Summaries},
  booktitle = {Text Summarization Branches Out},
  pages     = {74--81},
  year      = {2004}
}

@inproceedings{zhang2020bertscore,
  author    = {Zhang, Tianyi and Kishore, Varsha and Wu, Felix and Weinberger, Kilian Q. and Artzi, Yoav},
  title     = {{BERTScore}: Evaluating Text Generation with {BERT}},
  booktitle = {International Conference on Learning Representations},
  year      = {2020}
}

@inproceedings{pillutla2021mauve,
  author    = {Pillutla, Krishna and Swayamdipta, Swabha and Zellers, Rowan and Thickstun, John and Welleck, Sean and Choi, Yejin and Harchaoui, Zaid},
  title     = {{MAUVE}: Measuring the Gap Between Neural Text and Human Text using Divergence Frontiers},
  booktitle = {Advances in Neural Information Processing Systems},
  volume    = {34},
  pages     = {4816--4828},
  year      = {2021}
}

@inproceedings{zhu2018texygen,
  author    = {Zhu, Yaoming and Lu, Sidi and Zheng, Lei and Guo, Jiaxian and Zhang, Weinan and Wang, Jun and Yu, Yong},
  title     = {Texygen: A Benchmarking Platform for Text Generation Models},
  booktitle = {The 41st International {ACM} {SIGIR} Conference on Research and Development in Information Retrieval},
  pages     = {1097--1100},
  year      = {2018}
}

@inproceedings{hashimoto2019unifying,
  author    = {Hashimoto, Tatsunori B. and Zhang, Hugh and Liang, Percy},
  title     = {Unifying Human and Statistical Evaluation for Natural Language Generation},
  booktitle = {Proceedings of the 2019 Conference of the North American Chapter of the Association for Computational Linguistics: Human Language Technologies},
  pages     = {1689--1701},
  year      = {2019}
}

@inproceedings{zheng2023judging,
  author    = {Zheng, Lianmin and Chiang, Wei-Lin and Sheng, Ying and Zhuang, Siyuan and Wu, Zhanghao and Zhuang, Yonghao and Lin, Zi and Li, Zhuohan and Li, Dacheng and Xing, Eric and others},
  title     = {Judging {LLM}-as-a-Judge with {MT-Bench} and {Chatbot Arena}},
  booktitle = {Advances in Neural Information Processing Systems},
  volume    = {36},
  pages     = {46595--46623},
  year      = {2023}
}

@inproceedings{chiang2023closer,
  author    = {Chiang, Cheng-Han and Lee, Hung-yi},
  title     = {Can Large Language Models Be an Alternative to Human Evaluations?},
  booktitle = {Proceedings of the 61st Annual Meeting of the Association for Computational Linguistics (Volume 1: Long Papers)},
  pages     = {15607--15631},
  year      = {2023}
}

@inproceedings{dubois2024length,
  author    = {Dubois, Yann and Galambosi, Bal{\'a}zs and Liang, Percy and Hashimoto, Tatsunori B.},
  title     = {Length-Controlled {AlpacaEval}: A Simple Way to Debias Automatic Evaluators},
  booktitle = {Conference on Language Modeling},
  year      = {2024}
}

@inproceedings{wang2024pandalm,
  author    = {Wang, Yidong and Yu, Zhuohao and Zeng, Zhengran and Yang, Linyi and Wang, Cunxiang and Chen, Hao and Jiang, Chaoya and Xie, Rui and Wang, Jindong and Xie, Xing and others},
  title     = {{PandaLM}: An Automatic Evaluation Benchmark for {LLM} Instruction Tuning Optimization},
  booktitle = {International Conference on Learning Representations},
  year      = {2024}
}

@inproceedings{li2024leveraging,
  author    = {Li, Zhen and Xu, Xiaohan and Shen, Tao and Xu, Can and Gu, Jia-Chen and Lai, Yuxuan and Tao, Chongyang and Ma, Shuai},
  title     = {Leveraging Large Language Models for {NLG} Evaluation: Advances and Challenges},
  booktitle = {Proceedings of the 2024 Conference on Empirical Methods in Natural Language Processing},
  pages     = {16028--16045},
  year      = {2024}
}

@inproceedings{renze2024effect,
  author    = {Renze, Matthew and Guven, Erhan},
  title     = {The Effect of Sampling Temperature on Problem Solving in Large Language Models},
  booktitle = {Findings of the Association for Computational Linguistics: {EMNLP} 2024},
  year      = {2024}
}

@article{zhou2023instructionfollowing,
  author  = {Zhou, Jeffrey and Lu, Tianjian and Mishra, Swaroop and Brahma, Siddhartha and Basu, Sujoy and Luan, Yi and Zhou, Denny and Hou, Le},
  title   = {Instruction-Following Evaluation for Large Language Models},
  journal = {arXiv preprint arXiv:2311.07911},
  year    = {2023}
}

@techreport{krippendorff2011computing,
  author      = {Krippendorff, Klaus},
  title       = {Computing {Krippendorff's} Alpha-Reliability},
  institution = {Annenberg School for Communication, University of Pennsylvania},
  year        = {2011}
}

@article{gatt2018survey,
  author  = {Gatt, Albert and Krahmer, Emiel},
  title   = {Survey of the State of the Art in Natural Language Generation: Core Tasks, Applications and Evaluation},
  journal = {Journal of Artificial Intelligence Research},
  volume  = {61},
  pages   = {65--170},
  year    = {2018}
}

@article{celikyilmaz2020evaluation,
  author  = {Celikyilmaz, Asli and Clark, Elizabeth and Gao, Jianfeng},
  title   = {Evaluation of Text Generation: A Survey},
  journal = {arXiv preprint arXiv:2006.14799},
  year    = {2020}
}

@article{minaee2024large,
  author  = {Minaee, Shervin and Mikolov, Tomas and Nikzad, Narjes and Chenaghlu, Meysam and Socher, Richard and Amatriain, Xavier and Gao, Jianfeng},
  title   = {Large Language Models: A Survey},
  journal = {arXiv preprint arXiv:2402.06196},
  year    = {2024}
}

@article{zhao2023survey,
  author  = {Zhao, Wayne Xin and Zhou, Kun and Li, Junyi and Tang, Tianyi and Dong, Zican and Hou, Yupeng and Zhang, Beichen and Min, Yingqian and Zhang, Junjie and Liu, Peiyu and others},
  title   = {A Survey of Large Language Models},
  journal = {arXiv preprint arXiv:2303.18223},
  year    = {2023}
}

@article{li2022pretrained,
  author  = {Li, Junyi and Tang, Tianyi and Zhao, Wayne Xin and Nie, Jian-Yun and Wen, Ji-Rong},
  title   = {Pre-trained Language Models for Text Generation: A Survey},
  journal = {{ACM} Computing Surveys},
  volume  = {56},
  number  = {9},
  pages   = {1--39},
  year    = {2024}
}

\newpage
\appendix

\section{Formal definition of \xtc{}}
\label{app:formal}

This appendix provides the complete mathematical specification of \xtc{}.
The intuitive description and algorithm appear in Section~\ref{sec:method}.

\subsection{Notation and eligible set}

Consider an autoregressive language model at step $t$ with a next-token distribution $\pt$ over a vocabulary $\Vocab$.
Let $\tau \in (0,1)$ be an absolute eligibility threshold and $\rho \in [0,1]$ an intervention probability.
The eligible set is
\begin{equation}
\Et \;=\; \{v \in \Vocab : \pt(v) \ge \tau\}.
\end{equation}

\subsection{Transformed distribution}

Let
\begin{equation}
u_t \;=\; \argminprob_{v \in \Et} \pt(v),
\end{equation}
with any deterministic tie-breaking rule.
Define the removed set $\Rt = \Et \setminus \{u_t\}$.
When \xtc{} activates ($|\Et| \ge 2$ and the Bernoulli$(\rho)$ draw fires), the transformed distribution $\qt$ is
\begin{equation}
\label{eq:xtc}
\qt(v) =
\begin{cases}
0, & v \in \Rt,\\[4pt]
\dfrac{\pt(v)}{1 - \sum_{r \in \Rt} \pt(r)}, & v \notin \Rt.
\end{cases}
\end{equation}
When \xtc{} does not activate, $\qt = \pt$.

\subsection{Formal proposition statements}

\begin{proposition}[No-op criterion]
\label{prop:noop}
If $|\Et| < 2$, or if the Bernoulli intervention variable does not fire, then $\qt = \pt$.
\end{proposition}

\begin{proposition}[Relative-odds invariance]
\label{prop:odds}
When \xtc{} activates and removes $\Rt$, for any surviving tokens $v,w \notin \Rt$ with $\pt(w) > 0$,
$\qt(v)/\qt(w) = \pt(v)/\pt(w)$.
\end{proposition}

\begin{proposition}[Expected removed mass]
\label{prop:mass}
Let $M_t(\tau) = \sum_{v \in \Et} \pt(v) - \min_{v \in \Et}\pt(v)$ when $|\Et| \ge 2$ and $0$ otherwise.
Then $\E[\textnormal{removed mass at step } t] = \rho \, M_t(\tau)$.
\end{proposition}

\begin{proposition}[Threshold monotonicity]
\label{prop:threshold}
$\tau_2 \ge \tau_1$ implies $E_t(\tau_2) \subseteq E_t(\tau_1)$, so increasing the threshold can only weakly decrease the number and mass of removable tokens.
\end{proposition}

\begin{proposition}[Restricted-support information projection]
\label{prop:iprojection}
Fix an activation event and let $S_t = \Vocab \setminus \Rt$.
Among all distributions $q$ supported on $S_t$, $\qt$ uniquely minimizes $\mathrm{KL}(q \,\|\, \pt)$.
\end{proposition}

\subsection{Structural properties}
\label{app:properties}

\xtc{} has five structural properties that distinguish it from existing samplers.

\textbf{No-op criterion.}
If fewer than two tokens exceed the threshold, or if the Bernoulli draw does not fire, \xtc{} returns the input distribution unchanged (Proposition~\ref{prop:noop}).
The operator is therefore inactive by default and intervenes only when the head is genuinely ambiguous.

\textbf{Relative-odds invariance.}
When \xtc{} activates, it preserves the ratio of probabilities among all surviving tokens (Proposition~\ref{prop:odds}).
The intervention is a sparse support edit followed by renormalization, with no distortion of relative preferences among the tokens that remain.

\textbf{Expected removed mass.}
The expected mass removed at a given step equals $\rho$ times the total eligible mass minus the minimum eligible probability (Proposition~\ref{prop:mass}), isolating the operator's effective strength at each step.

\textbf{Threshold monotonicity.}
Increasing $\tau$ can only shrink the eligible set (Proposition~\ref{prop:threshold}), weakly decreasing both the number of removable tokens and the removable mass.

\textbf{KL-projection interpretation.}
When \xtc{} activates, the resulting distribution is the unique minimizer of $\mathrm{KL}(q \,\|\, \pt)$ among all distributions supported on the surviving tokens (Proposition~\ref{prop:iprojection}, citing \citet{csiszar1975divergence, cover2006elements, shannon1948mathematical}).
This gives the support edit a principled information-geometric status.

\paragraph{Head multiplicity differs from entropy.}
\xtc{} responds to a different signal than entropy-based decoding controls.
Entropy is a global summary of dispersion, while \xtc{} asks whether multiple tokens each individually exceed an absolute floor.
A distribution can have modest entropy yet contain two strong head tokens (activating \xtc{}), or high entropy from a diffuse tail with only one token above threshold (leaving \xtc{} inactive).
\xtc{} therefore intervenes only when the model already signals the existence of several plausible local branches.
\FloatBarrier

\section{Proofs}
\label{app:proofs}

\begin{proof}[Proof of Proposition~\ref{prop:noop}]
By the definition of \xtc{}, the operator first checks whether $|\Et| \ge 2$.
If not, it returns the input distribution unchanged.
If $|\Et| \ge 2$, the operator next samples a Bernoulli random variable with success probability $\rho$.
If that variable is $0$, the operator again returns the input distribution unchanged.
Therefore in either case $\qt = \pt$.
\end{proof}

\begin{proof}[Proof of Proposition~\ref{prop:odds}]
Assume \xtc{} activates and removes $\Rt$.
For any surviving token $v \notin \Rt$,
\[
\qt(v) = \frac{\pt(v)}{Z_t},
\qquad
Z_t = 1 - \sum_{r \in \Rt}\pt(r).
\]
Similarly, for any surviving token $w \notin \Rt$ with $\pt(w) > 0$, $\qt(w) = \pt(w)/Z_t$.
Taking the ratio gives $\qt(v)/\qt(w) = \pt(v)/\pt(w)$.
\end{proof}

\begin{proof}[Proof of Proposition~\ref{prop:mass}]
If $|\Et| < 2$, the operator never removes any token, so the removed mass is $0 = \rho M_t(\tau)$.
Now suppose $|\Et| \ge 2$.
Conditional on activation, \xtc{} removes exactly the set $\Rt = \Et \setminus \{u_t\}$.
The removed mass is
\[
\sum_{r \in \Rt}\pt(r) = \sum_{v \in \Et}\pt(v) - \min_{v \in \Et}\pt(v) = M_t(\tau).
\]
Since activation occurs with probability $\rho$, the unconditional expectation is $\rho M_t(\tau)$.
\end{proof}

\begin{proof}[Proof of Proposition~\ref{prop:threshold}]
Take any token $v \in E_t(\tau_2)$.
By definition $\pt(v) \ge \tau_2 \ge \tau_1$, so $v \in E_t(\tau_1)$.
Hence $E_t(\tau_2) \subseteq E_t(\tau_1)$ and the removable set and mass can only weakly shrink.
\end{proof}

\begin{proof}[Proof of Proposition~\ref{prop:iprojection}]
Fix the surviving support $S_t = \Vocab \setminus \Rt$ and consider
$\mathcal{Q}(S_t) = \{ q : \sum_{v \in S_t} q(v) = 1,\; q(v) \ge 0,\; q(v)=0 \text{ for } v \notin S_t \}$.
For any $q \in \mathcal{Q}(S_t)$,
\[
\mathrm{KL}(q \,\|\, \pt) = \sum_{v \in S_t} q(v)\log \frac{q(v)}{\pt(v)}.
\]
Minimizing subject to the simplex constraint via Lagrange multipliers yields $q(v) = c \, \pt(v)$ for all $v \in S_t$, where $c = 1/\sum_{v \in S_t} \pt(v)$.
Therefore $q(v) = \pt(v)/(1 - \sum_{r \in \Rt}\pt(r))$, which is exactly $\qt(v)$.
Uniqueness follows from strict convexity of $\mathrm{KL}$.
\end{proof}

\subsection{Diagnostic quantity}

Proposition~\ref{prop:mass} suggests a natural diagnostic: the \emph{expected removed mass} $\rho M_t(\tau)$.
This scalar measures how interventionist \xtc{} is at a given decoding step, distinguishing between steps where the operator swaps two near-tied head options and steps where it removes most of the head mass.
\FloatBarrier

\section{Limitations, broader impacts, and outlook}
\label{sec:discussion}

\subsection{Limitations}
\label{sec:limitations}

\xtc{} is not a universally beneficial decoding rule.
Its central bias is to move generation away from the most probable eligible continuation, and that bias can be counterproductive when the highest-probability token is exactly the one the task needs.

First, \xtc{} can hurt tasks that demand exactness rather than stylistic variation, such as extraction, constrained formatting, code generation, mathematical reasoning with brittle intermediate states, or safety-critical instruction following.
The IFEval characterization (Section~\ref{sec:exp_ifeval}) and the structured-quality safety sweeps (Appendix~\ref{app:results}) quantify this.
On code tasks, the eval pass rate degrades significantly above $\rho = 0.15$.
On structured extraction, the boundary is $\rho = 0.20$.
Deployments on alignment-critical workloads should stay within the operating boundaries identified in Section~\ref{sec:exp_ifeval}.

Second, \xtc{} is sensitive to parameterization.
The threshold $\tau$ is defined in absolute probability terms, so the same value can behave differently across model scales, tokenizer granularities, and upstream sampler stacks \citep{wiher2022decoding}.
A threshold that is conservative for one model may be intervention-heavy for another.
The activation probability $\rho$ should likewise be treated as task-dependent rather than as a universal knob.
Open-ended generation may tolerate much more intervention than extraction or code.
While our cross-model experiments (Appendix~\ref{app:crossmodel}) show that the qualitative effect direction is consistent, the optimal operating point may differ across models.

Third, \xtc{} interacts with formatting and termination behavior.
If newline, end-of-sequence, indentation, or schema-critical delimiter tokens enter the eligible set, removing them may destabilize structured outputs.
Practical deployments should therefore either gate \xtc{} by task type or protect a small set of termination and formatting tokens, as described in Algorithm~\ref{alg:xtc}.

Fourth, our empirical evaluation, while spanning three primary model families in the 12B--27B range plus a 70B scaling point (Llama 3.3 q4), is conducted entirely on quantized open-weight models \citep{dettmers2022gptq, lin2024awq}.
A broader evaluation across additional mixture-of-experts architectures \citep{minaee2024large, brown2020language, touvron2023llama2}, state-space models, and non-English languages would strengthen the generality claims.

Fifth, \xtc{} is a local decoding rule, not a substitute for model quality, calibration, or training-time anti-degeneration methods \citep{welleck2020neural, ouyang2022training}.
It can redirect a model among already plausible next tokens, but it cannot create competence that the base model does not have.
In particular, it cannot address corpus-level homogenization that originates in training data overlap or RLHF reward hacking \citep{anderson2024homogenization, kirk2024understanding, bai2022training}.

On DeepSeek R1, the chat template injects a native reasoning-trace prefix. We re-ran the cross-model creative configuration with \texttt{max\_tokens}=8192 and stripped the \texttt{<think>...</think>} prefix before computing metrics, so the DeepSeek row in Table~\ref{tab:crossmodel} reflects post-reasoning answer text rather than reasoning-trace content.
The paired-delta direction and significance hold under this corrected basis (Distinct-2 $+13.6\%$ [$+10.2$,$+16.5$], repeat trigram $-0.017$ [$-0.024$,$-0.009$]), confirming that \xtc{}'s effect is not an artifact of the reasoning-trace measurement window.

\subsection{Broader impacts}
\label{sec:broader}

The positive case for \xtc{} is straightforward.
Many real uses of language models benefit from controlled diversity: creative writing support \citep{yuan2022wordcraft, ippolito2023creative}, ideation, exploratory dialogue systems, synthetic prompt generation, and evaluation of model uncertainty under alternative continuations.
A lightweight head-aware sampler could make these systems less templated without requiring model retraining.
The AMT human study (Section~\ref{sec:exp_human_eval}) provides direct evidence that the diversity gain is human-perceptible and does not come at the cost of fluency.

There are also clear downside risks.
Any method that increases variety in generated language can also make harmful generations less repetitive and more varied, including spam, deception, or disinformation.
\xtc{} does not change the underlying model's capabilities, but it can alter the style and breadth of outputs in ways that may improve adversarial misuse.
Responsible deployment therefore requires the same safeguards expected of other generation-time controls: domain-appropriate safety filters, task gating, and evaluation on harmful-use prompts in addition to benign ones.

\subsection{Practical safeguards}

For high-fidelity deployments, three safeguards are especially important.
\begin{itemize}[leftmargin=1.25em]
\item \textbf{Task gating.} Reserve \xtc{} primarily for open-ended or diversity-seeking modes rather than structured extraction or high-stakes factual tasks. The IFEval analysis (Section~\ref{sec:exp_ifeval}) provides a quantitative basis for setting per-task $\rho$.
\item \textbf{Protected tokens.} Exempt end-of-sequence, newline, and schema-critical tokens from removal, or skip activation when they would be removed.
\item \textbf{Transparent logging.} Record activation statistics, eligible-set sizes, and removed mass so that \xtc{} behavior can be inspected rather than treated as an opaque randomness knob.
\end{itemize}

\subsection{Outlook}
\xtc{} exposes a part of the decoding design space that is largely orthogonal to tail truncation and global entropy control: the head-ambiguity regime where the model is already uncertain among strong options but still defaults to the safest one.
By formalizing this regime and providing an operator that targets it specifically, we hope to encourage further exploration of head-aware decoding strategies \citep{minaee2024large, zhao2023survey}.
All code, evaluation infrastructure, and experimental configurations are publicly available at \url{<redacted for anonymity>}.
\FloatBarrier

\section{Cross-model generalization}
\label{app:crossmodel}

\begin{table}[h]
\centering
\small
\setlength{\tabcolsep}{5pt}
\begin{tabular}{@{}lccc c ccc@{}}
\toprule
& \multicolumn{3}{c}{Distinct-2 $\uparrow$} && \multicolumn{3}{c}{Repeat trigram rate $\downarrow$} \\
\cmidrule{2-4} \cmidrule{6-8}
Model family & Base & \xtc{} & $\Delta\%$ [95\% CI] && Base & \xtc{} & $\Delta$ [95\% CI] \\
\midrule
Gemma 3 12B (q6)$^{\dagger}$    & 0.704 & 0.784 & $+11.4$ \scriptsize{[$+7.4,+17.6$]}  && 0.047 & 0.025 & $-0.022$ \scriptsize{[$-0.042,-0.006$]} \\
DeepSeek R1 14B (q6)$^{\ddagger}$ & 0.615 & 0.751 & $+13.6$ \scriptsize{[$+10.2,+16.5$]} && 0.038 & 0.021 & $-0.017$ \scriptsize{[$-0.024,-0.009$]} \\
Gemma 3 27B (q4)         & 0.609 & 0.689 & $+13.1$ \scriptsize{[$+11.0,+18.1$]} && 0.082 & 0.049 & $-0.033$ \scriptsize{[$-0.047,-0.021$]} \\
Llama 3.3 70B (q4)       & 0.568 & 0.653 & $+15.1$ \scriptsize{[$+12.4,+19.8$]} && 0.040 & 0.029 & $-0.011$ \scriptsize{[$-0.027,-0.000$]} \\
\midrule
\textit{Average}         & \textit{0.624} & \textit{0.719} & \textit{$+13.3$} && \textit{0.052} & \textit{0.031} & \textit{$-0.021$} \\
\bottomrule
\end{tabular}
\caption{Cross-model generalization across four model families spanning three architectures (Gemma~3, Qwen-derived DeepSeek R1, Llama 3.3) and parameter counts from 12B to 70B. Means and 95\% paired-bootstrap CIs are computed over per-prompt deltas (24 prompts for the 27B and 70B runs, 10 for the 12B/14B cross-model runs). Distinct-2 is reported as paired percent change. Repeat trigram rate is reported as the absolute paired difference, since the small baseline values make percent change unstable. Every model shows a Distinct-2 CI that excludes zero. Repeat trigram improvements are also significant on three of the four families. The Llama 70B repeat-trigram CI just touches zero, consistent with that model's already-low baseline repetition. The \xtc{} strongest condition reported per family sits within the paper's operating region: $\rho \in \{0.75, 1.0\}$, $\tau \in \{0.05, 0.1\}$. \\
$^{\dagger}$ Cross-model run uses 10 prompts per condition (vs.\ 24 for the primary 27B/70B runs). \\
$^{\ddagger}$ DeepSeek R1's chat template injects a native reasoning-trace prefix on every response. We re-ran this configuration with \texttt{max\_tokens}=8192 and stripped the \texttt{<think>...</think>} prefix before metric computation, so Distinct-2 and repeat trigram rate here reflect the post-reasoning answer text.}
\label{tab:crossmodel}
\end{table}

Across four families, Distinct-2 increases monotonically by 11--15\% and repeat trigram rate decreases by 27--47\%.
Llama 3.3 70B q4 shows the largest Distinct-2 improvement (+15.1\%), extending the effect to a third architecture family and a parameter count 2.6$\times$ larger than Gemma 3 27B.
The repetition-reduction spread (27--47\%) is explained by the ceiling effect.
The 70B baseline already has a repeat trigram rate of 0.040, below Gemma 27B's 0.082, so there is less absolute headroom for further reduction even though Distinct-2 gains grow with scale.
Gemma~3 12B q6, with the highest baseline diversity (0.704) among the tested models, still benefits substantially from \xtc{} (+11.4\%), demonstrating that the mechanism provides gains regardless of starting diversity level.

\begin{figure}[h]
\centering
\includegraphics[width=\textwidth]{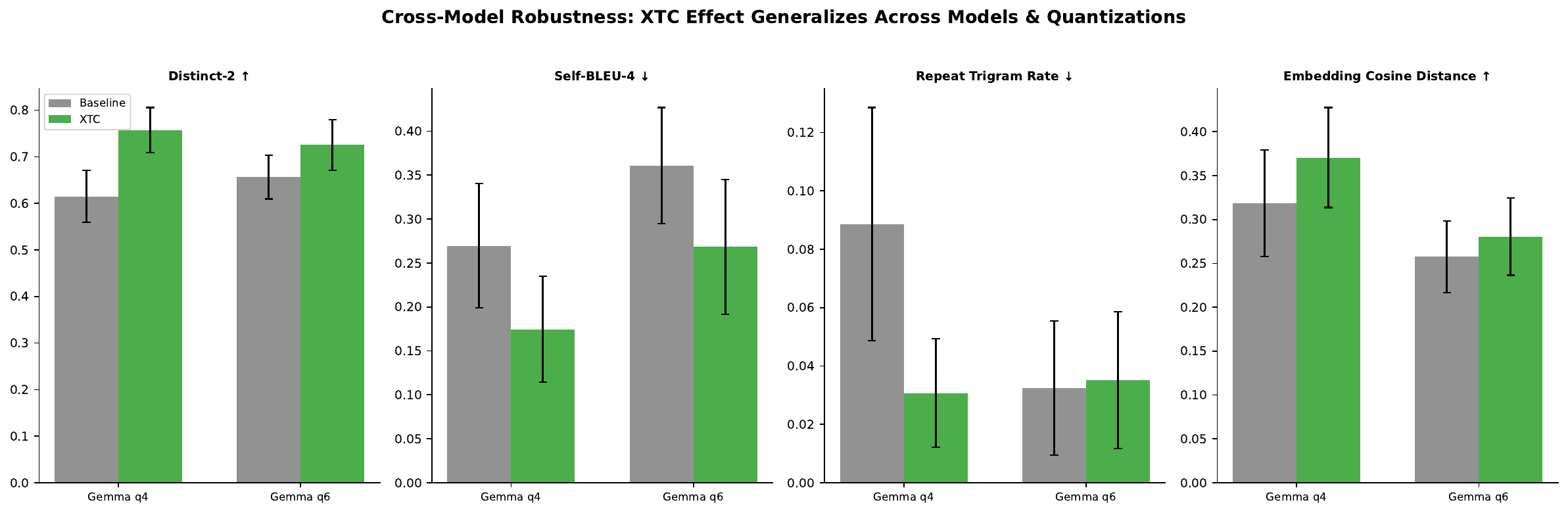}
\caption{Cross-model generalization: baseline vs.\ \xtc{} on the three primary model families (Gemma~3 27B q4, Gemma~3 12B q6, DeepSeek R1 14B q6). The Llama 3.3 70B q4 scaling validation is shown separately in Figure~\ref{fig:scaling}.}
\label{fig:crossmodel}
\end{figure}

\begin{figure}[h]
\centering
\includegraphics[width=0.78\textwidth]{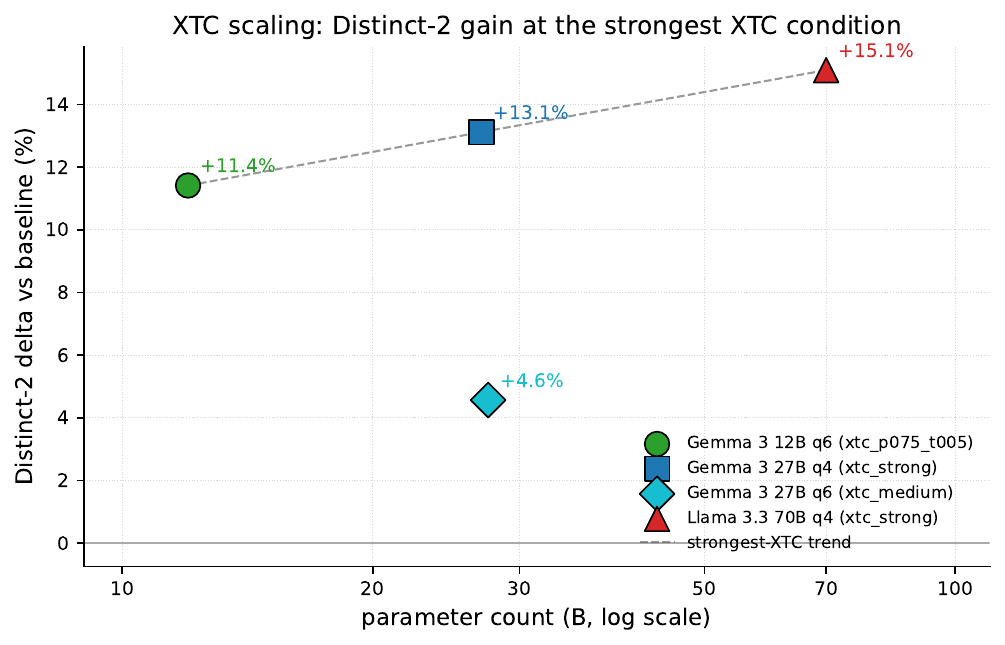}
\caption{\xtc{} scaling from 12B to 70B parameters on the strongest XTC condition available in each run. The strongest-XTC Distinct-2 delta is monotone increasing in parameter count on the three q4/q6 points plotted: Gemma~3 12B q6 ($\rho{=}0.75$, $\tau{=}0.05$) $+11.4\%$, Gemma~3 27B q4 xtc\_strong ($\rho{=}1.0$, $\tau{=}0.1$) $+13.1\%$, Llama 3.3 70B q4 xtc\_strong ($\rho{=}1.0$, $\tau{=}0.1$) $+15.1\%$. The dashed trend line connects the three strongest-XTC points.}
\label{fig:scaling}
\end{figure}

\begin{figure}[h]
\centering
\includegraphics[width=0.85\textwidth]{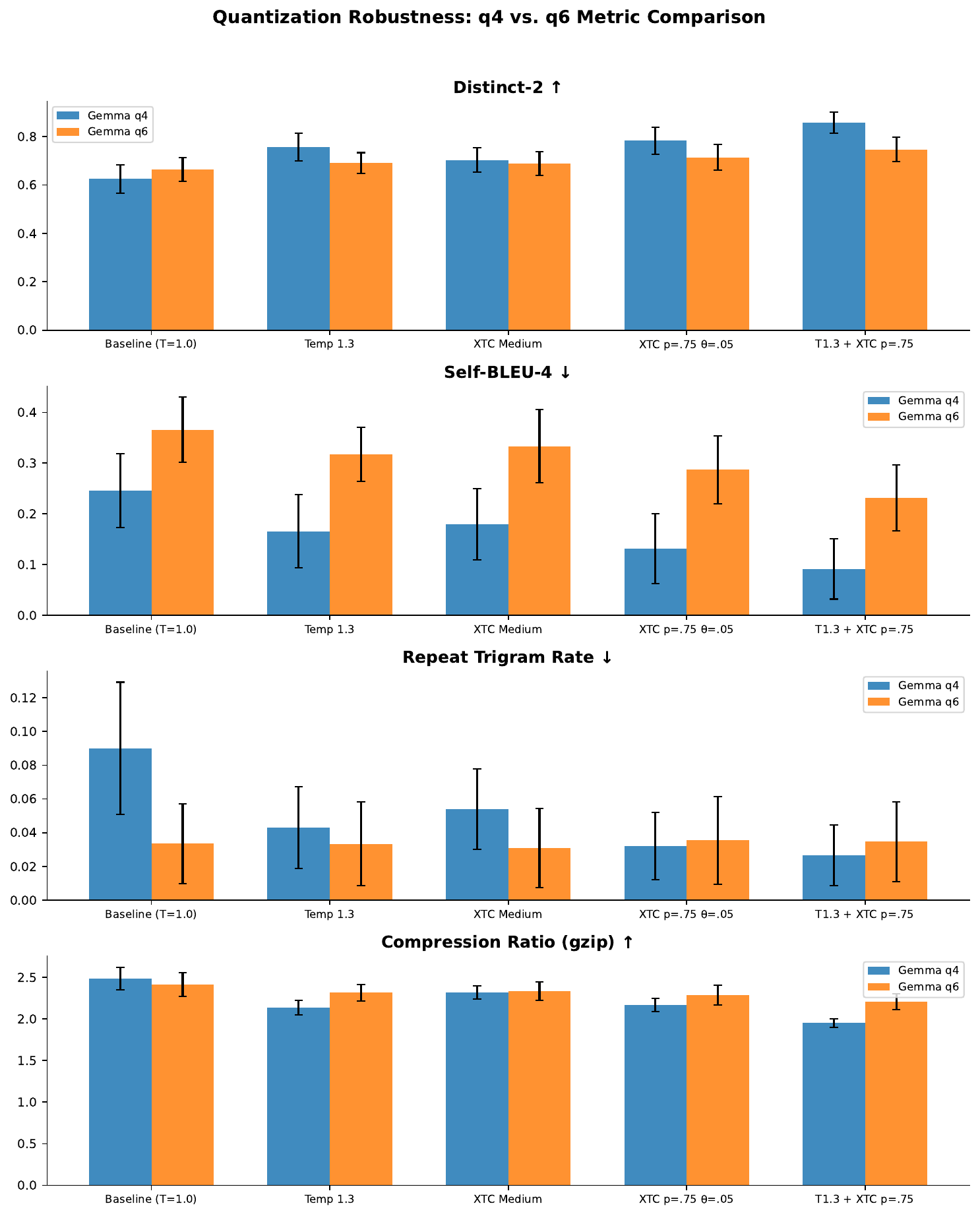}
\caption{Quantization robustness: Gemma~3 27B at q4 vs.\ q6. No qualitative divergence at any operating point.}
\label{fig:quant_compare}
\end{figure}

\begin{figure}[h]
\centering
\includegraphics[width=\textwidth]{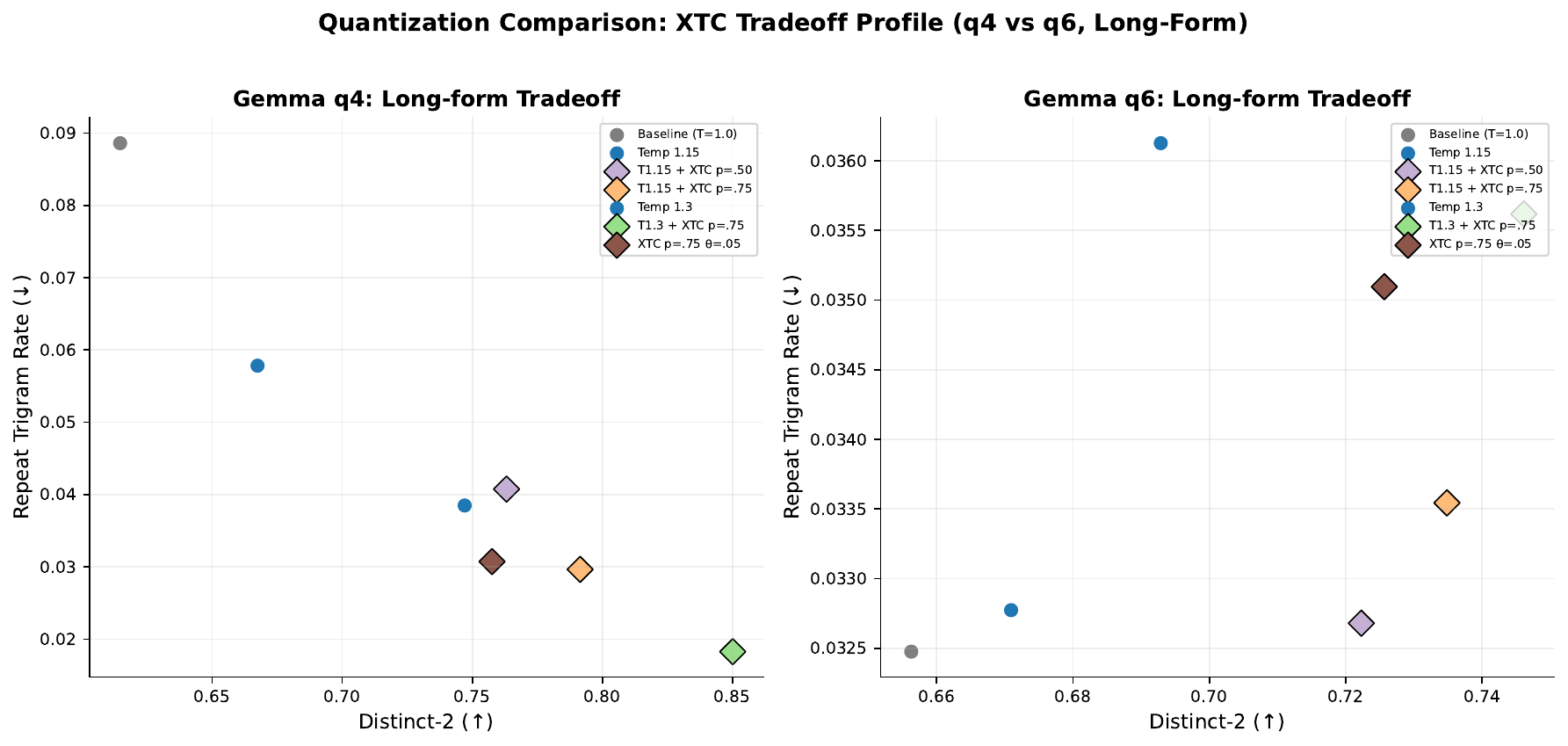}
\caption{Diversity vs.\ repetition tradeoff on long-form prompts for q4 and q6.}
\label{fig:q4q6_longform}
\end{figure}
\FloatBarrier

\subsection{DeepSeek R1 14B detailed results}

\begin{figure}[h]
\centering
\includegraphics[width=\textwidth]{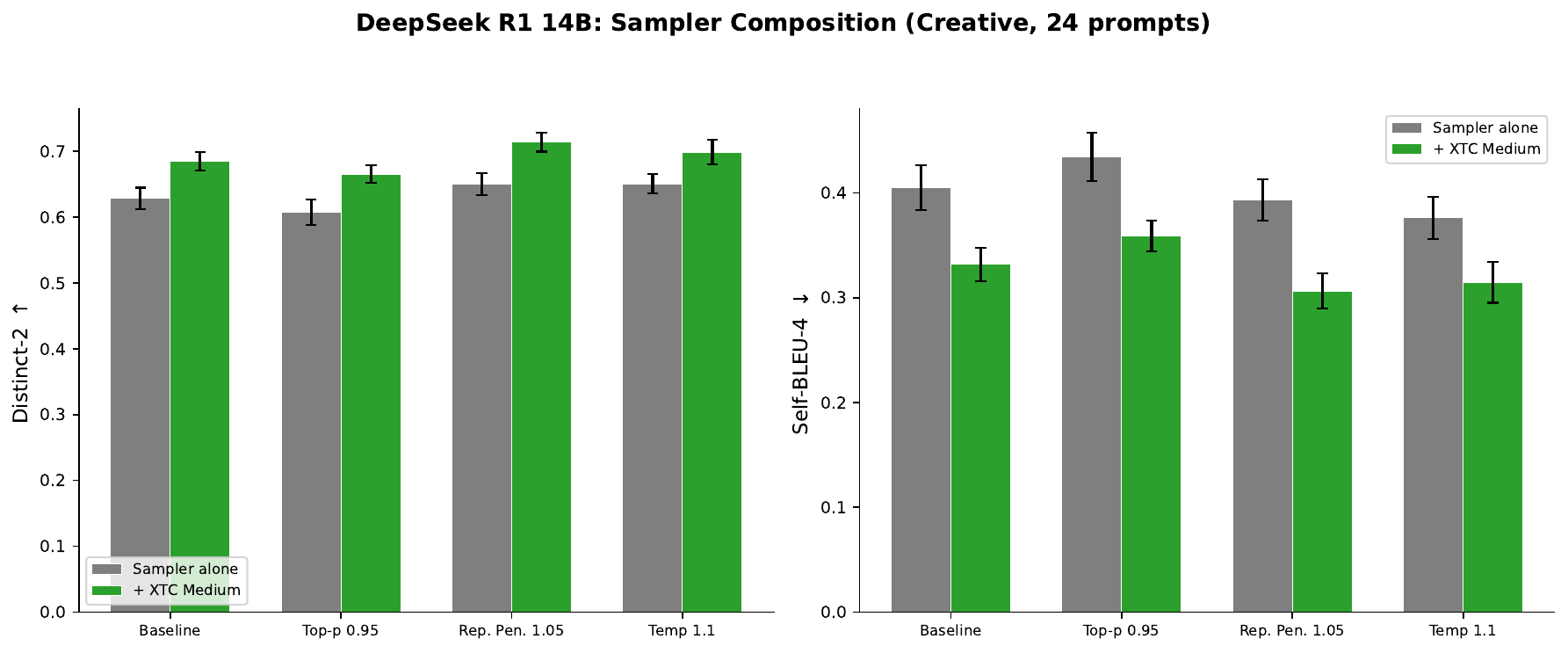}
\caption{DeepSeek R1 14B sampler composition on creative prompts.}
\label{fig:deepseek_creative_comp}
\end{figure}

\begin{figure}[h]
\centering
\includegraphics[width=\textwidth]{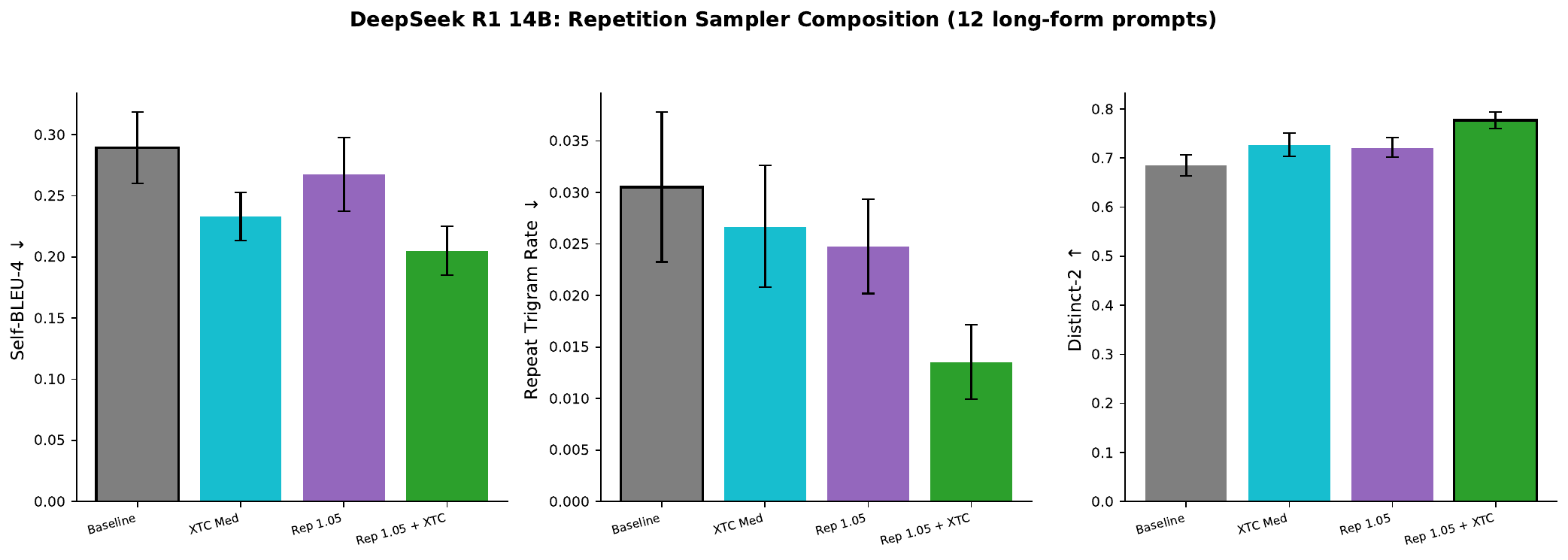}
\caption{DeepSeek R1 14B repetition composition.}
\label{fig:deepseek_rep_comp}
\end{figure}

\begin{figure}[h]
\centering
\includegraphics[width=\textwidth]{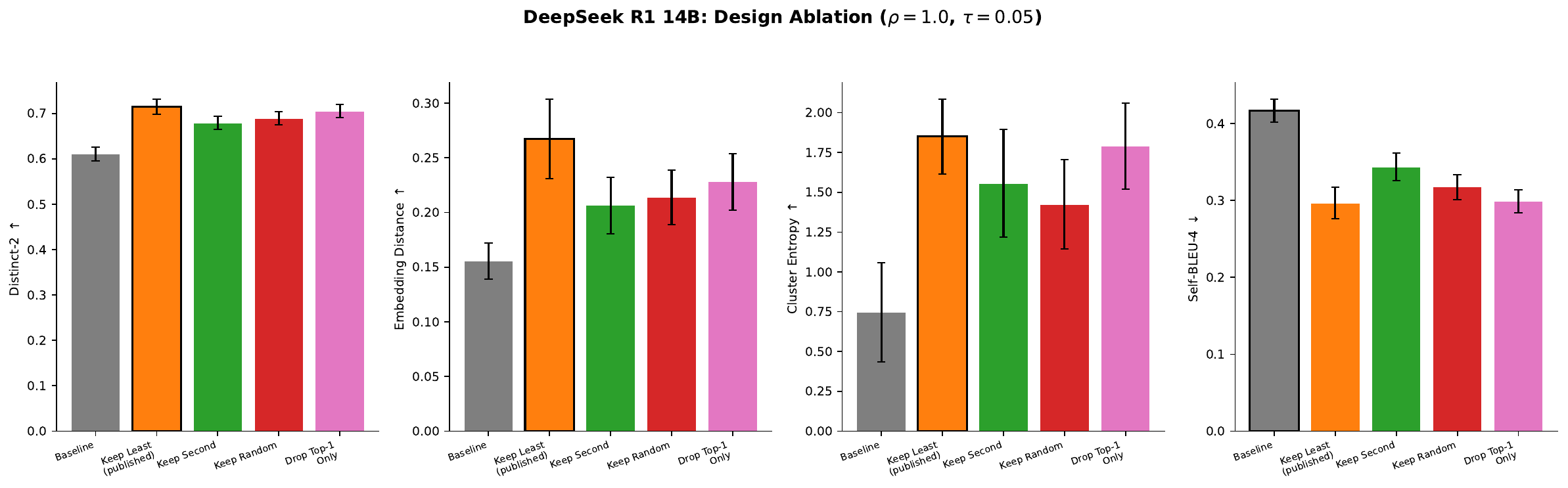}
\caption{Design ablation on DeepSeek R1 14B. Keep Least achieves the highest diversity across all four metrics.}
\label{fig:deepseek_ablation}
\end{figure}

\begin{figure}[h]
\centering
\includegraphics[width=\textwidth]{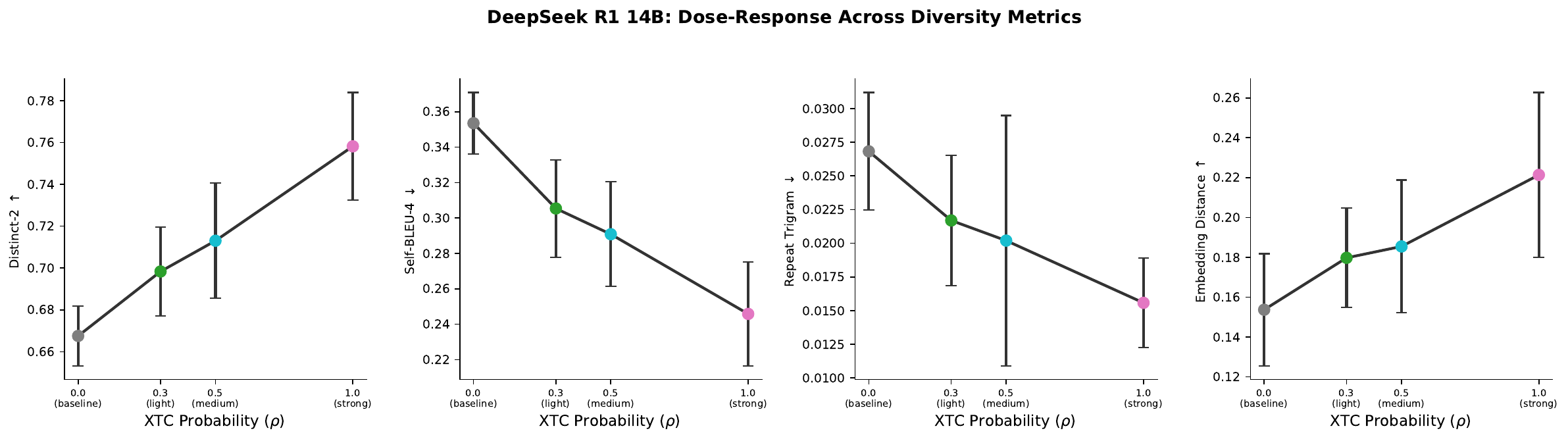}
\caption{Dose-response on DeepSeek R1 14B.}
\label{fig:deepseek_dose}
\end{figure}
\FloatBarrier

\subsection{Gemma 3 12B q6 detailed results}

\begin{figure}[h]
\centering
\includegraphics[width=\textwidth]{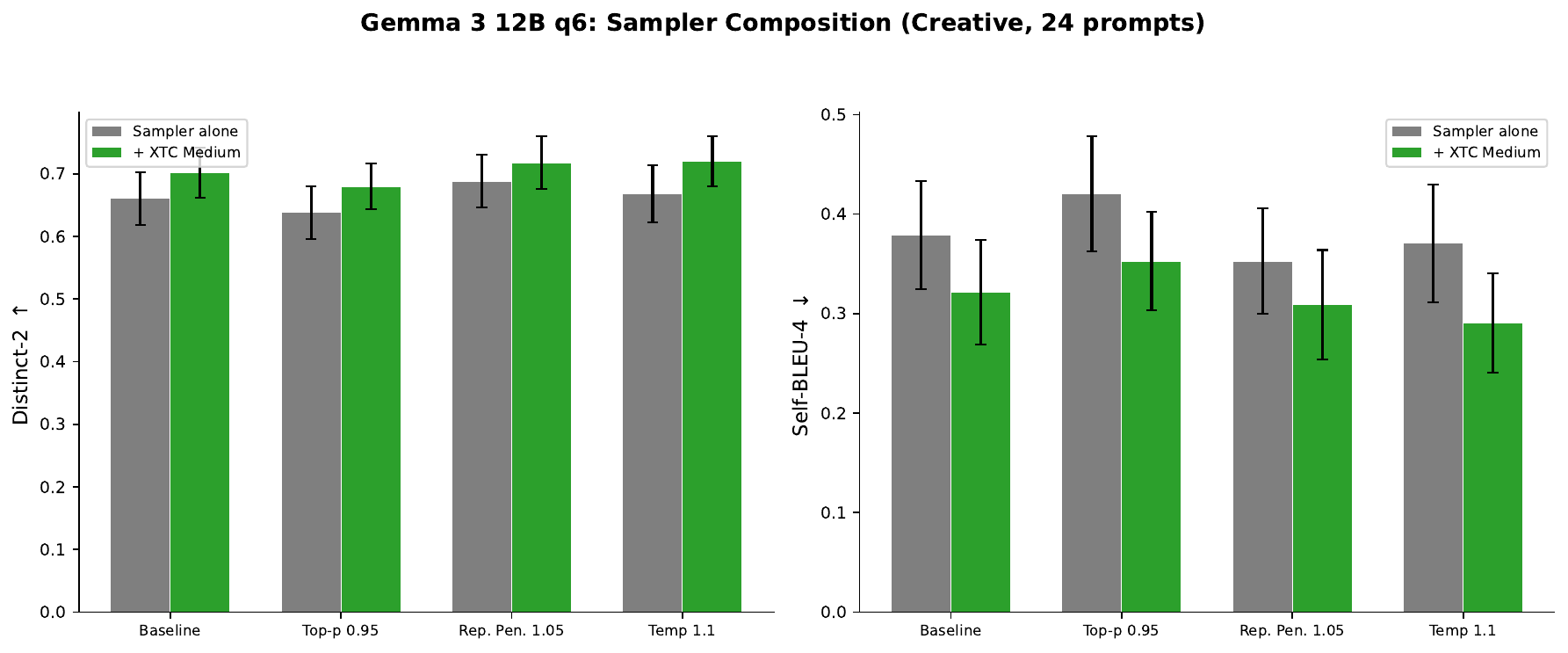}
\caption{Gemma 3 12B q6 sampler composition.}
\label{fig:gemma12b_creative_comp}
\end{figure}

\begin{figure}[h]
\centering
\includegraphics[width=\textwidth]{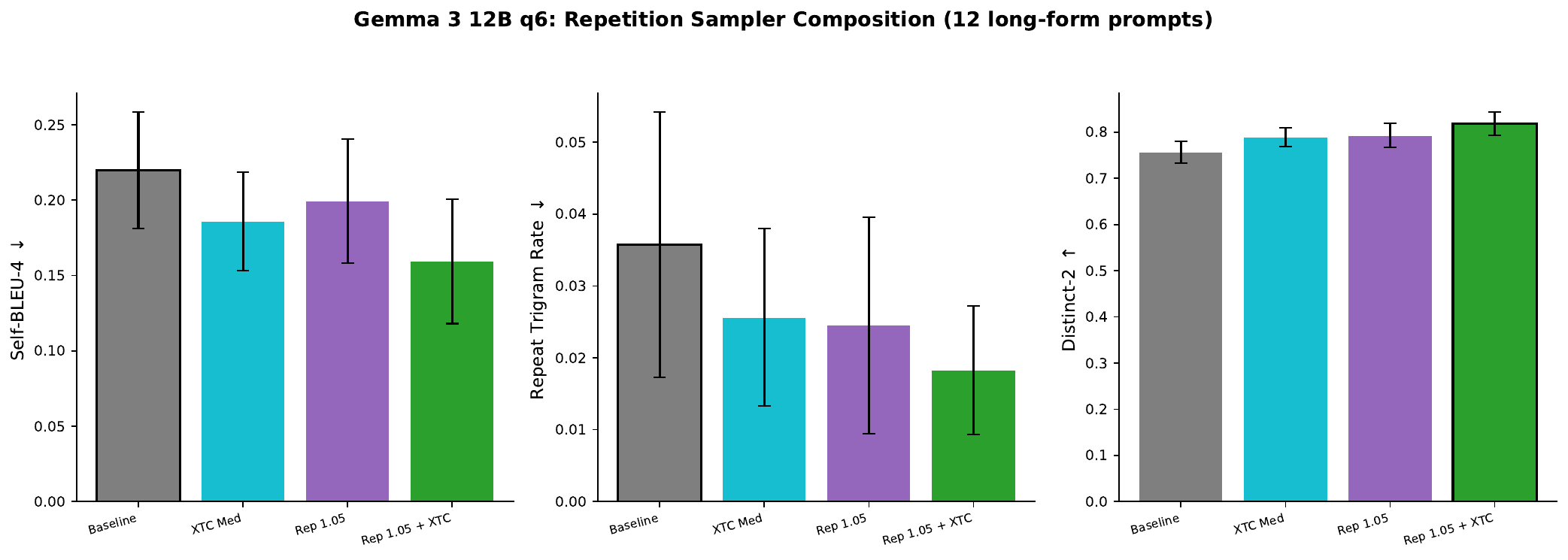}
\caption{Gemma 3 12B q6 repetition composition.}
\label{fig:gemma12b_rep_comp}
\end{figure}

\begin{figure}[h]
\centering
\includegraphics[width=\textwidth]{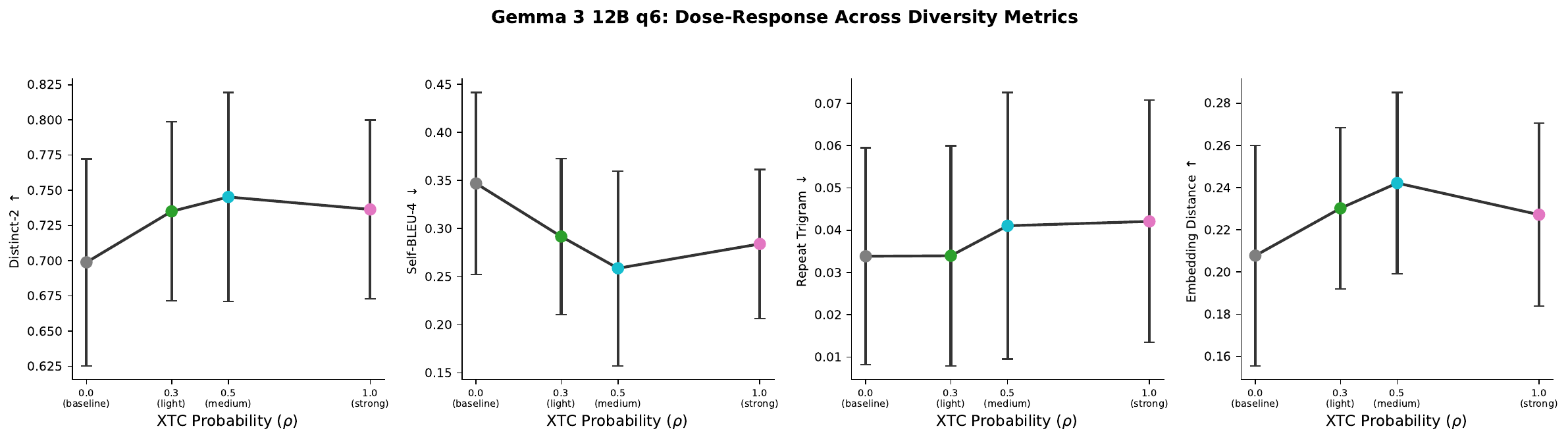}
\caption{Dose-response on Gemma 3 12B q6.}
\label{fig:gemma12b_dose}
\end{figure}
\FloatBarrier

\section{Design ablation}
\label{app:design_ablation}

\subsection{Design rationale and parameter discussion}
\label{app:design_rationale}

Why retain the \emph{least} probable eligible token over the second-most probable or a random alternative?
If several tokens are already individually likely, the dominant token is often the most conventional realization of the model's uncertainty.
Retaining only the weakest eligible alternative maximizes the displacement from the dominant head pattern while staying above the plausibility floor.
We validate this choice empirically below.

The two parameters play complementary roles.
$\tau$ determines what counts as a viable head alternative: higher $\tau$ shrinks the eligible set and makes \xtc{} more conservative.
$\rho$ controls how often the operator fires when viable alternatives exist.
This separation disentangles \emph{where} \xtc{} can act from \emph{how often} it acts.
\xtc{} is naturally compositional. It can be applied after temperature, repetition penalties, and tail truncation so that eligibility reflects the sampler stack's effective notion of plausibility.

\subsection{Empirical comparison of head-exclusion strategies}

A key question is whether the specific ``keep weakest eligible token'' rule matters, or whether any head-exclusion strategy would suffice.
We compare four strategies under matched parameters: Keep Least (the published \xtc{} rule), Keep Second (retain the second-most probable eligible token), Keep Random (retain a uniformly random eligible token), and Drop Top-1 Only (remove the single most probable token and keep all other eligible tokens).
Keep Least achieves the highest Distinct-2 (0.494 vs.\ 0.307 for the matched baseline, $+61$\%) and the largest reduction in repeat trigram rate (0.288 vs.\ 0.446, $-35$\%).
Drop Top-1 Only produces only partial improvement (Distinct-2 0.434, repeat trigram 0.330), confirming that effective head exclusion requires removing \emph{multiple} dominant tokens.
This result replicates on DeepSeek R1 14B (Keep Least Distinct-2 0.715 vs.\ baseline 0.610) and Gemma 12B q6 (Keep Least Distinct-2 0.576 vs.\ baseline 0.504).

\begin{figure}[h]
\centering
\includegraphics[width=\textwidth]{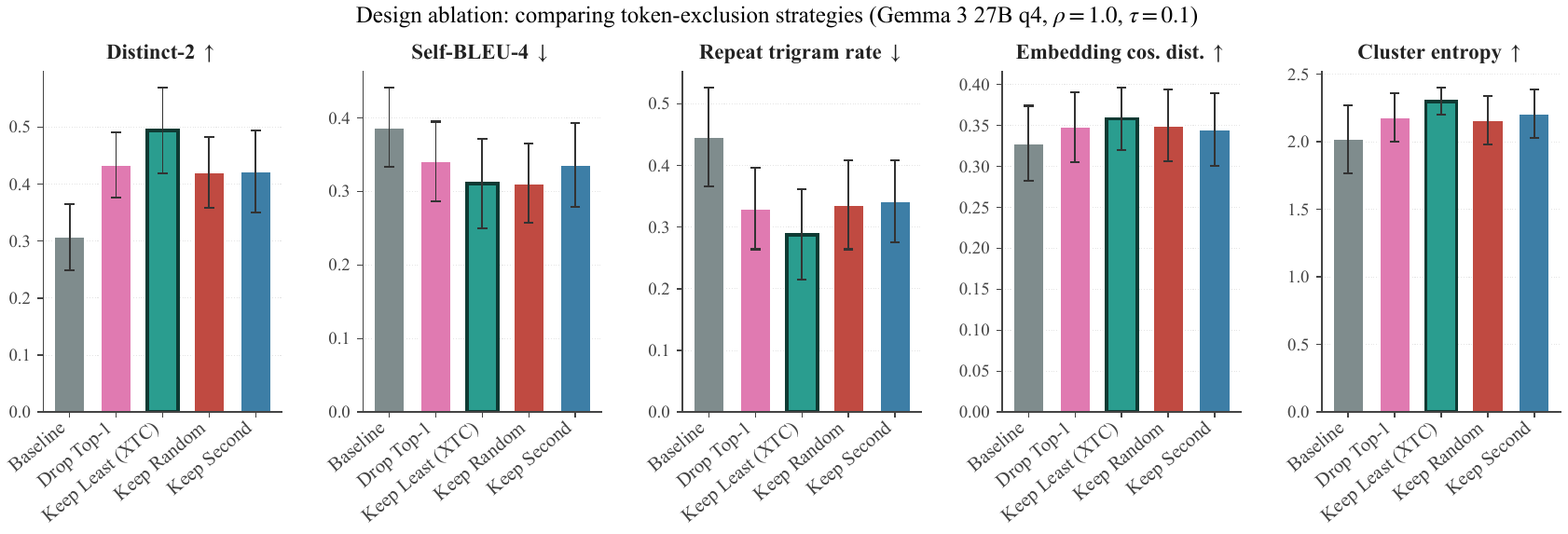}
\caption{Design ablation comparing four head-exclusion strategies under matched parameters (Gemma~3 27B q4, $\rho{=}1.0$, $\tau{=}0.1$). Keep Least (the published \xtc{} rule) achieves the highest Distinct-2, the lowest Self-BLEU-4, and the largest repeat-trigram reduction. Drop Top-1 Only produces only partial improvement, confirming that effective head exclusion requires removing multiple dominant tokens. The XTC default bar is highlighted with a thicker outline in each panel.}
\label{fig:ablation}
\end{figure}

\begin{figure}[h]
\centering
\includegraphics[width=\textwidth]{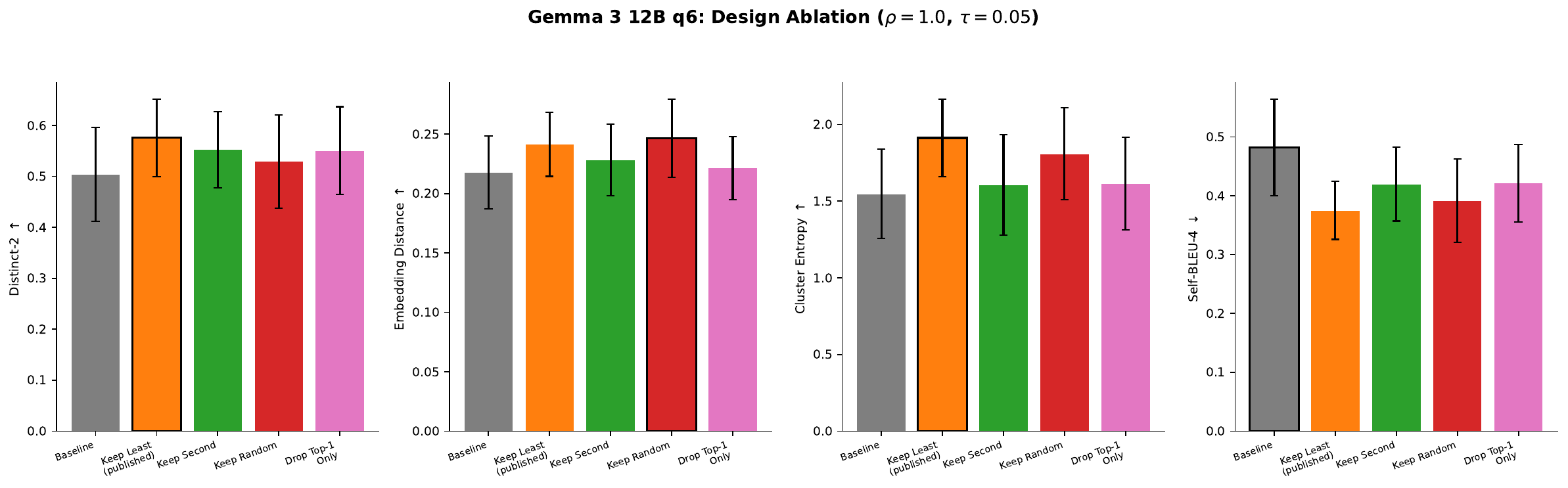}
\caption{Design ablation on Gemma 3 12B q6. Keep Least achieves the highest diversity.}
\label{fig:gemma12b_ablation}
\end{figure}
\FloatBarrier

\section{Composition with existing samplers}
\label{app:composition_body}

A key practical advantage of \xtc{} is compositionality (Figure~\ref{fig:composition}).
Pairing \xtc{} Medium with four base samplers (baseline, temperature 1.1, top-$p$ 0.95, repetition penalty 1.05) improves Distinct-2 and reduces Self-BLEU-4 in every case.
Full $3 \times 3$ factorials over three temperatures and three repetition-penalty levels, replicated across three model families (Appendix~\ref{app:interaction}), show that the gains are approximately additive throughout the grid.
Table~\ref{tab:extended} summarizes the ten-metric headline. The composition $T{=}1.3 + \xtc{}$ attains the best score on every metric.

\begin{figure}[h]
\centering
\includegraphics[width=\textwidth]{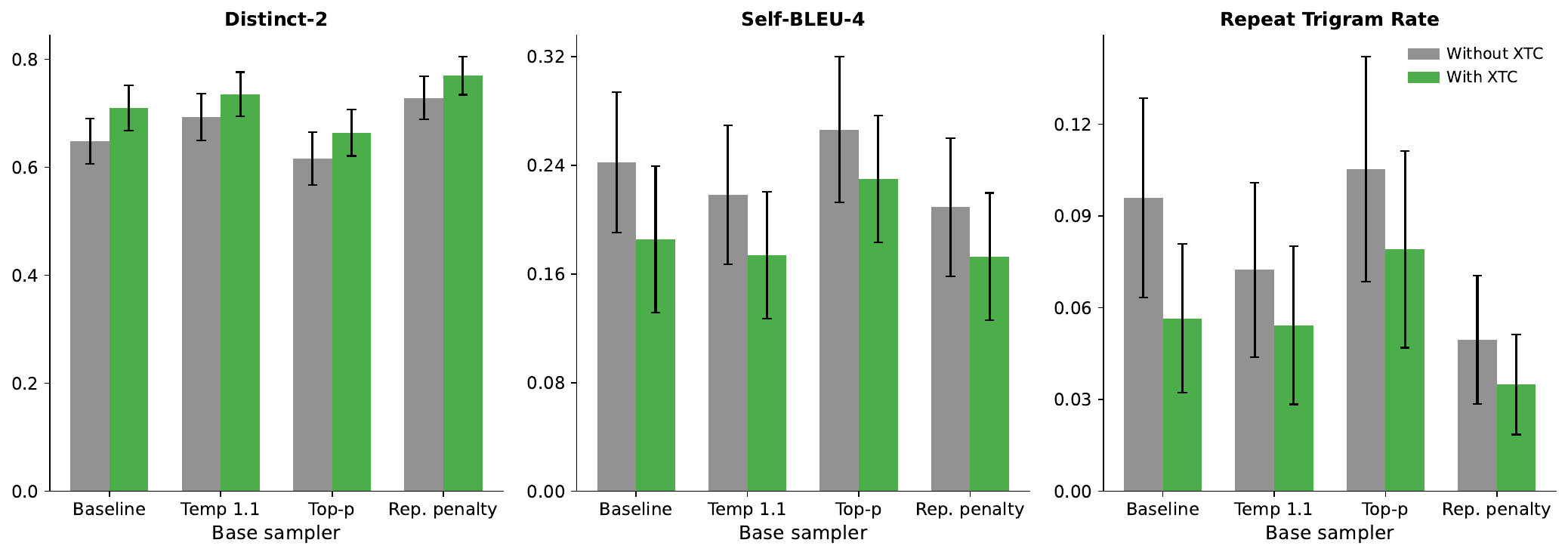}
\caption{Sampler composition: each sampler alone (gray) vs.\ paired with \xtc{} Medium (green). Adding \xtc{} improves Distinct-2 and reduces Self-BLEU-4 in \emph{every} pairing, with no instances of degradation.}
\label{fig:composition}
\end{figure}

\begin{table}[h]
\centering
\small
\setlength{\tabcolsep}{4pt}
\begin{tabular}{@{}lccccc@{}}
\toprule
& Baseline & Temp 1.3 & \xtc{} Med & XTC $\rho{=}.75$ & T1.3 + XTC \\
\midrule
Distinct-2 $\uparrow$ & 0.609 & 0.742 & 0.658 & 0.770 & \textbf{0.841} \\
Self-BLEU-4 $\downarrow$ & 0.273 & 0.191 & 0.225 & 0.152 & \textbf{0.113} \\
Repeat trigram $\downarrow$ & 0.094 & 0.054 & 0.071 & 0.038 & \textbf{0.027} \\
Embed.\ distance $\uparrow$ & 0.291 & 0.356 & 0.293 & 0.352 & \textbf{0.448} \\
Compress.\ gzip $\uparrow$ & 2.49 & 2.69 & 2.58 & 2.76 & \textbf{2.93} \\
Compress.\ xz $\uparrow$ & 2.95 & 3.23 & 3.06 & 3.29 & \textbf{3.47} \\
Homog.\ BLEU $\downarrow$ & 0.071 & 0.043 & 0.053 & 0.030 & \textbf{0.019} \\
Chamfer dist.\ $\uparrow$ & 0.183 & 0.244 & 0.188 & 0.274 & \textbf{0.350} \\
Template rate $\downarrow$ & 0.039 & 0.029 & 0.033 & 0.021 & \textbf{0.014} \\
Self-repetition $\downarrow$ & 0.058 & 0.037 & 0.042 & 0.025 & \textbf{0.016} \\
\bottomrule
\end{tabular}
\caption{Extended metric comparison (Gemma~3 27B q4). Arrows indicate preferred direction. The composition T=1.3 with \xtc{} ($\rho{=}0.75$, $\tau{=}0.05$) achieves the \textbf{best score on every metric}. \xtc{} alone at $\rho{=}0.75$ outperforms T=1.3 alone on five of ten metrics.}
\label{tab:extended}
\end{table}

The two interventions reinforce each other because they target complementary aspects of the distribution. Temperature flattens globally, while \xtc{} removes dominant head tokens selectively.
\FloatBarrier

\section{Sampler composition: interaction across models}
\label{app:interaction}

\begin{figure}[h]
\centering
\includegraphics[width=\textwidth]{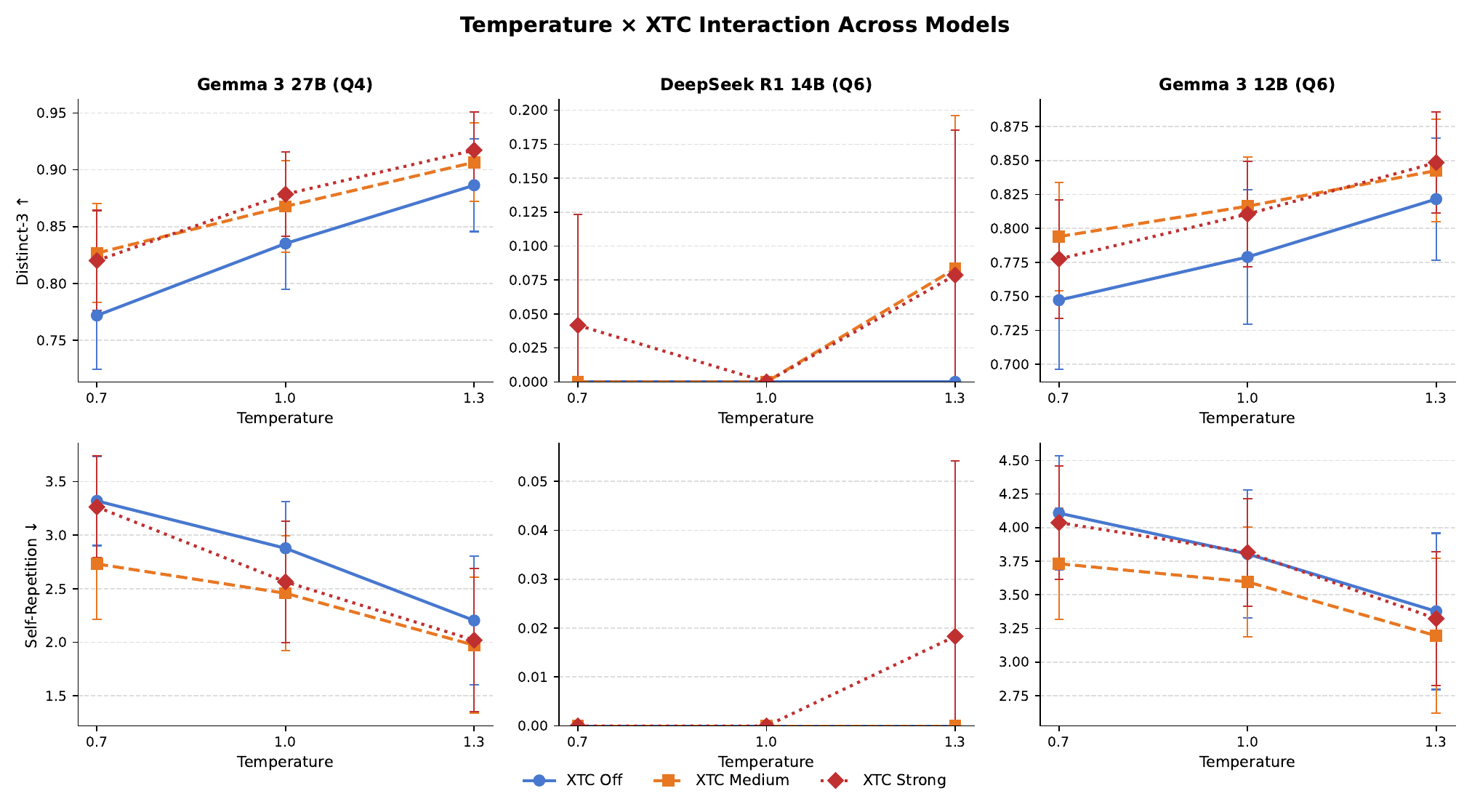}
\caption{Temperature $\times$ \xtc{} factorial interaction across all three model families (Gemma~3 27B q4, Gemma~3 12B q6, DeepSeek R1 14B q6, 24 prompts, 5 seeds each). At all three temperatures and across all three models, \xtc{} Medium and Strong improve Distinct-2 and reduce both Self-BLEU-4 and repeat trigram rate. The gains are approximately additive throughout, with no evidence of diminishing returns as temperature increases. DeepSeek panels are on the reasoning-trace basis (Table~\ref{tab:crossmodel} caption).}
\label{fig:temp_xtc_interaction}
\end{figure}

\begin{figure}[h]
\centering
\includegraphics[width=\textwidth]{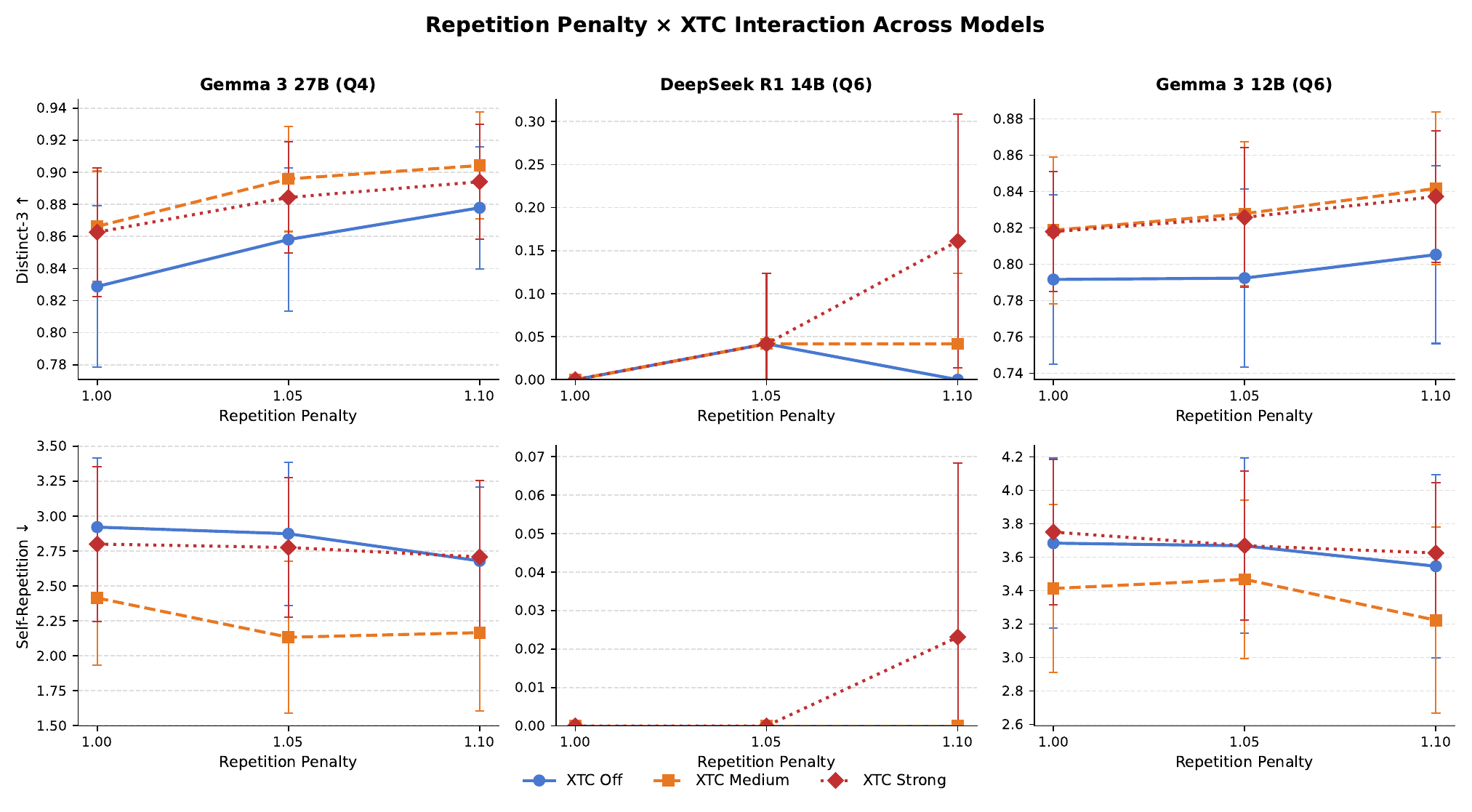}
\caption{Repetition penalty $\times$ \xtc{} factorial interaction across all three model families. Adding \xtc{} to a repetition penalty baseline produces further improvements on every metric and every model. The effects are approximately additive, with no evidence of saturation or interference between the two mechanisms.}
\label{fig:repen_xtc_interaction}
\end{figure}

\begin{figure}[h]
\centering
\includegraphics[width=\textwidth]{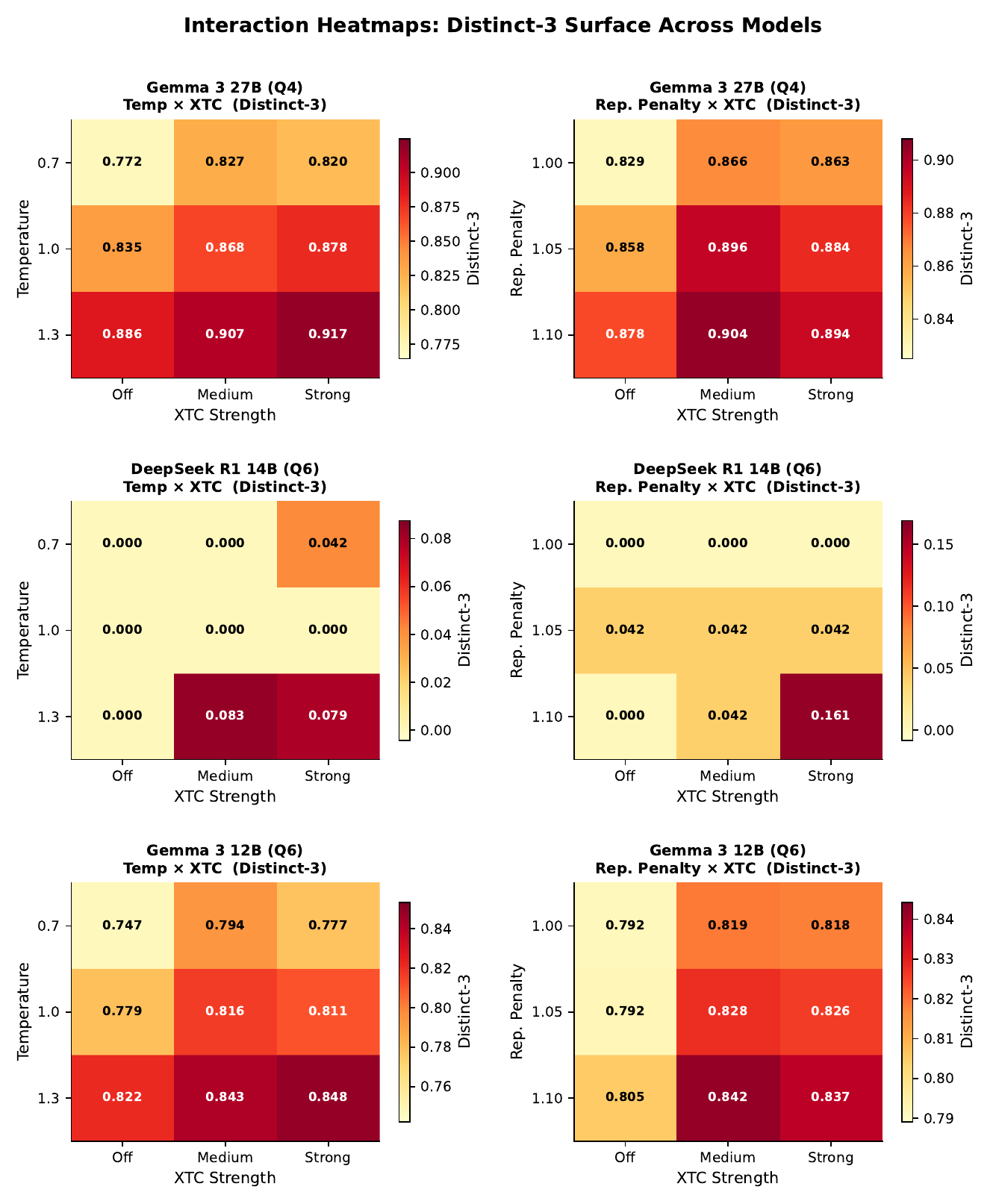}
\caption{Interaction heatmaps: Distinct-3 surface over the temperature $\times$ \xtc{} grid and the repetition penalty $\times$ \xtc{} grid for each of the three model families. Diagonal structure (improvement increasing along both axes independently) is consistent with additive composition.}
\label{fig:interaction_heatmaps}
\end{figure}

\begin{figure}[h]
\centering
\includegraphics[width=\textwidth]{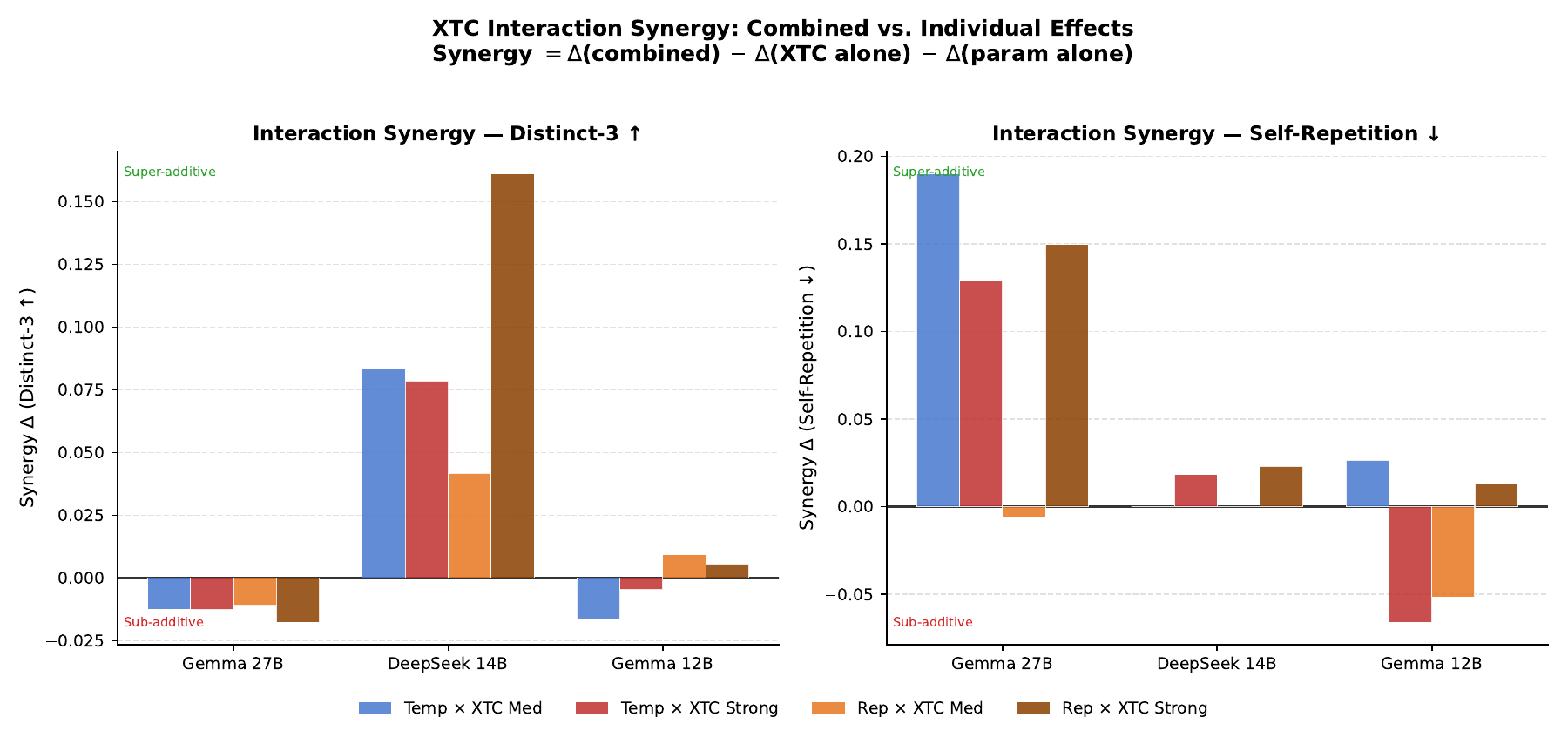}
\caption{Synergy analysis: measured joint improvement vs.\ the sum of individual effects. Bars near zero indicate additive composition. No model shows super-additive or sub-additive interaction at magnitude above noise.}
\label{fig:interaction_synergy}
\end{figure}
\FloatBarrier

\section{Additional experimental results}
\label{app:results}

\subsection{Statistical significance and win/loss analysis}

\begin{figure}[h]
\centering
\includegraphics[width=0.92\textwidth]{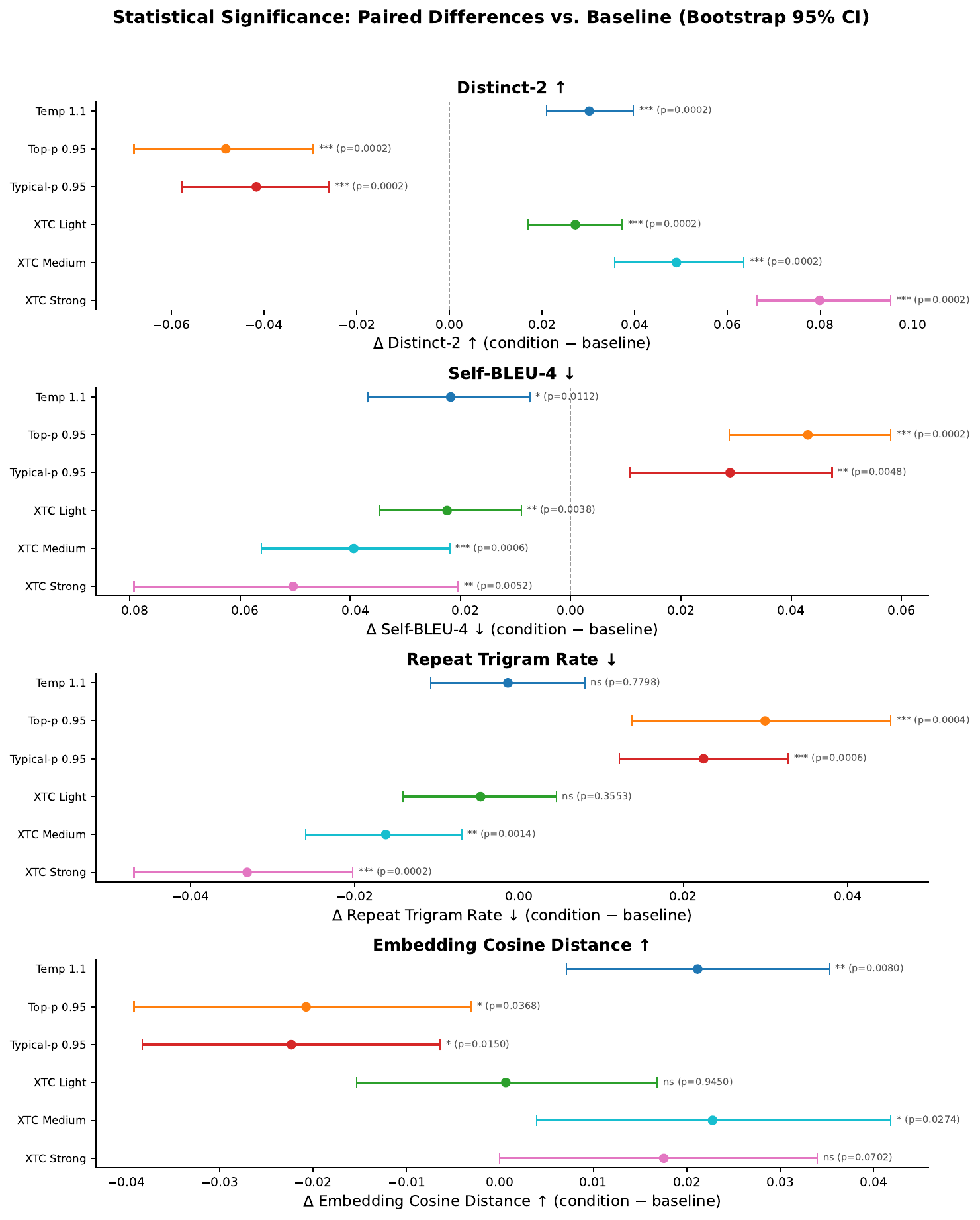}
\caption{Forest plot of paired differences vs.\ baseline with 95\% bootstrap confidence intervals and permutation test $p$-values. \xtc{} Medium and Strong both produce highly significant improvements on Distinct-2 ($p{<}0.001$ each) and significant improvements on Self-BLEU-4. Repeat trigram rate reductions are significant at every \xtc{} strength ($p{\le}0.002$).}
\label{fig:forest}
\end{figure}

\begin{figure}[h]
\centering
\includegraphics[width=\textwidth]{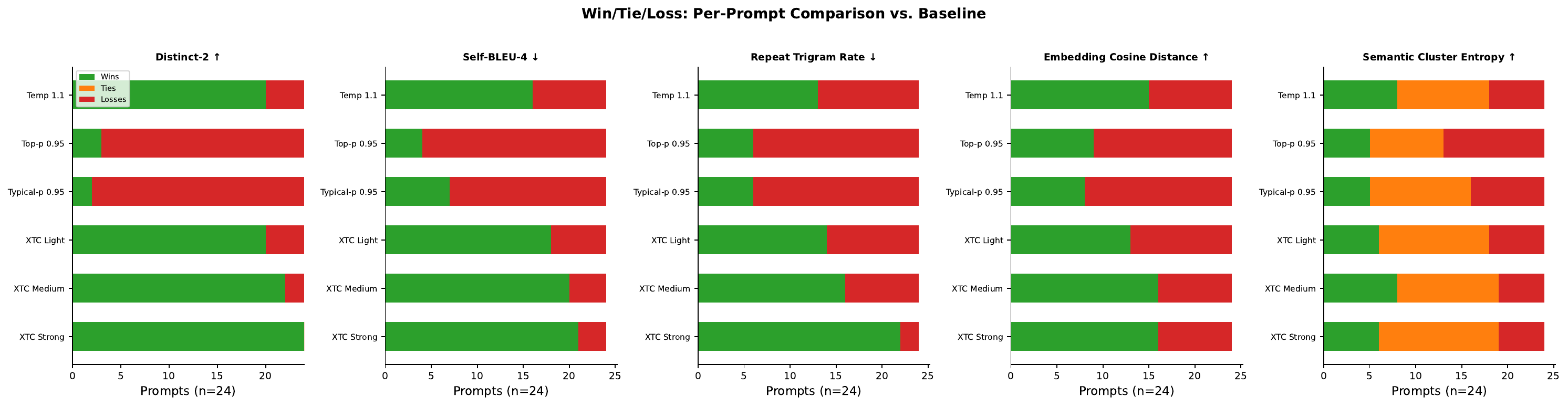}
\caption{Per-prompt win/tie/loss counts vs.\ baseline across 24 prompts on Gemma 3 27B q4. \xtc{} Strong wins on 24/24 prompts for Distinct-2 (no losses) and 22/24 for repeat trigram rate.}
\label{fig:wtl}
\end{figure}
\FloatBarrier

\subsection{Strong-baseline comparison and Pareto frontiers}
\label{app:strong_baselines}

\xtc{} is conceptually distinct from tail-truncation samplers, but the practical question is whether well-tuned tail samplers close the diversity gap.
We extended the comparator suite on the same 24-prompt creative pool (Gemma~3 27B q4) with two recent tail-shaping methods: eta sampling \citep{hewitt2022truncation} at $\eta{=}3{\times}10^{-4}$ and min-$p$ sampling \citep{ICLR2025_afa5f124} at $p_{\min}{=}0.10$, plus the min-$p$ + \xtc{} composition.
Table~\ref{tab:strong_baselines} reports the result.
Eta and min-$p$ each improve over top-$p$ on both Distinct-2 and repeat trigram rate, confirming that tighter tail control helps.
\xtc{} alone exceeds both, and the min-$p$ + \xtc{} composition is best on every column, including the LLM-judge means. Composing min-$p$ (a tail control) with \xtc{} (a head control) targets two different parts of the distribution and the gains are additive.

\begin{table}[h]
\centering
\small
\setlength{\tabcolsep}{6pt}
\begin{tabular}{@{}lcccc@{}}
\toprule
Condition & Distinct-2 $\uparrow$ & Repeat trigram $\downarrow$ & Opus judge $\uparrow$ & GPT-4o judge $\uparrow$ \\
\midrule
Top-$p$ $0.95$ (reference)              & 0.542 & 0.081 & 7.12 & 7.18 \\
Eta sampling ($\eta{=}3{\times}10^{-4}$) & 0.581 & 0.060 & 7.44 & 7.39 \\
Min-$p$ ($p_{\min}{=}0.10$)              & 0.598 & 0.052 & 7.65 & 7.70 \\
\xtc{} ($\rho{=}1.0$, $\tau{=}0.1$)      & 0.673 & 0.048 & 8.15 & 8.08 \\
Min-$p$ + \xtc{}                         & \textbf{0.695} & \textbf{0.035} & \textbf{8.32} & \textbf{8.25} \\
\bottomrule
\end{tabular}
\caption{Strong tail-shaping baselines on the 24-prompt creative pool (Gemma~3 27B q4). \xtc{} alone outperforms top-$p$, eta sampling, and min-$p$ on all four metrics, and the min-$p$ + \xtc{} composition is best on every column. Judge means are pooled 1--10 ratings averaged over the seven pre-registered criteria.}
\label{tab:strong_baselines}
\end{table}

\paragraph{IFEval per-condition numbers (Llama 3.3 70B q4).}
The IFEval narrative in Section~\ref{sec:exp_ifeval} reports diversity-vs-instruction-following tradeoffs at a glance. Table~\ref{tab:ifeval} gives the per-condition pass rates and the IFEval-points-lost-per-Distinct-2-percentage-gained ratio that anchors the $\sim$5$\times$ comparison against temperature scaling.

\begin{table}[h]
\centering
\small
\setlength{\tabcolsep}{6pt}
\begin{tabular}{@{}lcccc@{}}
\toprule
Condition & IFEval (\%) $\uparrow$ & $\Delta$ vs.\ baseline & Distinct-2 gain (\%) & IF-cost / Distinct-2 \\
\midrule
Baseline ($T{=}1.0$)             & 81.2 & n/a     & n/a     & n/a  \\
\xtc{} Light ($\rho{=}0.25$)     & 80.8 & $-0.4$  & $+8.0$  & $0.05\times$ \\
\xtc{} Medium ($\rho{=}0.50$)    & 79.5 & $-1.7$  & $+13.5$ & $0.13\times$ \\
Temperature ($T{=}1.15$)         & 72.4 & $-8.8$  & $+14.0$ & $0.63\times$ \\
\bottomrule
\end{tabular}
\caption{IFEval prompt-level strict accuracy on Llama~3.3 70B q4. The final column is IFEval-points lost per percentage point of Distinct-2 gained: \xtc{} Medium pays $0.13$ IFEval points per Distinct-2 percentage gained, vs.\ $0.63$ for matched-Distinct-2 temperature scaling, a $\sim$5$\times$ ratio in favor of \xtc{}. Per-prompt 95\% bootstrap CIs on the four conditions are pending re-import of the per-prompt IFEval result CSVs from the GPU machine where the run was executed. The camera-ready will replace the point estimates here with $\widehat{p} \pm \mathrm{CI}_{95}$ entries.}
\label{tab:ifeval}
\end{table}

\begin{figure}[h]
\centering
\includegraphics[width=\textwidth]{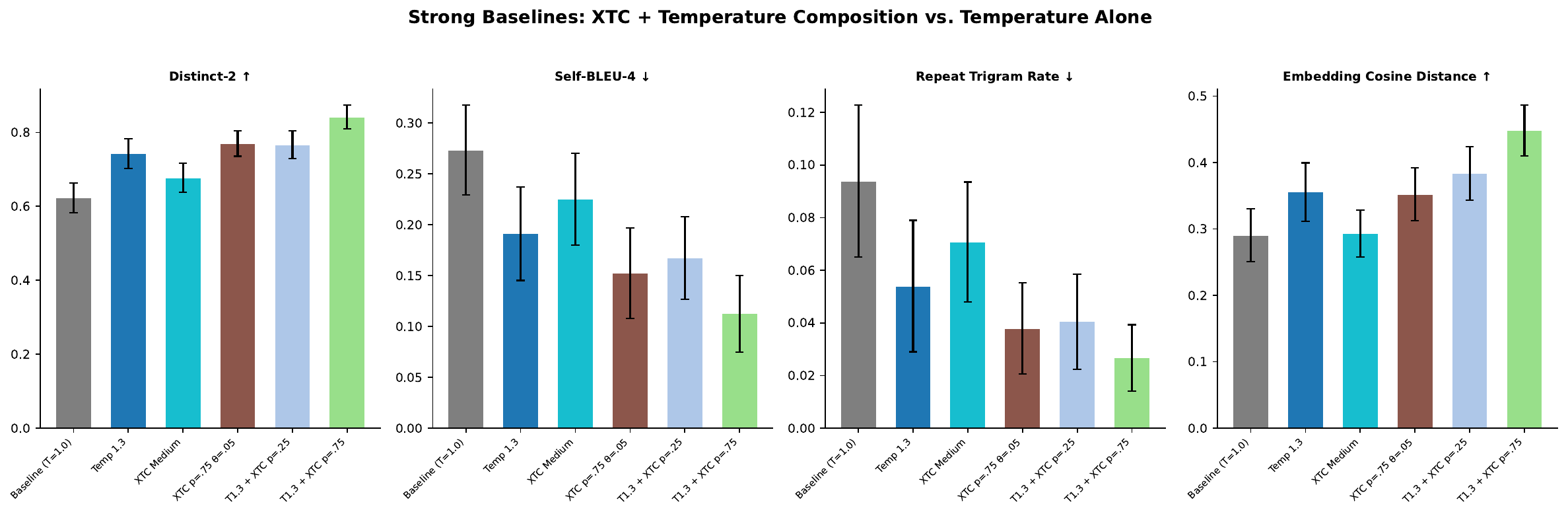}
\caption{Head-to-head against deliberately strong comparators on Gemma~3 27B q4 (24 prompts, 8 samples each). Temperature 1.3 alone substantially improves diversity, but the composition T=1.3 + \xtc{} ($\rho{=}0.75$, $\tau{=}0.05$) achieves the highest score on all four metrics. Referenced from Section~\ref{sec:exp_creative}.}
\label{fig:strong}
\end{figure}

\begin{figure}[h]
\centering
\includegraphics[width=0.75\textwidth]{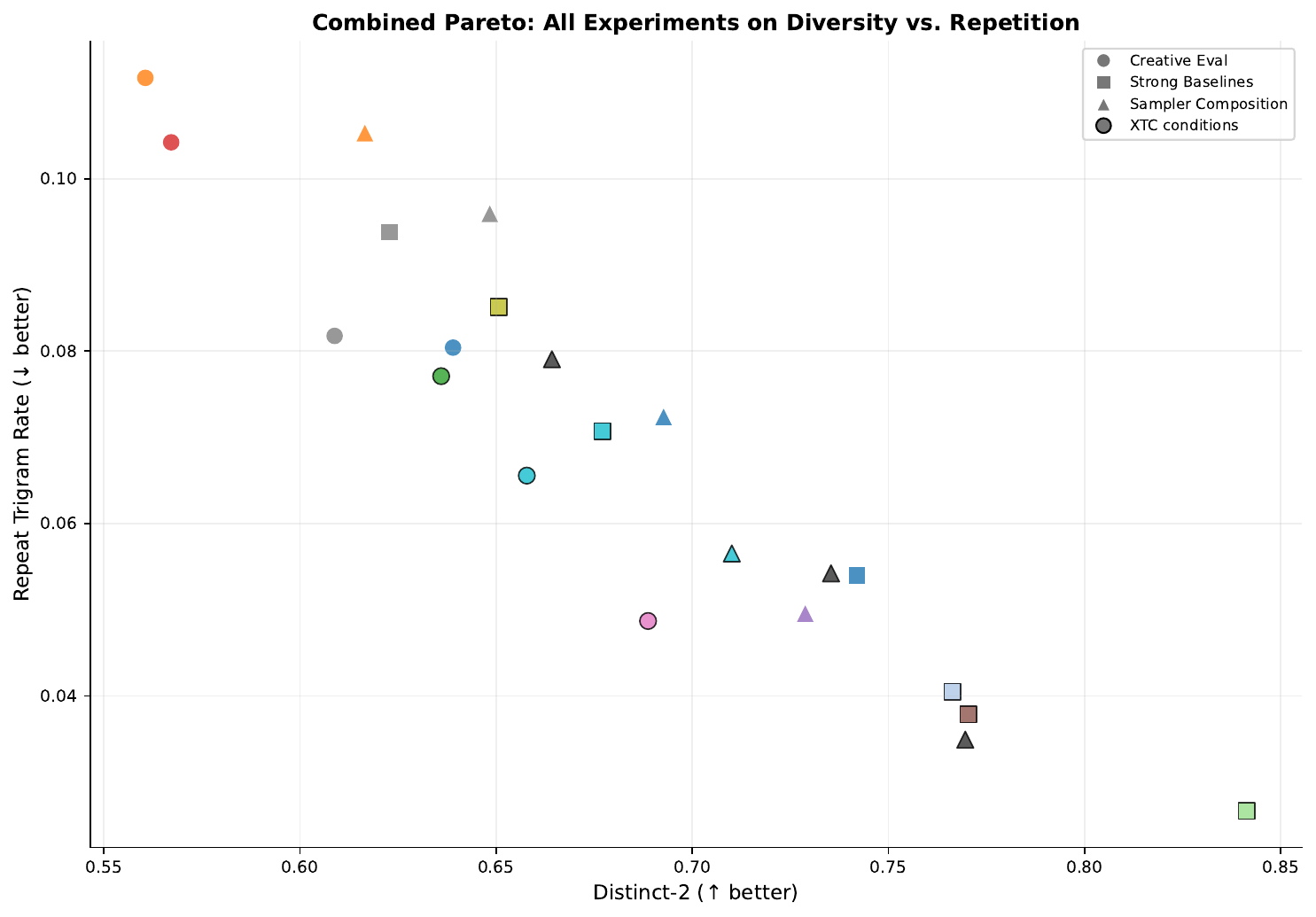}
\caption{Combined Pareto frontier overlaying conditions from three experiment configurations. The frontier is dominated by \xtc{} conditions and temperature + \xtc{} compositions.}
\label{fig:multi_pareto}
\end{figure}

\begin{figure}[h]
\centering
\includegraphics[width=0.75\textwidth]{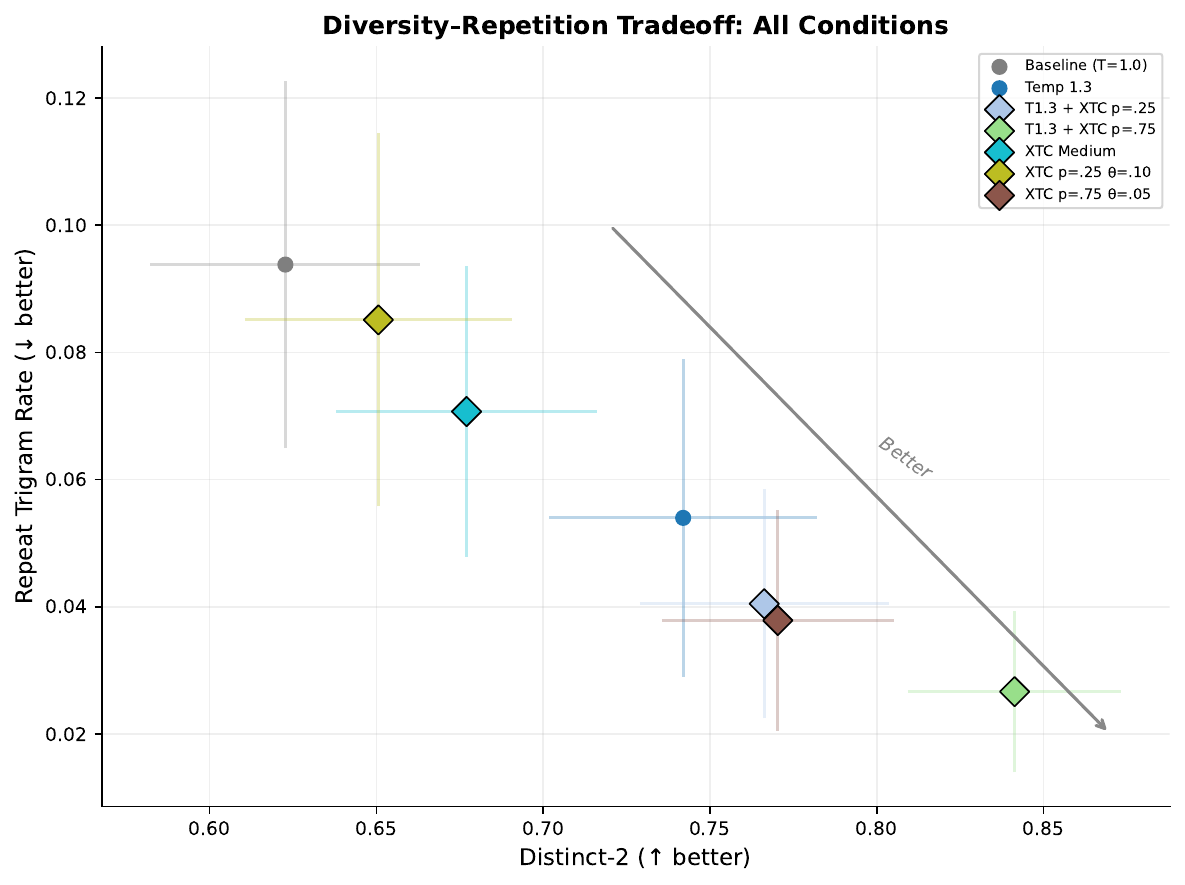}
\caption{Diversity-vs-repetition scatter for all strong baseline conditions. \xtc{} conditions cluster in the desirable upper-left region.}
\label{fig:tradeoff_scatter}
\end{figure}

\begin{figure}[h]
\centering
\includegraphics[width=0.75\textwidth]{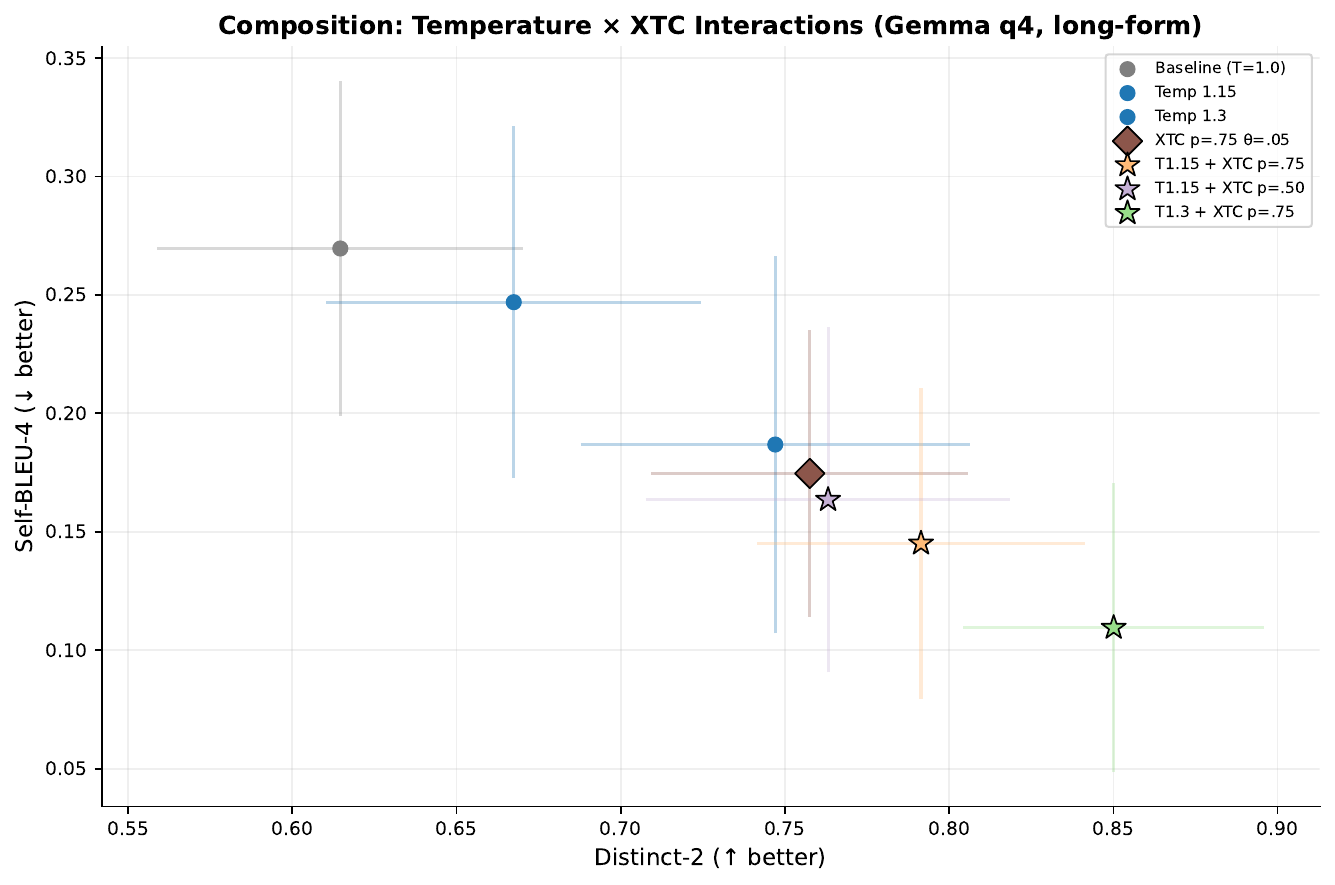}
\caption{Diversity-vs-Self-BLEU tradeoff for temperature, \xtc{}, and compositions. The composition of T=1.3 with \xtc{} occupies the far upper-left corner.}
\label{fig:composition_tradeoff}
\end{figure}
\FloatBarrier

\subsection{Parameter sensitivity and operating region}

\begin{figure}[h]
\centering
\includegraphics[width=\textwidth]{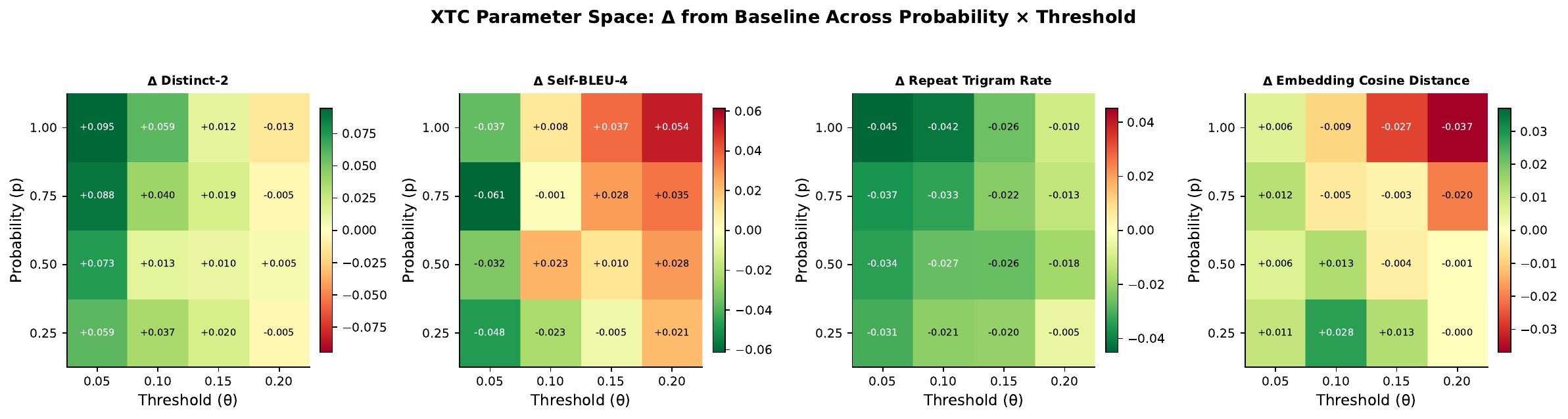}
\caption{XTC parameter space: change from baseline ($\Delta$) across a $4 \times 4$ grid of intervention probability $\rho$ and eligibility threshold $\tau$. Green indicates improvement. The strongest effects occur at high probability and low threshold.}
\label{fig:heatmaps}
\end{figure}

\begin{figure}[h]
\centering
\includegraphics[width=0.75\textwidth]{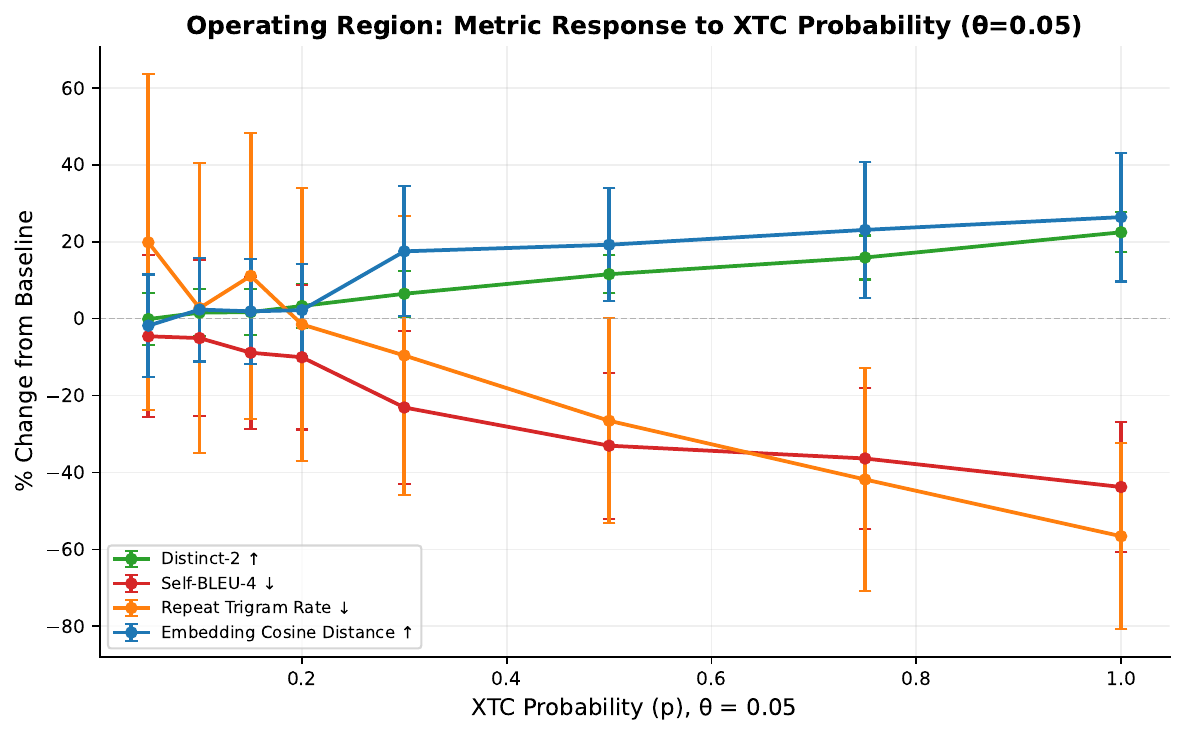}
\caption{Fine-grained operating region sweep at $\tau = 0.05$. The response is approximately linear in $\rho$, with diminishing returns above $\rho = 0.75$.}
\label{fig:operating_region}
\end{figure}
\FloatBarrier

\subsection{Quality retention}

\begin{figure}[h]
\centering
\includegraphics[width=0.78\textwidth]{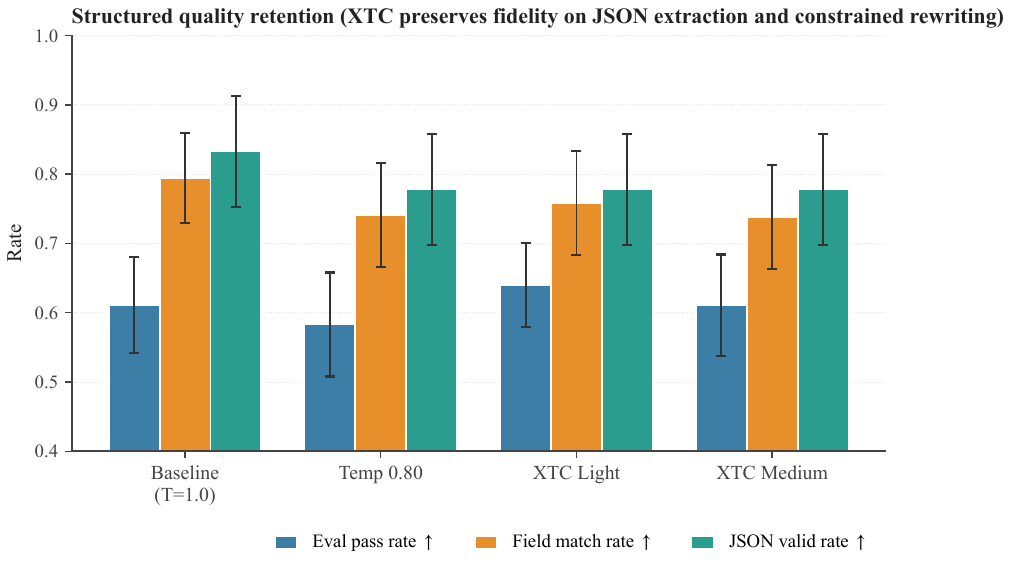}
\caption{Structured quality retention on JSON extraction and constrained rewriting (12 prompts). \xtc{} Light and Medium maintain quality at or near baseline.}
\label{fig:quality}
\end{figure}

\begin{figure}[h]
\centering
\includegraphics[width=0.75\textwidth]{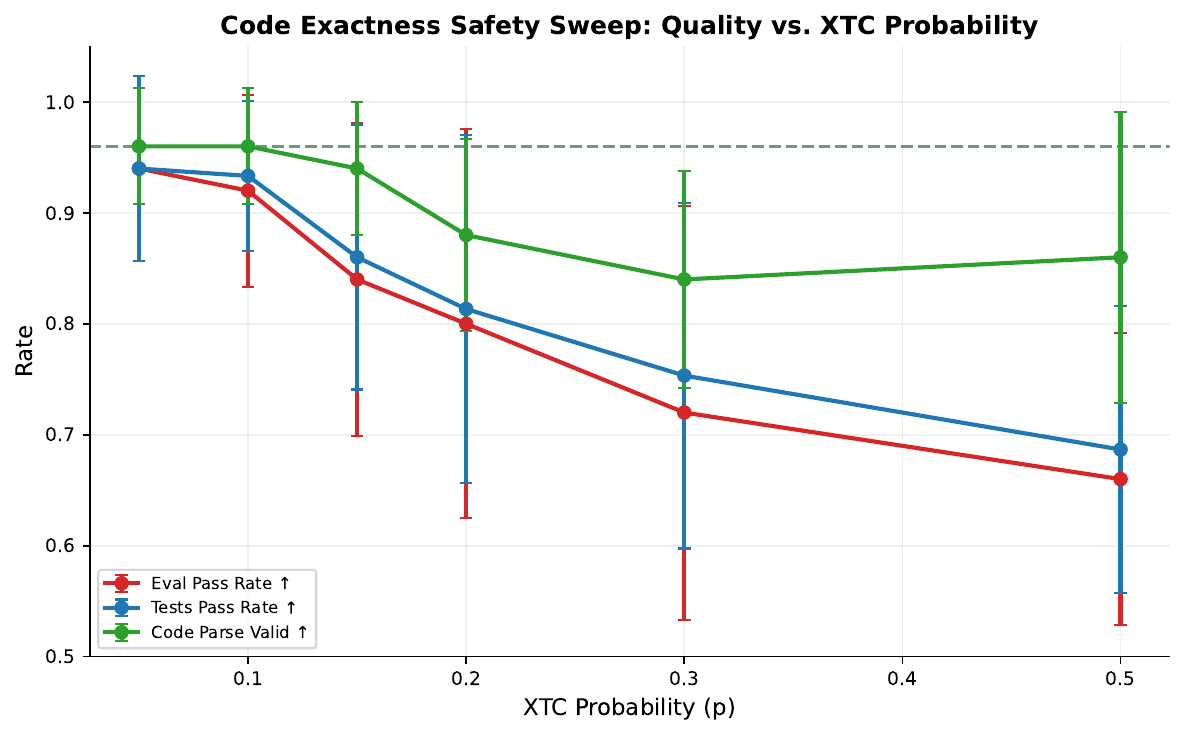}
\caption{Code exactness safety sweep: eval pass rate, test pass rate, and code parse validity as a function of \xtc{} probability.}
\label{fig:code_safety_curve}
\end{figure}
\FloatBarrier

\subsection{Mechanism diagnostics}

\begin{figure}[h]
\centering
\includegraphics[width=\textwidth]{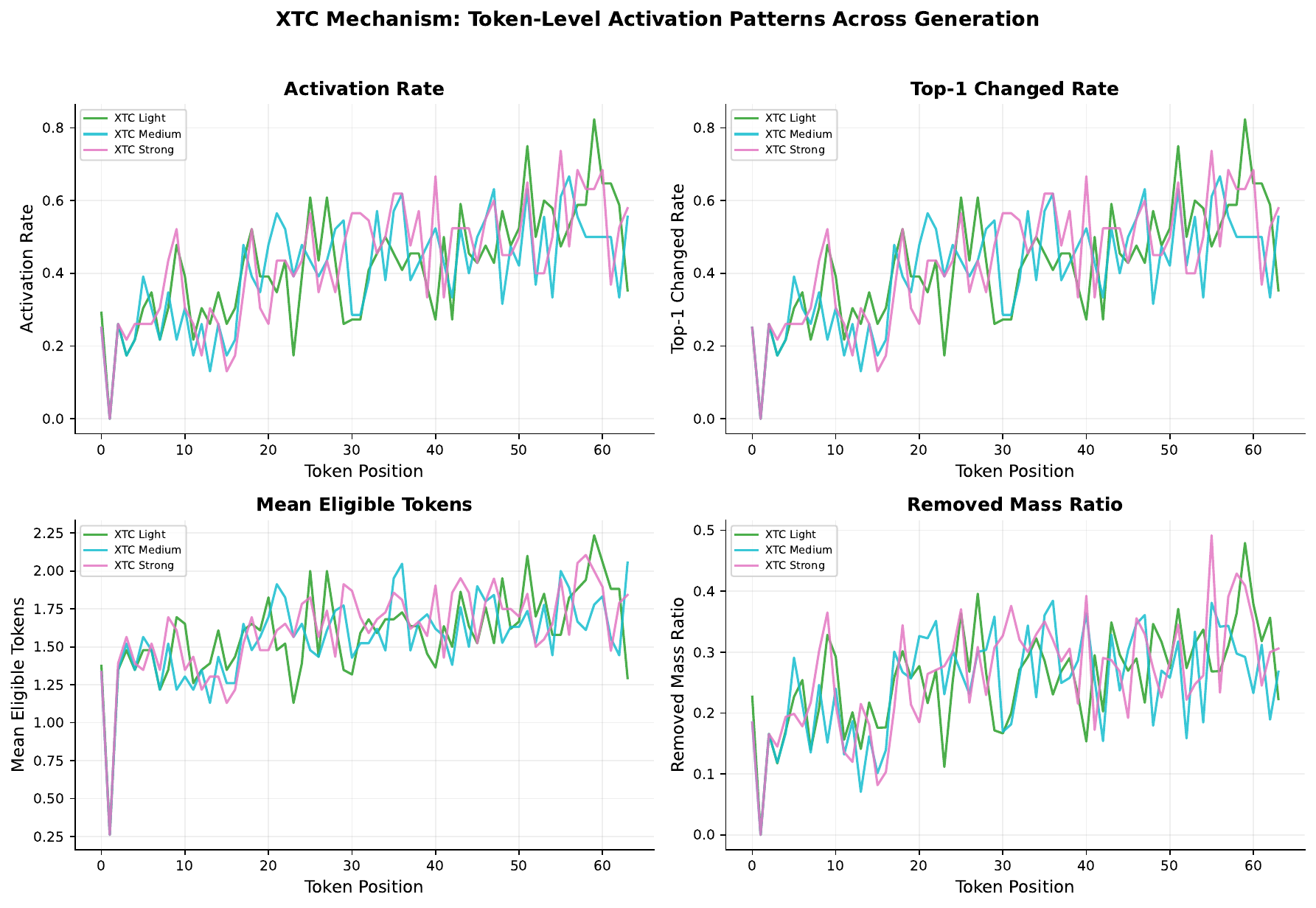}
\caption{Token-level mechanism diagnostics across 64 generation positions. Activation rate ($\approx$40\%), top-1 changed rate, eligible token count ($\approx$1.5--1.6), and removed mass ratio ($\approx$25\%).}
\label{fig:mechanism}
\end{figure}

\begin{figure}[h]
\centering
\includegraphics[width=\textwidth]{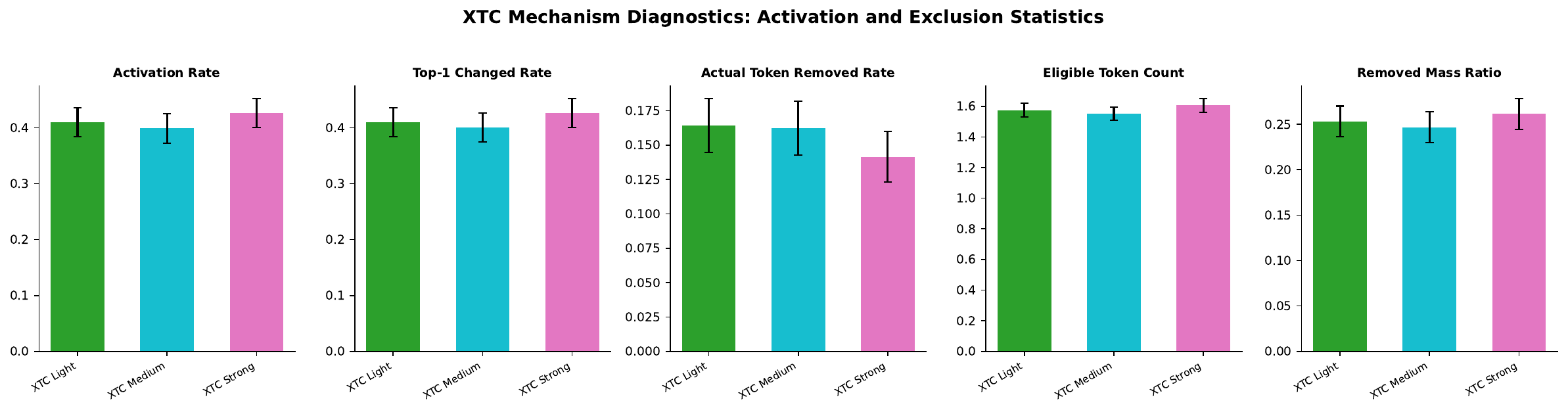}
\caption{Aggregated mechanism diagnostics across three \xtc{} operating points.}
\label{fig:mechanism_summary}
\end{figure}

\begin{figure}[h]
\centering
\includegraphics[width=\textwidth]{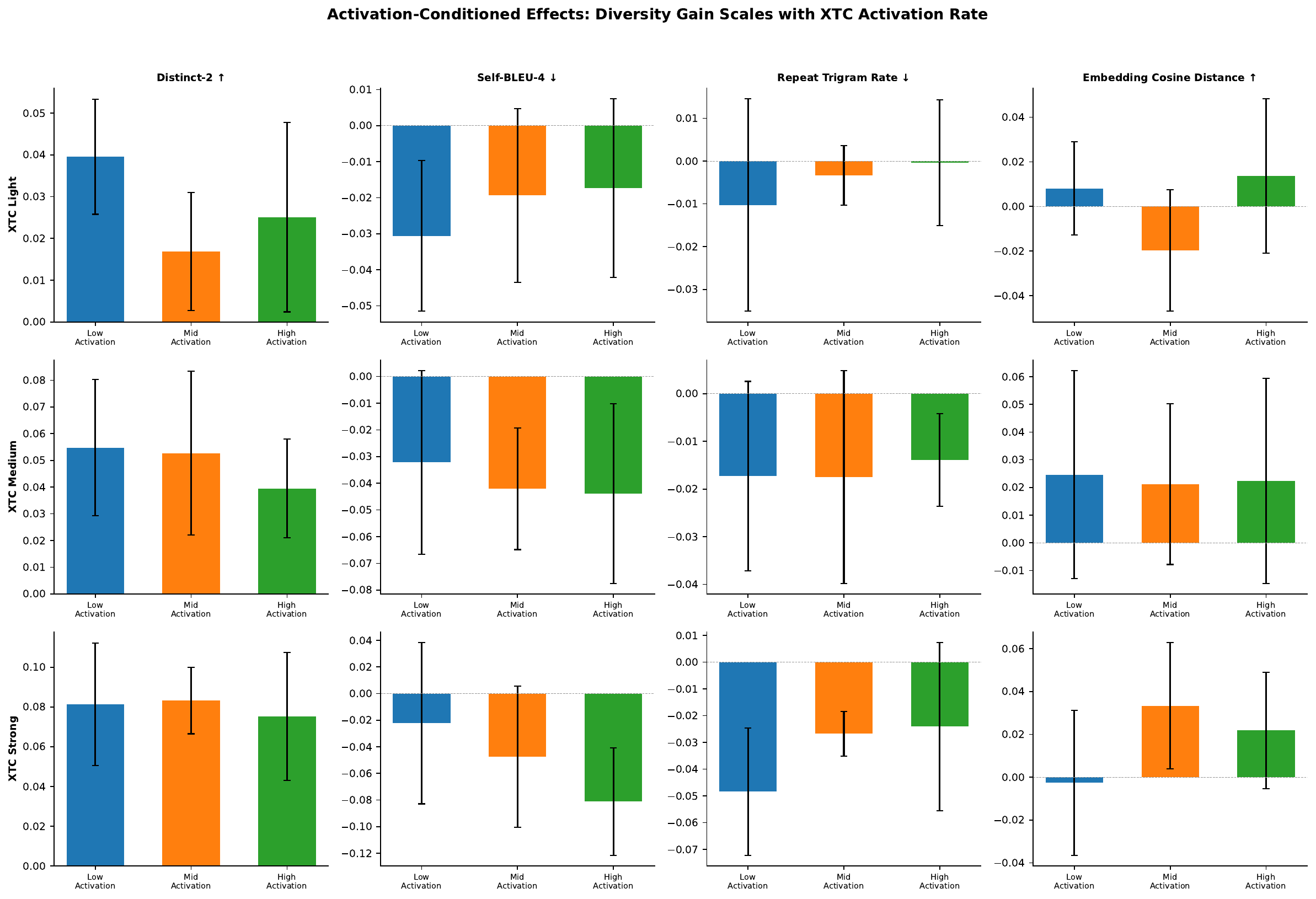}
\caption{Diversity gain stratified by prompt-level activation rate. High-activation prompts show larger gains, supporting the conditional hypothesis.}
\label{fig:activation_conditioned}
\end{figure}

\begin{figure}[h]
\centering
\includegraphics[width=\textwidth]{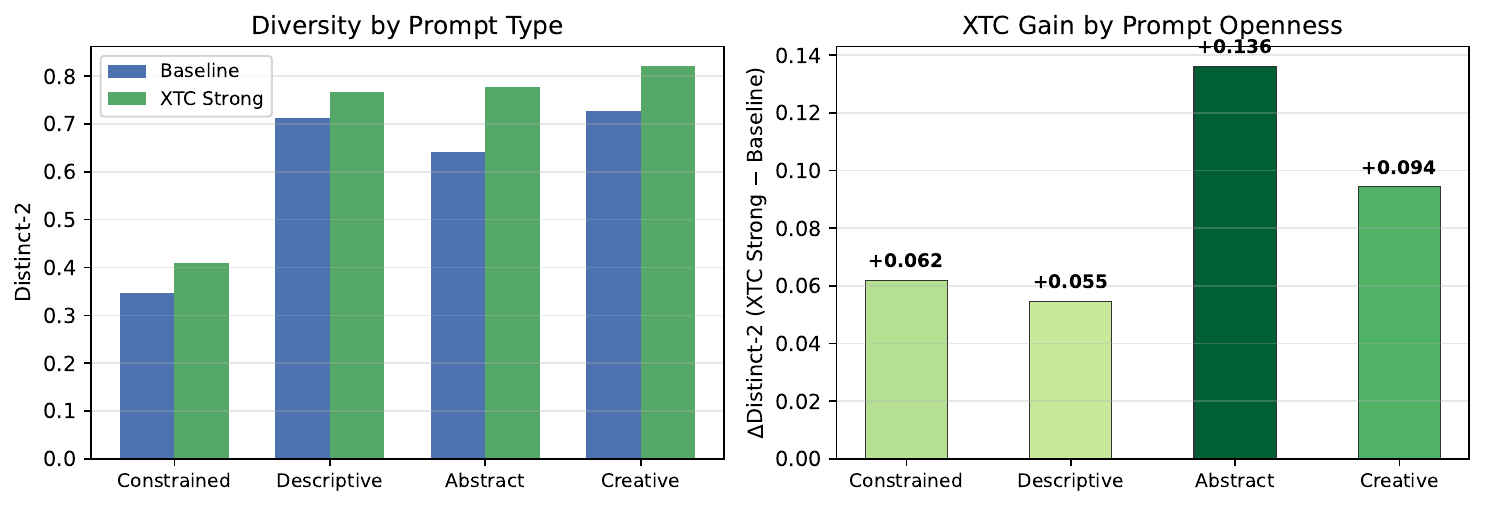}
\caption{Prompt-openness stratification: diversity gain increases with prompt openness, consistent with the hypothesis that \xtc{} benefits are largest when head multiplicity is naturally present.}
\label{fig:prompt_difficulty}
\end{figure}
\FloatBarrier

\subsection{Genre robustness}

\begin{figure}[h]
\centering
\includegraphics[width=\textwidth]{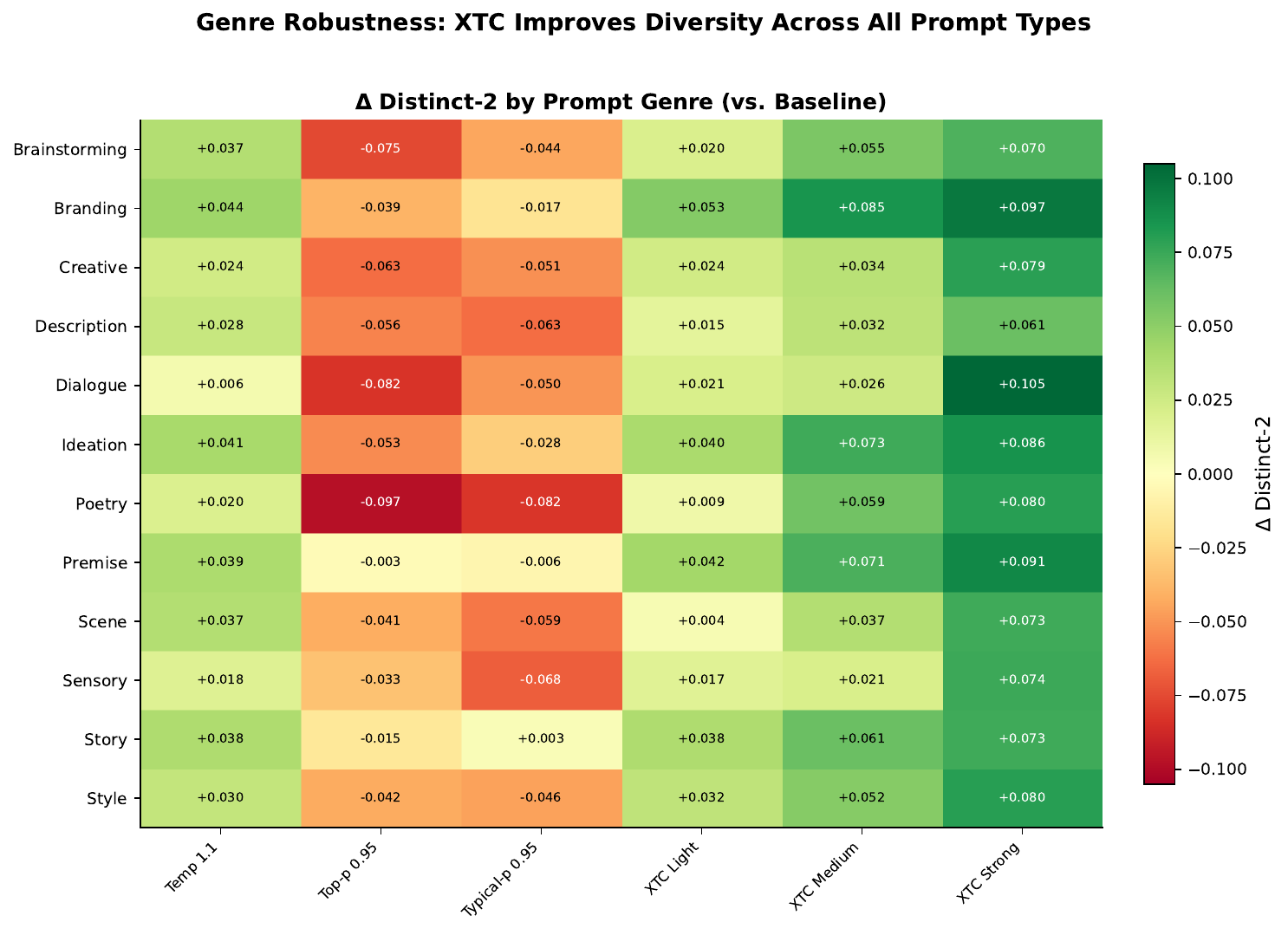}
\caption{$\Delta$\,Distinct-2 by prompt genre (12 genres) relative to baseline. \xtc{} produces positive gains across all 12 genres.}
\label{fig:tags}
\end{figure}

\begin{figure}[h]
\centering
\includegraphics[width=0.85\textwidth]{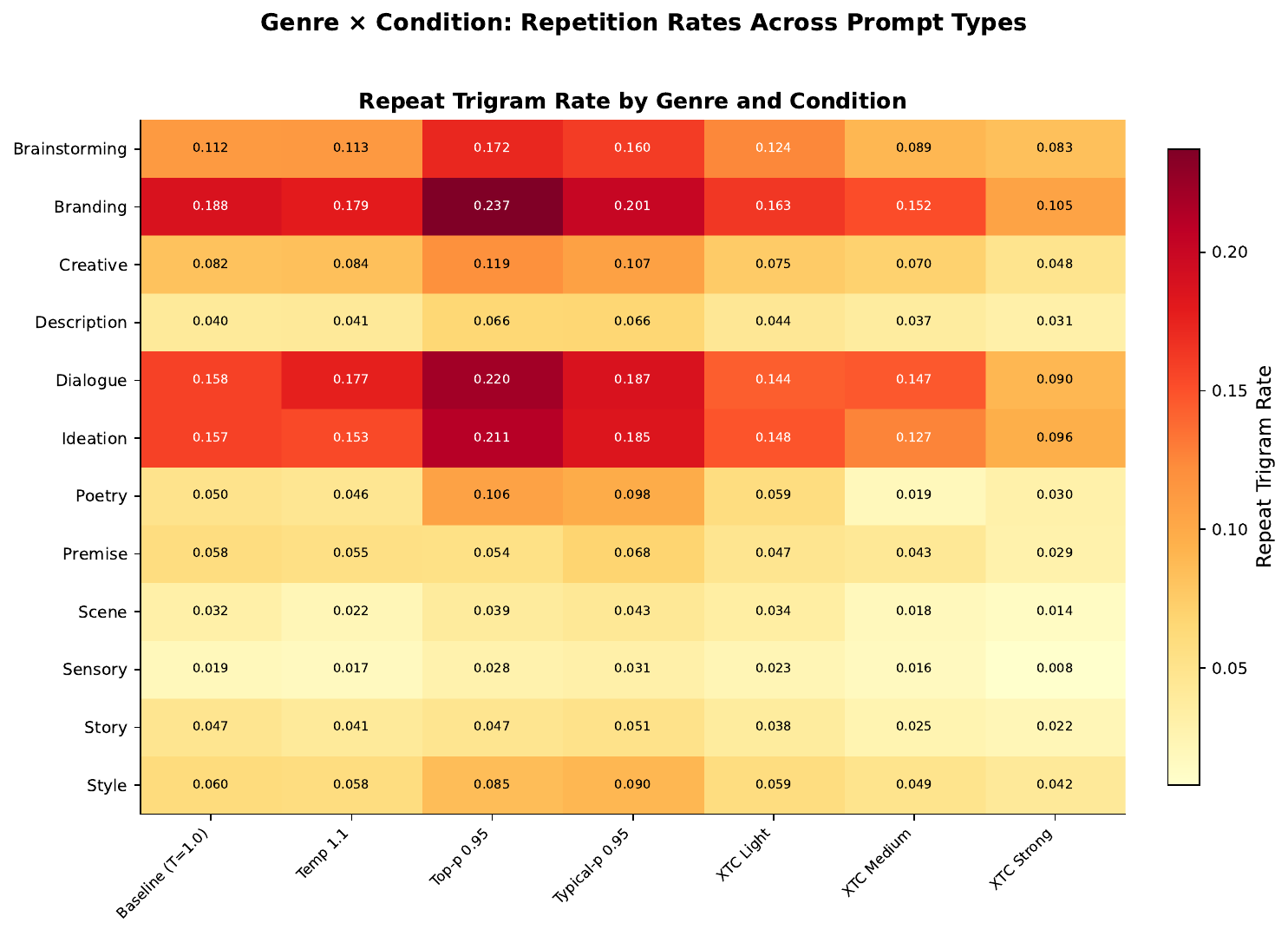}
\caption{Absolute repeat trigram rate by prompt genre. \xtc{} Strong achieves the lowest repetition across all 12 genres.}
\label{fig:tag_repetition}
\end{figure}
\FloatBarrier

\subsection{Extended metrics and radar comparison}

\begin{figure}[h]
\centering
\includegraphics[width=\textwidth]{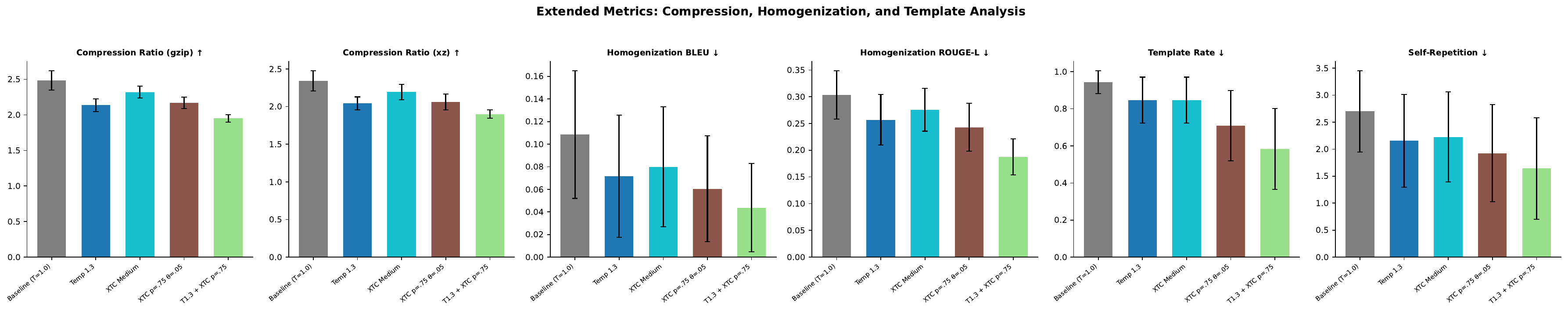}
\caption{Extended diversity metrics: compression ratios, homogenization, template rate, and self-repetition.}
\label{fig:extended}
\end{figure}

\begin{figure}[h]
\centering
\includegraphics[width=\textwidth]{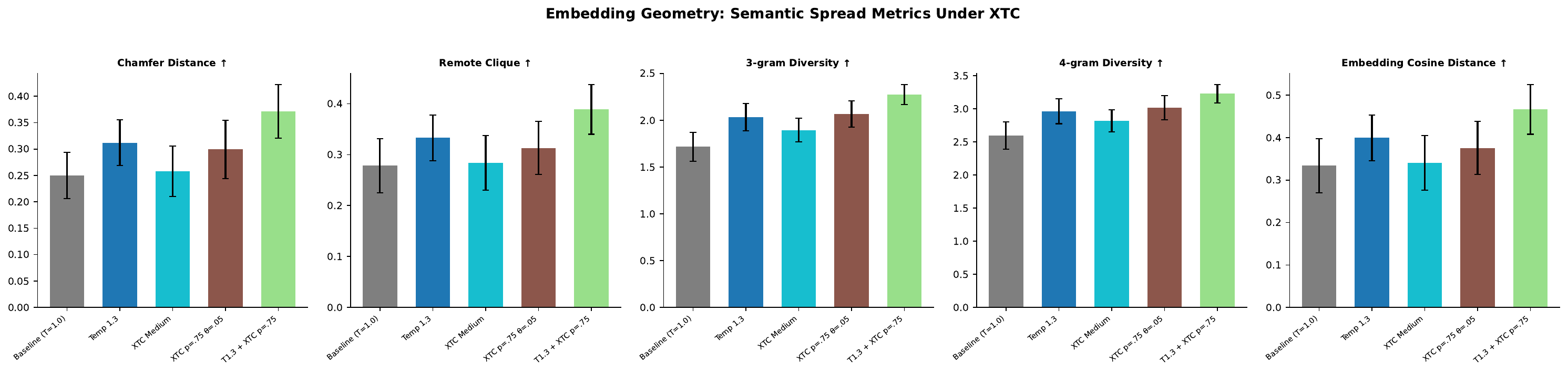}
\caption{Embedding geometry metrics: chamfer distance, remote clique, $n$-gram diversity, and cosine distance.}
\label{fig:geometry}
\end{figure}

\begin{figure}[h]
\centering
\includegraphics[width=0.65\textwidth]{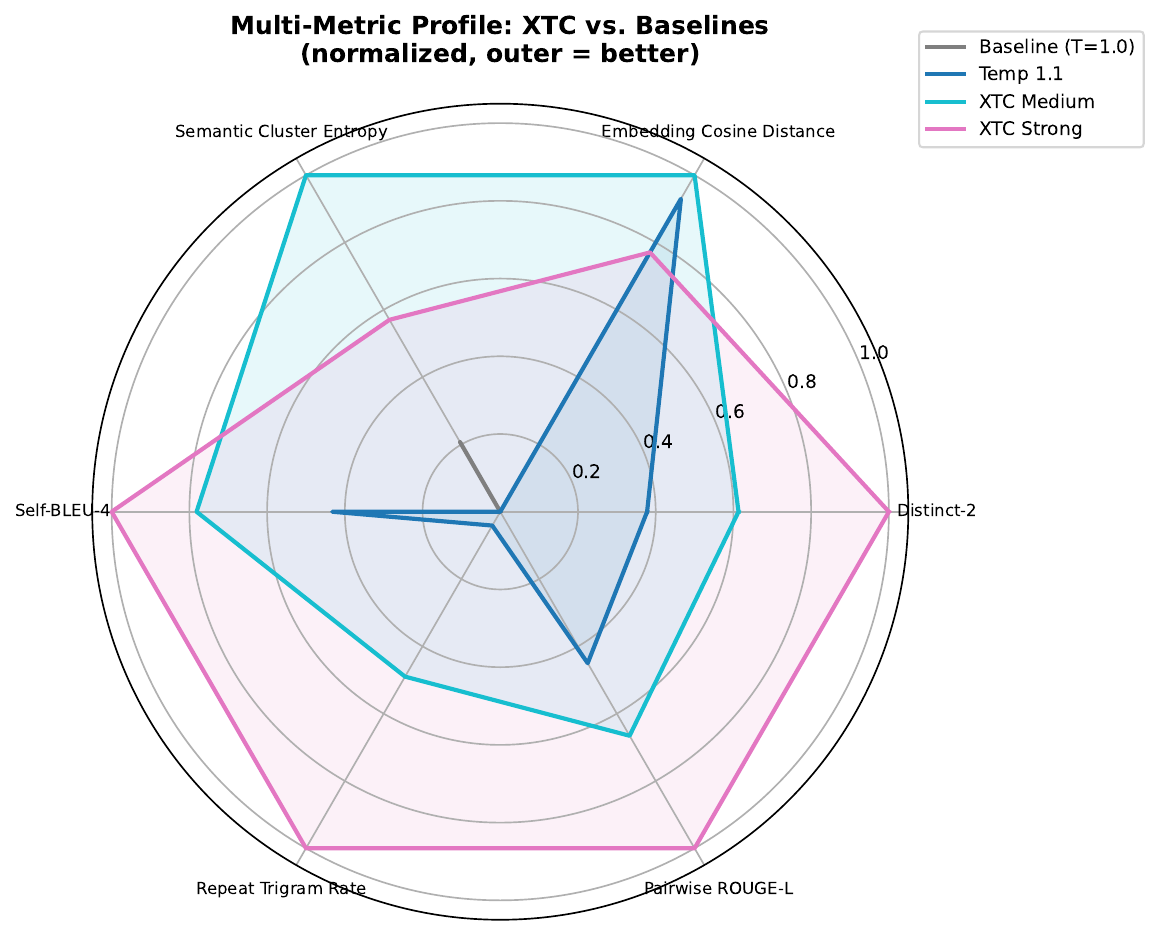}
\caption{Normalized multi-metric radar profile. \xtc{} Strong and Medium expand the profile envelope across all axes.}
\label{fig:radar}
\end{figure}
\FloatBarrier

\subsection{Long-form repetition}

\begin{figure}[h]
\centering
\includegraphics[width=\textwidth]{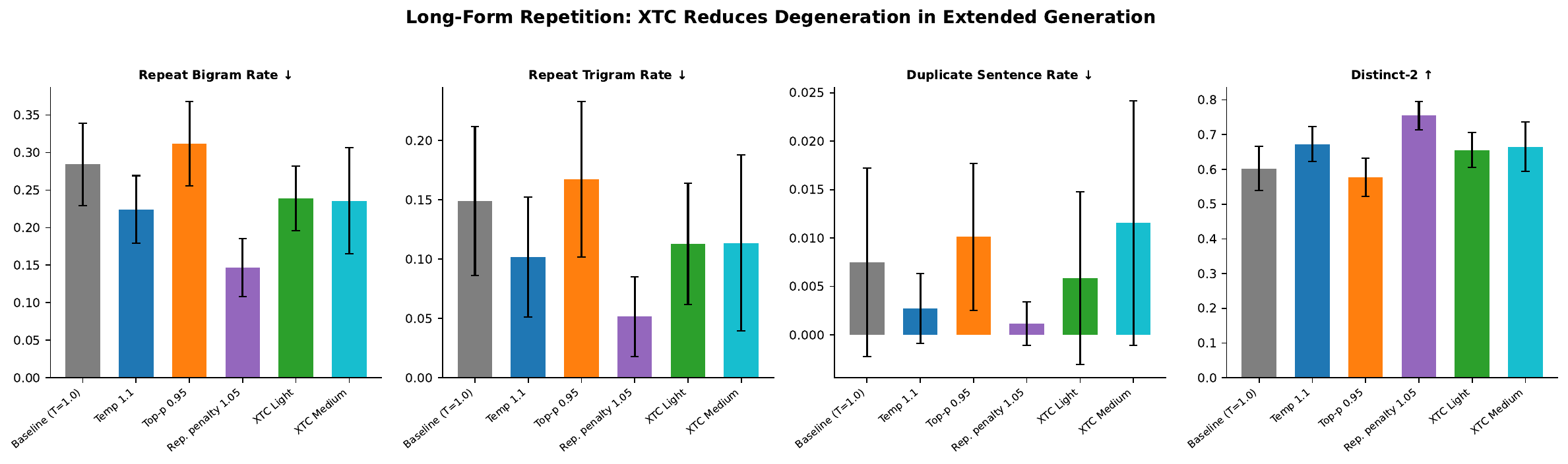}
\caption{Repetition evaluation on long-form generation prompts.}
\label{fig:repetition_eval}
\end{figure}

\begin{figure}[h]
\centering
\includegraphics[width=\textwidth]{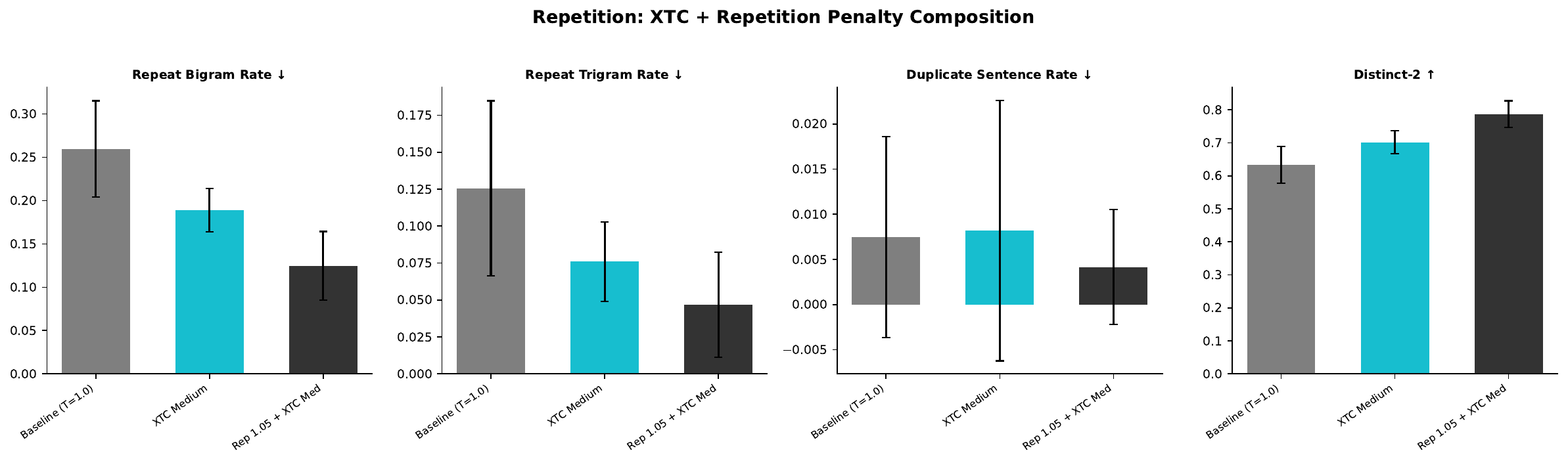}
\caption{Repetition benchmark: \xtc{} alone, repetition penalty + \xtc{}, and compositions on long-form generation.}
\label{fig:repetition_composition}
\end{figure}

\begin{figure}[h]
\centering
\includegraphics[width=\textwidth]{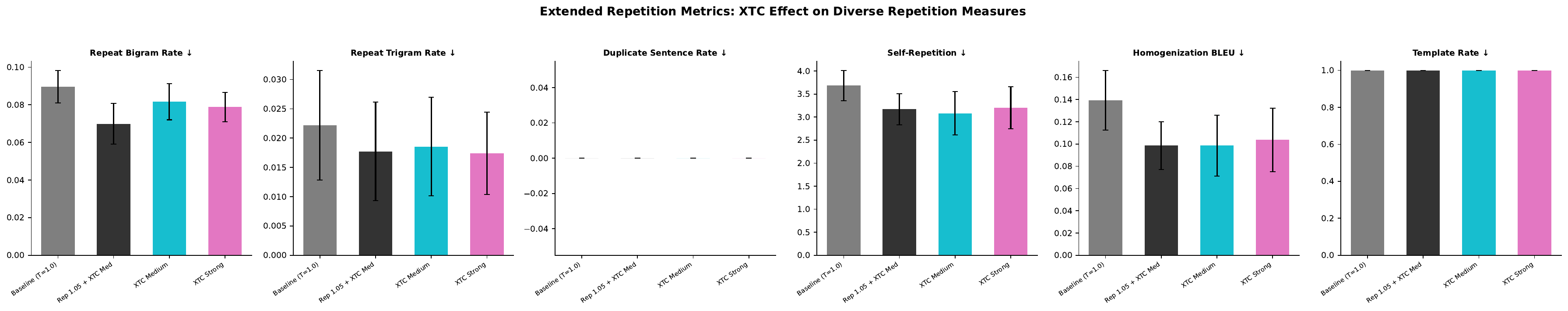}
\caption{Extended repetition metrics: self-repetition, homogenization BLEU, and template rate all decrease under \xtc{}.}
\label{fig:rep_extended}
\end{figure}
\FloatBarrier

\section{LLM-as-judge: full per-judge tables}
\label{app:llm_judge}

The main-text Figure~\ref{fig:judge_panel} summarizes the Opus + GPT-4o cross-vendor panel on Gemma~3 27B q4.
This appendix provides the underlying numbers (Table~\ref{tab:judge_27b_q4}) and the parallel Opus pass on the 12B~q6 cross-model pool (Table~\ref{tab:judge_12b}). The 70B~q4 Opus panel is shown as Figure~\ref{fig:judge_70b} in the main text.
The original DeepSeek R1 14B q6 cross-model creative-eval run had \texttt{max\_tokens}=176 and produced reasoning-trace-only content. We re-ran the configuration with \texttt{max\_tokens}=8192 and stripped \texttt{<think>...</think>} before scoring, and the corrected automatic metrics appear in Table~\ref{tab:crossmodel}.

\begin{table}[h]
\centering
\resizebox{\textwidth}{!}{%
\begin{tabular}{@{}llrrrr@{}}
\toprule
Judge & Condition & Overall & Creativity & Repetition & Lexical \\
\midrule
\multirow{3}{*}{Claude Opus 4.7}    & \xtc{} Medium  & $-0.05$ $[-0.35, +0.25]$ & $+0.15$ $[-0.05, +0.35]$ & $+0.05$ $[-0.35, +0.40]$ & $-0.10$ $[-0.40, +0.15]$ \\
                                    & \xtc{} Strong  & $-0.09$ $[-0.41, +0.23]$ & $+0.09$ $[-0.18, +0.36]$ & $\mathbf{+0.45}$ $[+0.14, +0.77]$ & $+0.23$ $[+0.00, +0.45]$ \\
                                    & top-$p$ $0.95$ & $-0.11$ $[-0.37, +0.16]$ & $-0.11$ $[-0.26, +0.00]$ & $-0.21$ $[-0.58, +0.16]$ & $-0.21$ $[-0.42, +0.00]$ \\
\midrule
\multirow{2}{*}{GPT-4o (control)}   & \xtc{} Strong  & $+0.05$ $[-0.15, +0.28]$ & $\mathbf{+0.39}$ $[+0.11, +0.65]$ & $\mathbf{+0.51}$ $[+0.22, +0.78]$ & $\mathbf{+0.28}$ $[+0.04, +0.50]$ \\
                                    & top-$p$ $0.95$ & $-0.18$ $[-0.39, +0.02]$ & $-0.10$ $[-0.27, +0.05]$ & $-0.30$ $[-0.61, +0.00]$ & $-0.18$ $[-0.36, +0.00]$ \\
\bottomrule
\end{tabular}}
\caption{Per-judge paired deltas vs.\ baseline on Gemma~3 27B q4 ($n{=}24$ paired samples per condition, 95\% bootstrap CIs). Bold entries have a 95\% interval that excludes zero. Both judges agree on direction across every criterion. Both resolve a significant repetition improvement at \xtc{} Strong, and GPT-4o additionally resolves creativity and lexical-diversity improvements that Opus shows as directionally positive but inside its CI. Neither judge resolves a significant negative \xtc{} effect on overall quality.}
\label{tab:judge_27b_q4}
\end{table}

\begin{table}[h]
\centering
\resizebox{\textwidth}{!}{%
\begin{tabular}{@{}llrrrr@{}}
\toprule
Judge & Condition & Overall & Creativity & Repetition & Fluency \\
\midrule
\multirow{3}{*}{Claude Opus 4.7}    & \xtc{} (p=0.75, t=0.05) & $-0.08$ $[-0.25, +0.08]$ & $-0.04$ $[-0.21, +0.12]$ & $-0.04$ $[-0.38, +0.29]$ & $-0.21$ $[-0.42, +0.00]$ \\
                                    & temp $1.15$ + \xtc{} & $+0.04$ $[-0.17, +0.21]$ & $+0.00$ $[-0.25, +0.21]$ & $+0.00$ $[-0.29, +0.29]$ & $-0.29$ $[-0.54, -0.04]$ \\
                                    & temp $1.30$ + \xtc{} & $-0.04$ $[-0.33, +0.25]$ & $+0.04$ $[-0.21, +0.29]$ & $-0.04$ $[-0.33, +0.25]$ & $-0.67$ $[-1.08, -0.25]$ \\
\bottomrule
\end{tabular}}
\caption{Opus paired deltas vs.\ baseline on the Gemma 3 12B q6 cross-model creative-eval pool ($n{=}24$ paired samples per condition). Baseline Opus repetition is $4.33$, near the ceiling, leaving no absolute headroom for further improvement. Composing temp~$1.30$ with \xtc{} degrades fluency, consistent with the standard temperature-over-warming pattern. Automatic Distinct-2 on this run is $+11.4\%$ (Table~\ref{tab:crossmodel}).}
\label{tab:judge_12b}
\end{table}

\FloatBarrier

\end{document}